\documentclass[a4paper,11pt]{article}
\usepackage{tikz}

\tikzset{font=\tiny}
\usepackage{subcaption}  
\usepackage[skip=0.5ex]{subcaption}
\usepackage[utf8]{inputenc}
\usepackage{mathpazo}
\usepackage{dsfont}

\usepackage[english]{babel}
\usepackage{amsmath,amssymb,amsfonts,siunitx,stmaryrd}
\usepackage{amsthm}
\usepackage{url,pgfplots}
\usepackage{xcolor}
\usepackage{enumitem} 

\definecolor{greendark}{HTML}{004d00}
\definecolor{greennormal}{HTML}{008000}
\definecolor{greenlight}{HTML}{00cc00}

\definecolor{bluedark}{HTML}{003366}
\definecolor{bluenormal}{HTML}{0066cc}
\definecolor{bluelight}{HTML}{3399ff}

\definecolor{orangenormal}{HTML}{FFA500}

\pgfplotsset{compat=1.18}
\definecolor{mplgreen}{rgb}{0.173, 0.627, 0.173}

\usepackage{rotating}
\usepackage{geometry}
\usepackage{mathtools}
\mathtoolsset{showonlyrefs}
\usepackage{subcaption}
\usepackage{comment}
\usepackage{algorithm}
\usepackage{algpseudocode}
\usepackage[hidelinks,english,pdfusetitle,colorlinks=false,linkcolor=blue,citecolor=greendark,
bookmarks=true]{hyperref}
\usepackage{orcidlink}

\theoremstyle{plain}
\newtheorem{lemma}{Lemma}[section]
\newtheorem{theorem}[lemma]{Theorem}
\newtheorem{corollary}[lemma]{Corollary}
\newtheorem{proposition}[lemma]{Proposition}
\newtheorem{remark}[lemma]{Remark}

\newtheorem{example}[lemma]{Example}
\theoremstyle{definition}

\renewcommand{\d}{\, \mathrm{d}}
\newcommand{\C}{\mathbb{C}}
\newcommand{\R}{\mathbb{R}}
\newcommand{\Rd}{{\mathbb{R}^d}}
\newcommand{\N}{\mathbb{N}}

\newcommand{\sd}{{\mathbb{S}^{d-1}}}
\newcommand{\E}{\mathbb{E}}

\newcommand{\e}{\mathrm{e}}

\renewcommand{\i}{\mathrm{i}}
\DeclareMathOperator*{\argmin}{arg\,min}

\DeclareMathOperator{\dom}{dom}
\DeclareMathOperator{\id}{id}

\DeclareMathOperator{\prox}{prox}
\DeclareMathOperator{\diam}{diam}
\DeclareMathOperator{\supp}{supp}
\DeclareMathOperator{\abs}{abs}

\DeclareMathOperator{\sinc}{sinc}
\DeclareMathOperator{\1}{\mathds{1}}
\DeclareMathOperator{\sgn}{sgn}

\newcommand{\forrall}{\quad \text{for all} \quad }

\usepackage{nicefrac} 

\title{Smoothed Distance Kernels for MMDs
	and Applications in Wasserstein Gradient Flows
}

\author{\texorpdfstring{Nicolaj Rux$^{1,2}$\footnotemark[1] \;\orcidlink{0009-0000-1299-0580}\\{\footnotesize\href{mailto:nicolaj.rux@math.tu-chemnitz.de}{nicolaj.rux@math.tu-chemnitz.de}}
		\and
		Michael Quellmalz$^2$\orcidlink{0000-0001-6206-5705}\\{\footnotesize\href{mailto:quellmalz@math.tu-berlin.de}{quellmalz@math.tu-berlin.de}}
		\and
		Gabriele Steidl$^2$\\{\footnotesize\href{mailto:steidl@math.tu-berlin.de}{steidl@math.tu-berlin.de}}}{Nicolaj Rux, Michael Quellmalz, Gabriele Steidl}}

\date{\today}

\begin{document}

	\maketitle
	\thispagestyle{empty}
	\begin{center}
		\parbox[t]{11em}{\footnotesize
			\hspace*{-1ex}$^1$TU Chemnitz\\
			Faculty of Mathematics\\
			Reichenhainer Straße 39\\
			D-09111 Chemnitz, Germany}
		\hspace{2em}
		\parbox[t]{11em}{\footnotesize
			\hspace*{-1ex}$^2$TU Berlin\\
			Institute of Mathematics\\
			Straße des 17. Juni 136 \\
			D-10623 Berlin, Germany} 
	\end{center}

	\begin{abstract}
		Negative distance kernels $K(x,y) \coloneqq - \|x-y\|$ were used in the definition of maximum mean discrepancies (MMDs) in statistics
		and lead to favorable numerical results in various applications.  In particular, so-called slicing techniques for handling  high-dimensional kernel summations profit from the 
		simple parameter-free structure of the distance kernel. 
		However, due to its non-smoothness in $x=y$, most of the classical theoretical results, e.g. on Wasserstein gradient flows of the corresponding MMD functional  do not longer hold true.
		In this paper, we propose a new kernel which 
		keeps the favorable properties of the negative distance kernel as being conditionally positive definite of order one with a nearly linear increase towards infinity and a simple slicing structure, but is Lipschitz differentiable now. Our construction is based on a simple 1D smoothing
		procedure of the absolute value function followed by a Riemann--Liouville fractional integral transform.
		Numerical results demonstrate that the new kernel performs similarly well as the negative distance kernel in gradient descent methods, but now with theoretical guarantees.

		
		
	\end{abstract}

	\section*{Declarations}
	\textbf{Acknowledgements:}
	We thank Sebastian Neumayer for his valuable suggestions, especially in the numerical aspects of this work.\\
	For open access purposes, the authors have applied a CC BY public copyright license to any author-accepted manuscript version arising from this submission.\\[1ex]
	\textbf{Conflict of Interest:} 
	On behalf of all authors, the corresponding author states that there is no conflict of interest.\\[1ex]
	\noindent\textbf{Competing Interests:}
	The authors have no competing interests to declare that are relevant to the content of this article. \\[1ex]
	\textbf{Funding Information:}
	MQs research was funded  the German Research Foundation (DFG): STE 571/19-1, project number 495365311, within the Austrian Science Fund (FWF) SFB 10.55776/F68 ``Tomography Across the Scales''. 
	GS acknowledges the funding support by the 
	DFG within the Excellence Cluster MATH+.
	NR gratefully acknowledges the funding support from the European Union and the Free State of Saxony (ESF). \\[1ex]
	\textbf{Author contribution:}
	All three authors have contributed equally to the manuscript.\\[1ex]
	\textbf{Data Availability Statement:}
	Data availability is not applicable to this article as no new data were created or analyzed in this study.\\[1ex]
	\textbf{Research Involving Human and /or Animals:}
	Not applicable because no research involving humans or animals has been conducted for this article.\\[1ex]
	\textbf{Informed Consent:}
	Not applicable because no research involving humans has been conducted for this article.
	
	\section{Introduction}
	Symmetric, positive definite functions have been playing a role in kernel-based learning for a long time \cite{zhou2007,steinwart2008supportderivfeature}.
	While mostly Gaussian kernels are used, recently, also the conditionally positive definite 
	negative distance kernel $K(x,y) \coloneqq -\|x-y\|$ 
	has attained interest, e.g. in
	statistics \cite{SSGF2013},
	image dithering/halftoning \cite{EGNS2021,GraPotSte12a,KV2023}, sampling \cite{NKSZ2022} and generative
	modeling \cite{HHABCS2024,HWAH2024}.
	Indeed, more general Riesz kernels $K(x,y) \coloneqq -\|x-y\|^s$,
	$s \in [0,2)$, were examined in optimization equilibrium problems, see, e.g.  \cite{CSW2022,GCO2021,FM2025}.
	Let us also mention that gradient flows with respect to the Coulomb kernel $K(x,y) \coloneqq \|x-y\|^{2-d}$ were quite recently examined in \cite{BV2023}, see also \cite{CDEFS2020},
	and $K(x,y) \coloneqq \|x-y\|^{-1}$ was applied in image halftoning in \cite{SGBW2010}.
	For interesting translation invariance properties of MMDs and connections with Wasserstein distances, we refer to \cite{MD2023}.
	
	Depending on the kernel, the  maximum mean discrepancy (MMD) between two measures can be defined
	as the sum of an interaction energy and potential energy.
	Fixing one of the measures, in generative learning called target measure, Wasserstein gradient flows of the corresponding functional
	on the Wasserstein-2 space starting in a simple (latent) measure can be applied to sample from that target distribution.
	While such gradient flows together with numerical forward and backward schemes
	for their computation are well understood for Lipschitz differentiable kernels,
	see, e.g. \cite{ambrosio2005gruenesbuchgradient,Arbel2019}, the
	convergence behavior of  forward steepest descent \cite{HGBS2024} and Euler backward (JKO) schemes \cite{JKO1998} are not clear for the negative distance kernel due to its nondifferentiability in $x=y$.
	One exception is the one-dimensional case, where the MMD functional
	becomes, in contrast to higher dimensions, (geodesically) $\lambda$-convex, see
	\cite{DSBHS2024} and the references therein.
	
	Gradient flows of the MMD functional or just the interaction energy
	with the negative distance kernel or Riesz kernels show a mathematically richer structure than those for smooth kernels and were the object of numerous
	examinations, see e.g. \cite{CDM2016,CH2017,CS2023}.
	In particular, singular
	measures can become absolutely continuous along the flow curve 
	and conversely \cite{BaCaLaRa13,HGBS2024},
	so that these flows
	are no longer just particle flows when starting in an empirical measure.
	Finally, let us mention flows in the MMD dissipation geometry \cite{zhu2024kernel} which differ from the setting considered in this paper.
	
	If applied in a straightforward way, MMD flows suffer from high computational costs
	in large scale computations, since each gradient step
	requires the computation of kernel sums
	(or their derivatives) with a large number of summands. For positive definite kernels, a remedy is to apply random Fourier feature techniques \cite{RR2007} based on Bochner's theorem.
	Unfortunately, the negative distance kernel does not fit into the setting of Bochner's theorem, but here
	efficient so-called slicing techniques, which project the high-dimensional problem in a bunch of one-dimensional ones, 
	can be used \cite{hertrich2024,hertrich2025fast}.
	For an interesting quite general fast summation approach using deep learning, we refer to
	\cite{HN2025}.
	
	In this paper, we construct a smoothed negative distance kernel
	such that its MMD functional fulfills the classical
	assumptions on its Wasserstein gradient flow and ensures in particular that empirical measures evolve as particle flows with proven convergence of Euler forward and backward schemes. On the other hand, these kernels are still conditionally positive definite of order one and behave
	in applications similarly as the negative distance kernel,
	but now with theoretical convergence guarantees.
	
	Our paper is organized as follows:
	in Section \ref{sec:pre}, we provide some notation and recall several results on (generalized) Fourier transforms.
	For readers not familiar with the topic, more material on tempered distributions and the relationship between the generalized and distributional Fourier transforms is added in Appendix~\ref{app:fourier}.
	
	The next three sections contain the steps for  defining  our smoothed distance kernels:
	Section \ref{sec:abs} starts with  appropriate smoothings of the absolute value function in~$\R^1$.
	Although not directly relevant for our construction, a relation to the often applied Huber function is addressed in Appendix \ref{app:huber}.
	Then, Section~\ref{sec:rd} establishes  smoothed Euclidean norm functions  in $\R^d$, $d \ge 2$
	based on Riemann--Liouville integral transforms, which
	finally lead  to our smoothed kernels in  Section~\ref{sec:sdistance}.
	Using these kernels, we define reproducing kernel Hilbert spaces and MMDs based on kernel mean embeddings in Section~\ref{sec:sdistance_1}.
	Wasserstein gradient flows of our MMDs are considered in Section \ref{sec:wgf}. We add considerations on the geodesic convexity of the MMDs in Appendix~\ref{app:geodesic}.
	Finally, we demonstrate the very good performance of Wasserstein gradient flows of the MMD with our new kernel by numerical examples in Section~\ref{sec:numerics}.
	
	All proofs, which are not indicated to be taken directly from the literature, are given in Appendix \ref{app:a}.
	
	\section{Preliminaries} \label{sec:pre}%
	The natural numbers including $0$ are denoted by $\N\coloneqq \{0,1,2,\ldots\}$.
	Let $\mathcal C_b({\mathbb R}^d)$ be the  space of continuous bounded functions $f\colon {\mathbb R}^d \to \mathbb C$ 
	with norm
	\[
	\lVert f\rVert_{\infty}
	\coloneqq
	\sup_{ x\in{\mathbb R}^d}\, \lvert f( x) \rvert,
	\]
	$\mathcal C_0({\mathbb R}^d)$ the subspace of  functions $f\colon {\mathbb R}^d \to \mathbb C$
	vanishing as $\| x\| \rightarrow \infty$,
	$\mathcal C_c({\mathbb R}^d)$ the subspace of  continuous functions with compact support, 
	$\mathcal C^n(\mathbb R^d)$, $n \in \N$ the space of $n$-times continuously differentiable functions
	and  $\mathcal C_c^n(\mathbb R^d)$ the space of $n$-times continuously differentiable functions with compact support.
	For  $1 \le p \le \infty$, let $L^p(\mathbb R^d)$ be the Banach space of all (equivalence class of)
	measurable functions $f\colon \mathbb R^d \to \mathbb C$
	with finite norm $\lVert f \rVert_{L^p}$ and $L^p_{\text{loc}}(\R^d)$ the corresponding locally integrable functions. 
	
	Further, we denote by $\mathcal S(\R^d)$ the space of complex-valued \emph{Schwartz functions}.
	The \emph{Fourier transform} 
	$\mathcal F\colon  \mathcal S(\R^d) \to  \mathcal S(\R^d)$
	is the bijective mapping defined  by
	\begin{equation}\label{eq:fourier_trafo}
		\hat \varphi (\omega)  = \mathcal F[\varphi](\omega)
		\coloneqq \int_\Rd \e^{-2\pi \i \langle x,\omega\rangle} \varphi (x)\d x
		,\qquad \omega\in\R^d.
	\end{equation}
	The \emph{Fourier transform} can be extended as a mapping
	$\mathcal F\colon L^1(\R^d) \to   \mathcal C_0(\R^d)$.
	The \emph{convolution function} $f*g$ of two functions $f,g$ on $\R^d$ is defined, if it exists, by
	$$
	(f*g)(x) \coloneqq \int_{y \in \R^d} f(x-y) g(y) \, \d y  = \int_{y \in \R^d} f(y) g(x-y) \, \d y
	,\qquad  x\in\R^d.
	$$
	In particular, if $f,g \in L^1(\R^d)$, then $f*g$ is defined almost everywhere and it holds the Fourier convolution theorem
	$$\mathcal F[f*g] = \hat f \,  \hat g.$$
	For $ r \in \N $, we define the space 
	\[
	\mathcal S_r(\R^d) \coloneqq
	\{\varphi \in  \mathcal S(\R^d): \varphi(x) \in \mathcal O (\|x\|^{r}) \; \text{ as } \; \|x\| \to 0 \}.
	\]
	A  measurable function $ \hat{ f} \in L_{\textup{loc}}^2 (\mathbb{R}^d \setminus \{0\}) $ is called  
	\emph{generalized Fourier transform} of a slowly increasing function
	$ f \in \mathcal C(\R^d) $, if there exists an integer $ r \in \N $ such that
	\begin{equation} \label{g_fourier}
		\int_{\mathbb{R}^d} f(x) \hat \varphi(x) \d x = \int_{\mathbb{R}^d} \hat f(\omega) \varphi(\omega) \d \omega
		\quad \text{ for all } \quad  \varphi \in \mathcal S_{2r}(\R^d),
	\end{equation}
	see \cite[Def.\ 8.9]{Wendland2004}. 
	If $f$ fulfills  \eqref{g_fourier} for some $r \in \N$, then it fulfills
	this relation also for all integers larger than $r$.
	In particular, if $f \in \mathcal S(\R^d)$, then \eqref{g_fourier} holds for all $r \ge 0$.
	The smallest $r \in \N$ such that \eqref{g_fourier} is fulfilled is called \emph{order of the generalized Fourier transform}.
	We have that $\hat f$ is uniquely determined.
	The generalized Fourier transform differs from the Fourier transform of
	so-called tempered distributions, in particular of continuous, slowly increasing functions, but coincides with it if restricted to test functions in 
	$\mathcal S_{2r}(\R^d)$. This is briefly explained in Appendix \ref{app:fourier}. 
	
	In this paper, we are mainly concerned with powers
	of the Euclidean norm.
	
	\begin{theorem}[{\cite[Thm.~8.16]{Wendland2004}}]
		The function $f(x) \coloneqq \|x\|^{\beta}$, $x \in \R^d$, with $\beta>0$, $\beta \not \in 2 \N$ has 
		the generalized Fourier transform
		$$
		\hat f(\omega)  
		= 
		\frac{
			\Gamma (\frac{d+\beta}{2})}{\pi^{\beta+ \frac{d}{2}}\Gamma(-\frac{\beta}{2}) }\|\omega\|^{-\beta - d}, \qquad\omega\in\R^d
		$$
		of order $r= \lceil \frac{\beta}{2} \rceil$.
		In particular, we have for $\abs(x)=|x|$, $x\in\R$, that
		\begin{equation}\label{fabs}
			\widehat{\abs}(\omega)=-\frac{1}{2\pi^2\omega^2},\qquad\omega\in\R.
		\end{equation}
	\end{theorem}

	For the generalized Fourier transform, we have the following convolution property.
	
	\begin{proposition}\label{lem:four_conv}
		Let $f\in  \mathcal C(\Rd)$ be a slowly increasing function with generalized Fourier transform $\hat f$ of order $r$ and $u\in  \mathcal C_c(\Rd)$. 
		Then the convolution 
		$f*u \in  \mathcal C(\R^d)$  is slowly increasing and
		has a generalized Fourier transform of order $r$ which fulfills 
		$\mathcal F[f*u]=\hat f\, \hat u$.
	\end{proposition}
	
	Further, the notation of conditionally positive definiteness
	will be central in our paper.
	A continuous, even function $ f \colon \mathbb{R}^d \to \mathbb C $ is \emph{conditionally positive definite of order} $ r \in \N$,
	if for all $ N \in \N $, all $ x_1, \dots, x_N \in \mathbb{R}^d $, 
	and all $ a \in \mathbb{C}^N \setminus \{0\} $ satisfying  
	\begin{equation}\label{eq:pos_def_cond}
		\sum_{j=1}^N a_j p(x_j)=0
	\end{equation}
	for all $d$-dimensional polynomials $p$ of degree  $\le r-1$, we have
	\begin{equation}\label{eq:pos_dev_quadratic_form}
		\sum_{j,k=1}^N a_j \bar a_k f(x_j-x_k)\ge 0,
	\end{equation}
	see \cite{Micchelli1986,sun1993conditionally}. 
	We denote the space of
	conditionally positive definite functions of order~$r$ by
	$\mathrm{CP}_r(\R^d)$.
	In particular,
	$-\| \cdot\|^\beta \in \mathrm{CP}_1(\R^d)$, $\beta \in (0,2)$.
	If $r=0$, we just speak about positive definite functions.
	Note that, by this definition, every $f\in \mathrm{CP}_r(\R)$ is continuous and even.
	
	Bochner's theorem characterizes positive definite functions
	as Fourier transform of positive measures, see Theorem \ref{thm:bochner} in Appendix \ref{app:fourier}.
	There are different ways to modify Bochner's theorem 
	for conditionally positive definite functions. 
	We will use the following one \cite[Thm.~8.12]{Wendland2004}.
	
	\begin{theorem}[Bochner's Theorem for Generalized Fourier Transform]\label{thm:cond_bochner}
		Let
		$ f\colon \mathbb{R}^d \to \mathbb{C} $ 
		be continuous, slowly increasing, and possess a 
		generalized Fourier transform $ \hat f $ of order $ r $, which is continuous on $ \mathbb{R}^d \setminus \{0\} $. 
		Then $ f $ is conditionally positive definite of order $r $ if and only if $ \hat f $ is nonnegative. 
	\end{theorem}
	Contrary to the generalized Fourier transform of $\|\cdot\|$, its distributional Fourier transform is not a function in the classical sense, see Appendix~\ref{app:fourier}. Together with Bochner’s theorem \ref{thm:cond_bochner}, the generalized Fourier transform therefore provides an appropriate framework for studying the (conditional) positive definiteness of the functions in Section~\ref{sec:abs}.
	
	\section{Smoothed Absolute Value Function} \label{sec:abs}
	In this section, we 
	propose to embellish $\abs(x) = |x|$ by convolving it with 
	functions from the set
	\begin{equation} \label{eq:Un}
		\mathcal U^n (\R)
		\coloneqq 
		\left\{u \in  \mathcal C_c^n(\R): u , \hat u \ge 0,\, u \text{ even}, 
		\int_{\R} u \d x = 1 \right\}, \quad n \in \N.
	\end{equation}
	These functions have the following nice properties.
	
	\begin{proposition}\label{prop:1}
		Let $u \in \mathcal U^n(\R)$ and 
		\begin{equation} \label{eq:ueps}
			u_\varepsilon (x) \coloneqq \frac{1}{\varepsilon}\, u\left(\frac{x}{\varepsilon} \right),\quad \varepsilon>0. 
		\end{equation}
		Then 
		$f \coloneqq \abs*u$
		fulfills:
		\begin{itemize}
			\item[i)] $f > 0$ and $f$ is even,
			\item[ii)] $f(x) = \abs(x)$ for $|x| \ge \diam(\supp(u))/2$,   
			\item[iii)] $f'' = 2 u$ so that  $f$ is convex and $f \in \mathcal C^{n+2}(\R)$,
			\item[iv)] $- f$ is conditionally positive definite of order
			$r=1$, but not positive definite,
			\item[v)]
			$(\abs*u_\varepsilon) (x) = \varepsilon f\left(\frac{x}{\varepsilon} \right), 
			\quad 
			(\abs*u_\varepsilon)'(x) = f'\left(\frac{x}{\varepsilon} \right), 
			\quad 
			(\abs*u_\varepsilon)''(x)= \frac{2}{\varepsilon} u\left(\frac{x}{\varepsilon} \right),
			$
			\item[vi)] $\abs*u_\varepsilon \rightarrow \abs$ uniformly
			as $\varepsilon \to 0$.
		\end{itemize}
	\end{proposition}
	
	The most important functions $u \in \mathcal U^n(\R)$
	in our numerical part will be centered cardinal $B$-splines.
	The \emph{centered cardinal} $B$-\emph{spline of order} $m \in \N$, $m\ge 1$, is
	recursively defined by
	$$
	M_1\coloneqq \boldsymbol{1}_{[-\frac12,\frac12]}, \quad
	M_m  \coloneqq M_1 * M_{m-1} , \quad m=2,3, \ldots
	$$
	$B$-splines have many useful properties, see \cite{medhurst1965evaluation,schoenberg1946}.
	
	\begin{proposition} \label{prop:splines}
		For the centered cardinal $B$-splines with $m\ge 1$, the following holds true:
		\begin{enumerate}[label =\roman*)]
			\item $M_m \ge 0$ and $\int_{\R} M_m(x) \, \d x = 1$,
			\item $\supp M_m = \left[-\frac{m}{2},\frac{m}{2}\right]$ and $M_m$ is even,
			\item $M_m \in \mathcal C^{m-2}(\R)$, $m \ge 2$,
			\item $\widehat M_m(\omega) = \sinc^m(\omega)$, 
			where 
			$\sinc(\omega) \coloneqq \frac{\sin (\pi\omega)}{\pi \omega}$.
			This is a nonnegative function exactly for even $m$.
			\item For $m \ge 2$, we have
			\begin{equation}\label{bspline}
				M_m(x) = \frac{1}{(m-1)!} \sum_{k=0}^{m} (-1)^k \binom{m}{k} \Big(x-k + \frac{m}{2} \Big)_+^{m-1},
			\end{equation}
			where $x_+ \coloneqq \max (x,0)$, and
			\begin{align} \label{m0}
				M_m(0)
				&=
				\frac{2}{\pi}\int_0^\infty \left(\frac{\sin(x)}{x}\right)^m \d x
				=
				\frac{m}{2^{m-1} } \sum_{k=0}^{\lfloor\frac m2\rfloor} \frac{(-1)^k (m-2k)^{m-1}}{m!\,(m-k)!}\\
				&
				= \sqrt{\frac{6}{\pi m}} \left( 1+\mathcal O(m^{-1})\right).
			\end{align}
			\item Clearly, it holds $M_{2m} \in \mathcal U^{2m-2}(\R)$.
		\end{enumerate}
	\end{proposition}
	
	The convolution of $\abs$ with the centered cardinal $B$-splines
	is given in the following proposition. 
	
	\begin{corollary}\label{cor:1}
		For $f \coloneqq \abs*M_m$, it holds
		$$
		f(x) = \frac{2}{(m+1)!}
		\sum_{k=0}^{m} (-1)^k \binom{m}{k} \Big(x-k + \frac{m}{2} \Big)_+^{m+1} - x.
		$$
	\end{corollary}
	
	Here are two examples.
	
	\begin{example}\label{ex:spline1}
		From
		\begin{equation}
			M_2 = \left\{
			\begin{array}{ll}
				1-|x|, & |x| \le 1,\\
				0,&\text{otherwise,}
			\end{array}
			\right.
			\quad
			M_4 =
			\frac16\begin{cases}
				3|x|^3-6x^2+4, & |x|\le 1,\\
				(2-|x|)^3,& 1<|x|\le 2,\\
				0, & \text{otherwise,}
			\end{cases}
		\end{equation}
		we get
		\begin{equation} \label{eq:abs*M2}
			(\abs*M_2)(x)=\begin{cases}
				\frac{1}{3}(-|x|^3+3|x|^2+1), & |x|\le 1,\\
				|x|, &  \text{otherwise,}
			\end{cases}
		\end{equation}
		and
		\begin{equation}
			(\abs*M_4)\,(x) =\begin{cases}
				\frac{1}{20}|x|^5 - \frac{1}{6}x^4 + \frac{2}{3}x^2 + \frac{7}{15}, & 0\le |x| <1, \\[0.5ex]
				\frac{1}{60}(2-|x|)^5 + |x|, &  1\le |x| <2,\\[0.5ex]
				|x|, & \text{otherwise.}
			\end{cases}
		\end{equation}
		For a plot of $\abs*M_2$ with its first and second order derivatives see Figure \ref{fig:fandF_derivs}.
	\end{example}
	Analogously to \eqref{eq:ueps}, we write for $m\in \N$ with $m\ge 1$ and $\varepsilon>0$
	\begin{equation*}
		M_{m,\varepsilon}(x)\coloneqq \frac{1}{\varepsilon}\, M_m\left(\frac{x}{\varepsilon}\right),\quad x\in \R.
	\end{equation*}
	
	Asking for smoothed absolute value functions, the Huber function may first come into one’s mind. Unfortunately, by the following corollary, the negative Huber function is not conditionally positive definite.
	
	\begin{corollary}\label{cor:huber}
		The Huber function
		$$
		f(x) \coloneqq
		\left\{
		\begin{array}{ll}
			\frac12 x^2, &  |x| \le \lambda,\\
			\lambda(|x| - \frac{\lambda}{2}),& \text{otherwise,}
		\end{array}
		\right.
		$$
		for $\lambda>0$
		can be rewritten as $f = \lambda \, (\abs*M_{1,2\lambda}) - \frac{\lambda^2}{2}$ and
		has the generalized Fourier transform 
		\begin{equation}
			\hat f (\omega)= - \frac{\lambda}{2\pi^2 \omega^2} \, \sinc(2\lambda \omega),
		\end{equation}
		which takes positive and negative values, so that
		$-f$ is not conditionally positive definite.
	\end{corollary}
	
	The proof  follows from formula \eqref{fhuber} in Appendix \ref{app:huber}. The Huber function is the so-called Moreau envelope of the absolute value function. Moreau envelopes play an important role in convex analysis.
	Appendix \ref{app:huber} contains
	more results on the relation of $\abs*M_m$ 
	to Moreau envelopes, which are interesting on their own.
	
	\section{Smoothed Euclidean Norm}\label{sec:rd}
	Our aim is to approximate the Euclidean norm on $\R^d$
	by a function which on the one hand keeps its desirable properties, 
	in particular radial symmetry, simple computation and conditional positive definiteness of order 1,
	and on the other hand gives rise to Lipschitz differentiable kernels in the next section.
	First ideas could be the following two:
	\begin{itemize}
		\item[-] Convolve the Euclidean norm in $\R^d$ with some smooth filter. 
		Unfortunately, this is numerically expensive in high dimensions.
		\item[-] Use $f(\|\cdot\|)$ with 
		$f = \abs*u$ and $u \in \mathcal U^n(\R)$. 
		Unfortunately, this function is in general not conditionally positive definite, as the following lemma shows. 
	\end{itemize}
	
	\begin{lemma}\label{ex:not_cond_pos}
		For $f=\abs*M_2$, it holds that $-f(\|\cdot\|) \not \in \mathrm{CP}_r(\R^d)$ for any $d \ge 2$ and $r\in\N$.
	\end{lemma}
	
	Since the above approaches do not provide the desired functions,
	we propose to use the 
	Riemann--Liouville fractional integral transform, which we consider next.
	
	\subsection{Riemann--Liouville Fractional Integral and Slicing in \texorpdfstring{$\R^d$}{R\^d}}
	
	For $d \in \N$, $d \ge 2$,
	the \textit{Riemann--Liouville fractional integral} 
	$\mathcal I_d\colon L^\infty_\mathrm{loc}(\R) \to \mathcal C^n(\R)$, $n \coloneqq \lfloor \frac{(d-2)}{2} \rfloor$ is defined by
	\begin{equation}\label{eq:RiemannLFRI}
		F(s)= \mathcal I_d[f](s)\coloneqq  c_d \int_0^1 f(ts)(1-t^2)^\frac{d-3}{2}\d t\forrall s\in \R,
	\end{equation}
	where $c_d\coloneqq\frac{2w_{d-2}}{w_{d-1}}$ and 
	$w_{d-1} \coloneqq \frac{2\pi^{\frac{d}{2}}}{\Gamma(\frac{d}{2})}$ denotes the surface area of the sphere $\mathbb S^{d-1}$. For $d=2$, the term $(1-t^2)^\frac{d-3}{2}$ is not bounded, but integrable, so that we require $f$ to be locally bounded in order for \eqref{eq:RiemannLFRI} to exist. 
	For $d\ge 3$, the term $(1-t^2)^\frac{d-3}{2}$ is bounded and we can define $\mathcal I_d$ on  $L^1_\mathrm{loc}(\R)$.
	
	Our approach is motivated by the slicing techniques for fast kernel summation in \cite{hertrich2024,hertrich2025fast}. 
	In particular, the 
	Riemann--Liouville fractional integral has the following
	useful property, which relates a high-dimensional radial function to a function on one-dimensional projections of its inputs, see \cite{rux2024slicing}.
	
	\begin{theorem} \label{thm:slicing_b}
		Let $d \in \N$,  $d \ge 2$ and $f \in L^\infty_\mathrm{loc}(\R)$ be even. 
		Then the even function $F\colon  \R \to \mathbb R$ 
		defined by the Riemann--Liouville fractional integral
		\eqref{eq:RiemannLFRI}
		fulfills the projection/slicing condition
		\begin{equation}\label{eq:slicingRLFI}
			F(\|x\|)= \frac{1}{\omega_{d-1}}\int_\sd f(\langle \xi,x\rangle )\d x =
			\E_{\xi\sim \mathcal U_{\mathbb S^{d-1}}} \left[ f \left(\langle x ,\xi \rangle \right) \right] ,
		\end{equation} 
		where $\mathcal U_{\mathbb S^{d-1}}$ denotes the uniform distribution on the sphere.
		Further, if $f$ is positive definite, then $F(\| \cdot\|)$ is also positive definite for all  $d\ge 2$. Conversely, if $F(\| \cdot\|)$  is positive definite for some $d \ge 2$, then there exists an even positive definite function $f$
		on $\mathbb R$ such that \eqref{eq:RiemannLFRI} is fulfilled.
	\end{theorem}
	
	The slicing formula \eqref{eq:slicingRLFI} is a special case of the adjoint Radon transform, see \cite{rux2024slicing}.
	The following two propositions extend the last property of Theorem~\ref{thm:slicing_b} to conditionally positive functions.
	
	\begin{proposition}\label{prop:cpd1}
		Let $d \in \N$,  $d \ge 2$ and $f \in L^\infty_\mathrm{loc}(\R)$ be even. 
		Further, let $f\in\mathrm{CP}_r(\R)$ 
		for $r \in \N$
		and
		$F=\mathcal I_d [f]$.
		Then $F(\|\cdot\|)\in\mathrm{CP}_r(\R^d)$.
	\end{proposition}

	\begin{proposition}\label{prop:cpd2}
		Let $d \ge 3$ and let the $\lfloor\frac{d}{2}\rfloor$-th derivative of $F\in   \mathcal C^{\lfloor\frac{d}{2}\rfloor}([0,\infty))$ be slowly
		increasing. Moreover, assume that $F(\|\cdot\|)\in \mathrm{CP}_r(\R^d)$ has a generalized Fourier transform $\rho(\|\cdot\|)\in \mathcal C(\R^d\setminus\{0\})$. 
		Then the function $f\in \mathrm{CP}_r(\R)$ with generalized Fourier transform
		\begin{equation*}
			\hat f\in \mathcal C(\R\setminus\{0\})
			,\qquad
			\hat f(\omega)=\frac{w_{d-1}}{2}\rho(\omega)|\omega|^{d-1},
		\end{equation*}
		fulfills \eqref{eq:RiemannLFRI}.
	\end{proposition}
	
	\subsection{Riemann--Liouville Fractional Integral of Smoothed Absolute Value}\label{sec:seuclid}
	Next, we are interested in the Riemann--Liouville fractional integral of the smoothed absolute value function $f \coloneqq \abs *u$,
	$u \in \mathcal U^n(\R)$.
	First of all, the absolute value function is an eigenfunction of $\mathcal I_d$, see, e.g. \cite{hertrich2025fast}.
	\begin{lemma}\label{lem:eigenf}
		The functions $\abs^\beta$, $\beta > -1$ are eigenfunctions of $\mathcal I_d$ with eigenvalues
		$
		\frac{\Gamma(\frac{d}{2})
			\Gamma(\frac{\beta+1}{2})}{\sqrt{\pi} \Gamma(\frac{d+\beta}{2})}
		$.
	\end{lemma}
	
	The Riemann--Liouville fractional integral of $\abs*u$ has the following properties.
	
	\begin{proposition}\label{prop:of_F}
		Let $n,d\in \N$ with $d\ge 2$ and $u \in \mathcal U^n(\R)$. Then the function $F\coloneqq \mathcal I_d[\abs*u]$
		is even,  convex, positive and $(n+2)$-times continuously differentiable. Further, it satisfies  for $s \to \infty$
		the relation
		$$
		F(s)= C_d \, |s|+\mathcal O\left(\frac{1}{s}\right), \quad
		C_d \coloneqq \frac{\Gamma(\frac{d}{2})}{\sqrt{\pi} \Gamma( \frac{d+1}{2})}.
		$$
		In particular, $F - C_d \abs\in \mathcal C_0(\R) \cap L^2(\R)$ and $F'\in  \mathcal C_b(\R)$ with $F'(0)=0$.
		\\
		The function $F^\varepsilon \coloneqq \mathcal I_d[\abs *u_\varepsilon]$
		converges in $L^2(\R)$ and also pointwise  to $C_d\abs $ as $\varepsilon \to 0$.
	\end{proposition}
	
	For the special case of $B$-splines $u \coloneqq M_m$,
	we have the following result.
	
	\begin{proposition}\label{prop:id_spline}
		For $m\in\N$ with $m\ge2$, let $f \coloneqq abs*M_m$. Then we have for $d\ge2$
		\begin{equation}
			\mathcal I_d[f](s)
			=
			{c_d} \sum_{k=0}^{m} (-1)^k \binom{m}{k} 
			\sum_{n=0}^{m+1} \frac{(\tfrac m2-2)^{m+1-n}}{n!(m+1-n)!} s^n  q_d(n,k-\tfrac m2;s) - \frac{\pi c_{d+1}}{2}s,
			\qquad s>0,
		\end{equation}
		where
		\begin{equation}
			q_d(n,a;s)
			\coloneqq
			\begin{cases}
				\frac{\Gamma (\frac{d-1}{2}) \Gamma (\frac{n+1}{2})}{\Gamma (\frac{d+n}{2})},
				& a\le0,
				\\
				\frac{\Gamma (\frac{d-1}{2}) \Gamma (\frac{n+1}{2})}{\Gamma (\frac{d+n}{2})}-B_{a^2/s^2}(\tfrac{n+1}{2},\tfrac{d-1}{2}), 
				& 0<a<s,
				\\
				0,& a\ge s
			\end{cases}
		\end{equation}
		with the incomplete Beta function $B_x(a,b)\coloneqq\int_0^xt^{a-1}(1-t)^{b-1} \d t$ for $a,b>-1$ and $x\in[0,1]$.
	\end{proposition}
	
	Note that for odd $d$, the incomplete beta function in $q_d(n,a;s)$ is a polynomial of degree $n+d-4$ in $1/s$,
	and hence $\mathcal I_d[f](s)$ is a rational function of $|s|$.
	In particular, we obtain for $u= M_2$  and $u=M_4$ the following functions $F$.
	
	\begin{example}\label{ex:spline_2}
		For $d=3$, it holds
		\begin{equation} \label{eq:I3M2}
			\mathcal I_3[\abs*M_2](s)
			=\frac{1}{12} \begin{cases}
				-|s|^3+4s^2+4, & |s|\le 1, \\
				6 |s| + \frac{1}{|s|} , &\text{otherwise,}
			\end{cases}
		\end{equation}
		and
		\begin{equation*}
			\mathcal I_3[\abs*M_4](s)
			=\frac{1}{360}\begin{cases}
				3|s|^5-12s^4+80s^2+168, & |s|\le 1, \\
				-|s|^5 + 12s^4 - 60 |s|^3 + 160 s^2 - 60|s| + 192 -\frac{4}{|s|} ,&1\le|x|\le 2, \\
				180|s|+\frac{60}{|s|}, & \text{otherwise.}
			\end{cases} 
		\end{equation*}
		For an illustration of the first function, see Figure~\ref{fig:fandF_derivs}.
		We have $\mathcal I_3[\abs*M_2]\in \mathcal C^3(\R)$ and $\mathcal I_3[\abs*M_4]\in \mathcal C^5(\R)$.
	\end{example}

	\begin{figure}[h]
		\centering
		\begin{tikzpicture}[scale=1.1]
			\begin{axis}[
				axis lines = middle,
				xlabel = {$x$},
				ylabel = {},
				title = {},
				samples = 100,
				axis equal image,
				enlarge y limits=false,
				ymin =-1.2,
				ymax =3.4,
				]
				\addplot[bluenormal, very thick, domain=-1:1] {1/3*(-abs(x)^3 + 3*abs(x)^2 + 1)};
				\addplot[bluenormal, very thick, domain=1:3.25] {abs(x)};
				\addplot[bluenormal, very thick, domain=-3.25:-1] {abs(x)};
				
				\addplot[greennormal, very thick, domain=-1:1] {1/3*(-3*(x/abs(x))*x^2 + 6*x)};
				\addplot[greennormal, very thick, domain=1:3.25] {sign(x)};
				\addplot[greennormal, very thick, domain=-3.25:-1] {sign(x)};

				\addplot[orange, very thick, domain=-1:1] {1/3*(-6*abs(x) + 6)};
				\addplot[orange, very thick, domain=1:3.15] {0};
				\addplot[orange, very thick, domain=-3.25:-1] {0};

				
				\addplot[bluenormal, very thick, dashed, domain=-1:1] {2* 1/12*(-abs(x)^3+4*x^2+4)};
				\addplot[bluenormal, very thick, dashed, domain=1:3.25] {2* 1/12*(6*abs(x)+1/abs(x))};
				\addplot[bluenormal, very thick, dashed, domain=-3.25:-1] {2* 1/12*(6*abs(x)+1/abs(x))};

				\addplot[greennormal, very thick, dashed, domain=-1:1] {2* 1/12*(-3*sign(x)*x^2+8*x)};
				\addplot[greennormal, very thick, dashed, domain=1:3.25] {2* 1/12*(6*sign(x)-sign(x)/x^2)};
				\addplot[greennormal, very thick, dashed, domain=-3.25:-1] {2* 1/12*(6*sign(x)-sign(x)/x^2)};

				\addplot[orangenormal, very thick, dashed, domain=-1:1] {2* 1/12*(-6*abs(x)+8)};
				\addplot[orangenormal, very thick, dashed, domain=1:3.15] {2* 1/12*(2/abs(x)^3)};
				\addplot[orangenormal, very thick, dashed, domain=-3.25:-1] {2* 1/12*(2/abs(x)^3)};

			\end{axis}
		\end{tikzpicture}
		\caption{Smoothed absolute value $f=\abs*M_2$ (solid, blue) 
			with its first (solid, green) and second (solid, orange) derivatives, the latter being equal to $2M_2$;
			and 
			$F=2\mathcal I_3[f]$ (dashed blue)  with its first (dashed, green) and second (dashed, orange) derivatives.}
		\label{fig:fandF_derivs}
	\end{figure}
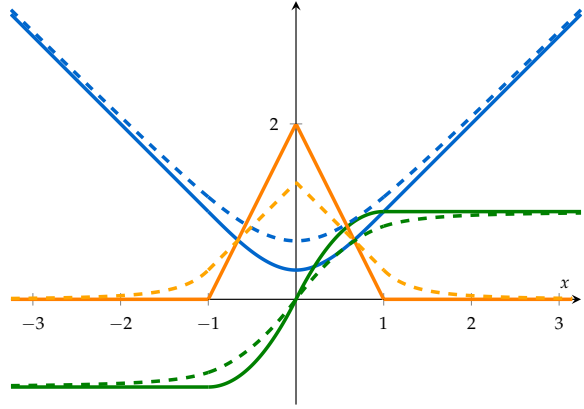

	Based on the previous results, we propose 
	to approximate the  negative Euclidean norm on $\R^d$ by
	\begin{equation}\label{def}
		\Phi = F(\| \cdot\|) \coloneqq \mathcal I_d [f] (\| \cdot\|), \quad f \coloneqq -\abs *u, \quad u \in \mathcal U^n(\R),\quad n \in \N.
	\end{equation}
	Summarizing Propositions \ref{prop:1} and \ref{prop:cpd1}, this function has the following properties.
	
	\begin{theorem}
		\label{thm:prop_of Phi}
		The function $\Phi$ in \eqref{def} has the following properties:
		\begin{enumerate}[label=\roman*)]
			\item $\Phi$ is conditionally positive definite of order one on $\R^d$.
			\item $\Phi(x) < 0$ for all $x \in \R ^d$.
			\item  $\Phi(x) = -C_d \, \|x \|+\varphi(\|x\|)$ with $\varphi\in  \mathcal C_0(\R)$ and $\varphi(s)\in \mathcal O(\frac{1}{s})$ as $ s\to \infty$.
			\item $\Phi$ is $n+2$ times continuously differentiable.
			\item $\nabla \Phi$ is Lipschitz-$L$ continuous with $L\coloneqq 2\sqrt{d}\|u\|_\infty $
			\item $\Phi$ is concave and $(-L)$-convex, i.e., for all $\lambda \in [0,1]$ and all $x,y \in \R^d$, we have
			$$
			\Phi(\lambda x + (1-\lambda) y) \le \lambda \Phi(x) + (1-\lambda) \Phi(y) + \tfrac{L}{2} \lambda (1-\lambda) \|x-y\|^2.
			$$
		\end{enumerate}
	\end{theorem}
	
	\section{Smoothed Distance Kernels}\label{sec:sdistance}
	In this section, we show how the above functions $\Phi$
	induce characteristic kernels with nice Lipschitz properties. These kernels can be used to define
	MMDs between measures and the MMDs can then serve as
	functionals for Wasserstein gradient flows.
	
	We call a symmetric function $K\colon\R^d \times \R^d \to \R$  a \emph{kernel}. A kernel is \emph{positive definite}, if
	for all $ N \in \N $, all $ x_1, \dots, x_N \in \mathbb{R}^d $, 
	and all $ a \in \mathbb{C}^N $ it holds
	\begin{equation}
		\sum_{j,k=1}^N a_j  a_k K(x_j,x_k )\ge 0.
	\end{equation}
	Unfortunately, the kernel $K(x,y) := F(\|x-y\|)$ with $F$ in \eqref{def} is  not positive definite, since $ F(\|\cdot\|)$
	is only conditionally positive definite of order $r=1$. 
	However, we have the following 
	proposition, see \cite[Thm~10.18]{Wendland2004}.
	Here $\Pi_{r-1}(\R^d)$ denotes the linear space of $d$-variate polynomials of degree $\le r-1$ which has dimension
	\smash{$N \coloneqq \binom{d + r -1}{r-1}$}.
	
	\begin{proposition}\label{cond_kernel}
		Let $\Phi\colon \R^d \to \R$ be a conditionally positive definite function
		of order $r \in \N$.
		Let $\Xi \coloneqq \{\xi_k: k=1,\ldots,N\}$ 
		be a set of points such that 
		$p(\xi_k) = 0$ for all $k=1,\ldots,N$ and any $p\in\Pi_{r-1}(\R^d)$ implies that $p$ is the 
		the zero polynomial. 
		Denote by $p_j,$ $j=1,\ldots,N$ the set of Lagrangian basis polynomials with respect to $\Xi$,
		i.e., $p_j(\xi_k) = \delta_{j,k}$.
		Then
		\begin{equation}
			\begin{aligned}
				K(x,y) 
				\coloneqq {}&
				\Phi(x-y) - \sum_{j=1}^N p_j(x) \Phi(\xi_j-y) 
				-  \sum_{k=1}^N p_k(y) \Phi(x-\xi_j )
				\\&+ \sum_{j,k=1}^N p_j(x) p_k(y)\Phi (\xi_j-\xi_k)
			\end{aligned}
			\label{make_pos_def}
		\end{equation}
		is a positive definite kernel. In particular, we have in case $r=1$ that
		$$
		\Phi(x-y) - \Phi(x) - \Phi(y) + \Phi(0)
		$$
		is positive definite, where we can skip the constant third term if $\Phi(0) \le 0$.
	\end{proposition}

	For our kernel from the function in  \eqref{def}, we obtain directly by Theorem \ref{thm:prop_of Phi}~v) and Proposition \ref{cond_kernel}
	the following corollary.
	
	\begin{corollary} \label{cor:grad_S_is_lip}
		Let $F(\| \cdot\|)$ be
		defined by \eqref{def}. Then 
		\begin{equation}\label{nice_kernel}
			K\colon\R^d \times \R^d \to \R,\qquad
			K(x,y) \coloneqq F(\|x-y\|) 
			- F(\|x\|) - F(\|y\|)
		\end{equation}
		is a positive definite kernel and
		\begin{equation} \label{kxx}
			K(x,x)
			= F(0) -2 F(\|x\|)
			\in \mathcal O(\|x\|).
		\end{equation}
		Moreover, $K$ is continuously differentiable with Lipschitz continuous gradient, i.e.,
		\begin{equation} \label{eq:dK-Lip}
			\|\nabla K(x,x')-\nabla K(y,y')\|\le L(\|x-y\|+\|x'-y'\|) \quad \text{for all} \quad
			x,x',y,y'\in \R^d.
		\end{equation}
	\end{corollary}
	
	\section{Maximum Mean Discrepancy with respect to \texorpdfstring{$K$}{K}}\label{sec:sdistance_1}
	A Hilbert space $\mathcal H$ of real-valued functions on $\mathbb R^d$ is called a
	\emph{reproducing kernel Hilbert space} (RKHS), 
	if the point evaluations $h \mapsto h(x)$, $h \in \mathcal H$, are continuous for all $x \in \R^d$.
	There exist various textbooks on RKHS from different points of view, see, e.g., \cite{zhou2007,Sc09,steinwart2008supportderivfeature}.
	By \cite[Thm.~4.20]{steinwart2008supportderivfeature}, every RKHS admits a unique positive definite kernel $K\colon\R^d \times \R^d \to \R$, 
	which is determined by the 
	reproducing property
	\begin{equation} \label{repr}
		h(x)
		= \langle h,K(x,\cdot) \rangle_{\mathcal H} \quad
		\text{for all} \quad h \in \mathcal H.
	\end{equation}
	In particular, we have $K(x,\cdot) \in \mathcal H$ for all $x \in \R^d$
	and
	\begin{equation}\label{k1}
		|h(x)| \le \|h\|_{\mathcal H} \|K(x,\cdot)\|
		_{\mathcal H} 
		=
		\|h\|_{\mathcal H} \sqrt{K(x,x)}.
	\end{equation}
	Conversely, for any positive definite kernel $K\colon\R^d \times \R^d \to \R$, 
	there exists a unique RKHS with reproducing kernel $K$, denoted by $\mathcal H_K$ \cite[Thm.~4.21]{steinwart2008supportderivfeature}. 
	
	RKHSs are closely related to measure spaces.
	Let $\mathcal M(\R^d)$ denote the space of finite, real-valued Radon measures
	and $\mathcal P(\R^d)$ the space of 
	probability measures on $\R^d$.
	Further, let
	\begin{equation}\label{eq:moment_M}
		\mathcal M_{\alpha}(\Rd)\coloneqq \left\{ \mu \in \mathcal M(\R^d)\colon \int_\Rd \|x\|^\alpha \d\mu(x)<\infty\right\},
		\quad  0<\alpha<\infty
	\end{equation}
	and similarly
	\begin{equation}\label{eq:moment_pM}
		\mathcal P_{\alpha}(\Rd)\coloneqq \left\{ \mu \in \mathcal P(\R^d)\colon \int_\Rd \|x\|^\alpha \d\mu(x)<\infty\right\},
		\quad  0<\alpha<\infty.
	\end{equation}
	
	Let $K(x,x) \in  \mathcal O(\|x\|^{\alpha})$.
	For example, we have by \eqref{kxx}
	for our kernel in \eqref{nice_kernel} that $\alpha = 1$.
	Then, it can be seen by \eqref{k1} that
	$\mathcal H_K \subset L^1(\mu)$ for all $\mu \in \mathcal M_{\nicefrac{\alpha}{2}}(\R^d)$ and 
	the so-called \emph{kernel mean embedding} (KME)
	$m \colon \mathcal M_{\nicefrac{\alpha}{2}}(\Rd)\to \mathcal H_K$, $\mu\mapsto m_\mu$ given by
	\begin{equation}\label{kme}
		\langle h, m_\mu \rangle_{\mathcal H_K}
		= 
		\int_{\R^d} h \d \mu \quad \text{for all} \quad h \in \mathcal H_K
	\end{equation}
	is well-defined, meaning that for every $\mu \in \mathcal M_{\nicefrac{\alpha}{2}}(\Rd)$ there exists a unique $m_\mu \in \mathcal H_K$
	such that \eqref{kme} is fulfilled 
	\cite[Lemma 4.24]{steinwart2008supportderivfeature}.
	In particular, we have by \eqref{repr} that
	\begin{equation}\label{kme_1}
		m_\mu(x) 
		= 
		\int_{\R^d} K(x,y) \d \mu(y) .
	\end{equation}
	The KME is not surjective  \cite{SF2021}. 
	For a positive definite kernel $K$ with  $K(x,x) \in  \mathcal O(\|x\|^{\alpha})$,
	the \emph{maximum mean discrepancy} (MMD)  
	$\mathcal D_K\colon \mathcal M_{\nicefrac{\alpha}{2}}(\R^d) \times \mathcal M_{\nicefrac{\alpha}{2}}(\R^d) \to \R_{\ge 0}$ 
	is by \eqref{k1} well-defined by
	\begin{align} 
		\mathcal D_K^2(\mu,\nu) &\coloneqq \int_{\R^d \times \R^d} K(x,y) \d (\mu(x) - \nu(x)) \d (\mu(y) - \nu(y))\label {mmd_def}\\
		&=
		\int_{\R^d \times \R^d} K(x,y) \d \mu(x) \d \mu(y) - 2 \int_{\R^d \times \R^d} K(x,y) \d \mu(x) \d\nu(y)\notag\\
		& \quad + 
		\int_{\R^d \times \R^d} K(x,y) \d \nu(x) \d\nu(y)
		\\
		&=
		\|m_\mu - m_\nu\|_{\mathcal H_K}^2,
		\label{eq:DK}
	\end{align}
	see \cite{BGRKSS06, GBRSS13}, where the last equality follows directly from the KME \eqref{kme_1}. If the KME is injective,
	then $K$ is called a \emph{characteristic kernel}. In this case,
	the MMD $\mathcal D_K$ is a distance on $\mathcal M_{\nicefrac{\alpha}{2}}(\R^d)$.
	Kernels induced by Gaussians are typical characteristic kernels.
	By the following proposition, also our kernel \eqref{nice_kernel} is characteristic, so that $\mathcal D_K$ is a distance 
	on $\mathcal M_{\nicefrac12}(\R^d)$.
	
	\begin{proposition} \label{thm:kme}
		Let $K$ be defined by \eqref{nice_kernel}. Then
		the kernel mean embedding $m\colon \mathcal M_{\nicefrac12}(\Rd) \to \mathcal H_K$ in \eqref{kme} is injective, i.e. $K$ is a characteristic kernel. 
		More precisely, for all $\mu \in \mathcal M_{\nicefrac{1}{2}}(\Rd)$, it holds    \begin{equation}\label{eq:MMD_S_Fourier}
			\|m_\mu\|_{\mathcal H_K}^2=
			\frac{1}{w_{d-1}\pi^2}
			\int_\Rd |\mu(\Rd)-\hat \mu(s)|^2 \frac{\hat u(\|s\|)}{\|s\|^{d+1}}\d s-F(0)\, \mu(\Rd)^2,
		\end{equation}
		where $\hat \mu$ denotes the Fourier transform of $\mu$, see \eqref{f_measure}.
	\end{proposition}
	In the proof of Proposition \ref{thm:kme}, equation  \eqref{eq:MMD_S_Fourier} is established first.  Then, the localization principle \cite[Lem~2.39]{plonka2018numerical} implies that the support of $\hat u$ is $\R$, because $u\in \mathcal U^n(\R)$ is compactly supported. As a consequence the kernel $K$ is characteristic.
	
	Fortunately, by the following theorem, when dealing with MMDs it is not necessary to work with the clumsy kernels \eqref{make_pos_def}, but instead we can directly use the conditionally positive definite kernels. Note that the MMD
	with respect to the negative distance kernel is also known as energy distances in statistics
	\cite{Szekely2002}.
	
	\begin{theorem}\label{prop:disc_of_ker}
		Let $\Phi\in\mathrm{CP}_r(\mathbb R^d)$ with $r\in \N$, $r\ge 1$ fulfill $\Phi\in \mathcal O(\|\cdot\|^{\alpha})$, and let $\tilde K(x,y) \coloneqq \Phi(x-y)$.
		Define the associate positive definite kernel $K$ by \eqref{make_pos_def}.
		Then  $\mathcal D_{\tilde K}$ in \eqref{mmd_def} is well-defined for $\mu,\nu \in \mathcal M_{\alpha} (\R^d)$ and
		$\mathcal D_{K}$ for $\mu,\nu \in \mathcal M_{\beta} (\R^d)$, where
		$\beta \coloneqq \max\{r-1, (r-1+\alpha)/2\}$.
		If $\mu,\nu\in \mathcal M_\alpha(\R^d)\cap \mathcal M_\beta(\R^d)$ have the same first $r-1$ moments, i.e. 
		$$ \int_{\R^d} p(x) \d\mu(x)= \int_{\R^d} p(x)\d\nu(x)\quad \text{ for all } p\in \Pi_{r-1}(\R^d),
		$$
		then 
		$$
		\mathcal D_{\tilde K}(\mu,\nu) = \mathcal D_{K}(\mu,\nu).
		$$
	\end{theorem}
	
	For our function $\Phi(x)\coloneqq F(\|x\|)$ with $F$ in \eqref{def}, we know already that
	$\mathcal D_K$ is well-defined for measures in $\mathcal M_{\nicefrac{1}{2}}(\R^d)$ which is in agreement with the proposition.
	However, by the proposition, $\mathcal D_{\tilde K}$ is only well-defined for measures in $\mathcal M_1(\R^d)$.
	If in addition $\int_\R \d \mu = \int_\R \d \nu$, then
	their distances $\mathcal D_K$ and $\mathcal D_{\tilde K}$ are the same.
	In particular, both distances are well-defined and coincide for measures in $\mathcal P_1(\R^d) \supset \mathcal P_2(\R^d)$.
	
	By the following remark, there is a relation between the degree
	of conditional positive definiteness and the growth of a function $\Phi$ towards infinity.
	
	\begin{remark}
		By \cite[Cor~2.3]{madych1990multi}, we have 
		$$\Phi\in\mathrm{CP}_r(\R^d) \quad \Longrightarrow \quad \Phi\in\mathcal O(\|\cdot\|^{2r}),$$
		which implies that $\alpha \le 2r$ in the assumption of 
		Theorem \ref{prop:disc_of_ker}.
		In general, this bound cannot be improved, since $(-1)^{r}\|\cdot\|^{2r-\varepsilon}\in \mathrm{CP}_{r}(\R^d)$ for any $r\in\N$, $r\ge 1$ and $\varepsilon\in[0,2)$ by \cite[Cor 8.18]{Wendland2004} and \cite[Lem 3.3]{sun1993conditionally}.
		However, for our function $\Phi(x)\coloneqq F(\|x\|)$ with $F$ in \eqref{def}, the above result says that
		$\Phi\in\mathcal O(\|\cdot\|^{2})$, but  we know already that
		$\Phi\in\mathcal O(\|\cdot\|)$.
	\end{remark}
	
	Finally, smoothness properties of the kernel transfer to the corresponding RKHS.
	
	\begin{proposition}\label{prop:RKHS_M_conv}
		For $d\ge 3$ and $n\ge 0$, let $u\in \mathcal U^n(\R^d)$. Let the kernel $K$ be given by  \eqref{nice_kernel}. Then every $h\in\mathcal H_K$ is $\big\lfloor \frac{n+2}{2} \big\rfloor$-times continuously differentiable. If $n\ge 2$ is even, then the gradient $\nabla h$ is $\sqrt{2d\|u''\|_\infty} \|h\|_{\mathcal H_K}$ Lipschitz continuous.
	\end{proposition}

	\section{Wasserstein Gradient Flows of MMDs}\label{sec:wgf}
	\subsection{Definition and Existence}
	The behavior of Wasserstein gradient flows of MMDs  depends on the kernel in their definition. While there exist many results for smooth kernels like the Gaussian, see, e.g., \cite{Arbel2019},  gradient flows of
	MMDs with Riesz kernels and in particular with the negative distance kernel
	have completely different properties, see, e.g., \cite{HGBS2024}.
	In contrast to smooth kernels, empirical measures do in general not remain empirical ones along the flow.
	Even if a steepest descent scheme, resp.\ the implicit Euler scheme exists, a convergence theory is still missing in dimensions larger than one.

	Let us briefly recall basic facts on Wasserstein gradient flows,
	see \cite{ambrosio2005gruenesbuchgradient,S2015} and show that our new kernels fulfill all assumptions which are required to ensure the existence of its 
	MMD gradient flow and the convergence of a forward and backward schemes.
	
	For $\mu,\nu \in \mathcal P_2(\R^d)$, we denote by 
	$$\Pi(\mu,\nu)\coloneqq\{\pi\in\mathcal P_2(\R^d\times\R^d):(P_1)_\#\pi=\mu,\,(P_2)_\#\pi=\nu\}
	$$ 
	the set of couplings with marginals $\mu$ and $\nu$, and by
	$(P_i)_{\#}\mu \coloneqq \mu \circ P_i^{-1}\in\mathcal P_2(\R^d)$ 
	the \emph{pushforward} of $\mu$ with respect to the projection
	$P_i(x_1,x_2) \coloneqq x_i$, $i=1,2$.
	Together with the Wasserstein distance
	\begin{equation} \label{eq:W2}
		W_2(\mu,\nu)^2 \coloneqq\min_{\pi\in\Pi(\mu,\nu)}\int_{\R^d\times\R^d}\|x-y\|_2^2\d \pi(x,y),
		\qquad \mu,\nu\in\mathcal P_2(\R^d),
	\end{equation}
	the set $\mathcal P_2(\R^d)$ becomes a complete metric space.
	The set of optimal couplings in \eqref{eq:W2} is denoted by
	$\Pi_{\text{opt}}(\mu,\nu)$.
	A curve $\gamma\colon I \to\mathcal P_2(\R^d)$ on an interval $I = [a,b]$, 
	$a<b$
	is called \emph{absolutely continuous}, 
	if there exists a Borel velocity field $v\colon I \times \R^d  \to\R^d$ with 
	$\| v_t \|_{L^2(\R^d; \gamma_t)} \in L^1(I)$
	such that the \emph{continuity equation} 
	\begin{equation}\label{eq:continuity-eq}
		\partial_t\gamma_t +\nabla_x \cdot(v_t\gamma_t)
		= 0    
	\end{equation}
	is fulfilled on $I \times\R^d$ in a weak sense, i.e., 
	for all 
	$\varphi \in \mathcal C_{c}^{\infty}\bigl((a,b)\times \R^d\bigr)$ it holds
	\begin{equation}
		\int_0^\infty \int_{\R^d} \partial_t \varphi(t,x) + \langle \nabla_x\varphi(t,x), v_t(x) \rangle \d \gamma_t(x) \d t
		= 0.
	\end{equation}
	There are many velocity fields corresponding to the same absolutely continuous curve, but only one with minimal  $\| v_t \|_{L^2(\R^d; \gamma_t)}$
	for a.e.\ $t \in I$.
	For a lower semi-continuous function $G\colon \mathcal P_2(\R^d) \to \R$, the reduced Fréchet sub\-differential $\partial G$ 
	consists of all $v \in L^2(\R^d, \mu; \R^d)$ such that for all $\eta \in \mathcal P_2(\R^d)$,
	\begin{equation}\label{eq:subdiff_ineq}
		G(\eta) - G(\mu) \ge \inf_{\pi \in \Pi_{\text{opt}}(\mu,\nu)}\int_{\R^d \times \R^d} \langle v(x), y - x\rangle \d \pi(x,y) + o(W_2(\mu,\nu)).
	\end{equation}
	If the minimal velocity field in the continuity equation 
	is determined by
	\begin{equation}\label{wgf}
		v_t  \in - \partial G(\gamma_t), \quad \text{for a.e.} 
		\quad t > 0,
	\end{equation}
	then $\gamma_t$ is called \emph{Wasserstein gradient flow of} $G$.

	Let $K\colon\R^d\times\R^d\to\R$ be a characteristic kernel such that its MMD is well-defined for measures in $\mathcal P_2 (\R^d)$. Examples are Gaussian kernels, the negative distance kernel, as well as our smoothed negative distance kernels in \eqref{nice_kernel}.
	For a fixed target measure $\nu\in\mathcal P_2(\R^d)$, we consider gradient flows of the squared MMD functional 
	\begin{equation} \label{eq:GK}
		G\colon \mathcal P_2(\R^d)\to[0,\infty)
		,\quad  G(\mu)\coloneqq \tfrac12 \mathcal D_K^2(\mu,\nu).
	\end{equation}
	If $K$ is continuously differentiable,
	the velocity field in \eqref{wgf} becomes
	\begin{align} \label{eq:wgf_mmd}
		v_t &= - \nabla_x \frac{\delta G}{\delta \gamma_t}
		= - \nabla_x \int_{\R^d} K(\cdot, y) \, \left(\text{d} \gamma_t(y) - \d \nu(y) \right) \\
		&= - \int_{\R^d} \nabla_x K(\cdot, y) \, \left(\text{d} \gamma_t(y) - \d \nu(y) \right), 
	\end{align}
	see, e.g., \cite{S2015}.
	Here, $\frac{\delta G}{\delta \gamma}$ denotes the functional derivative  defined, if it exists, by the
	function with
	$\frac{\mathrm d}{\mathrm d\epsilon}G(\gamma + \epsilon(\eta - \gamma))\big|_{\epsilon=0} = \int \frac{\delta G}{\delta \gamma}(\gamma)(\mathrm d\eta - \mathrm d\gamma)$
	for any $\eta \in\mathcal P_2(\R^d)$.
	Note that
	$v_t = - \nabla_x m_{\gamma_t - \nu}$
	for a positive definite kernel.
	For the negative distance kernel,
	we can compute $\frac{\delta G}{\delta \gamma_t}$ as above, but the gradient $\nabla_x$ does not exist in $x=y$, 
	i.e., \eqref{eq:wgf_mmd} is not well-defined,
	which causes the different behavior of those flows.
	The following result 
	guarantees the existence of Wasserstein gradient flows of  MMDs with sufficiently
	smooth kernels and its approximation by a Euler forward scheme.
	
	\begin{proposition}\label{prop:mmd-wgf}\cite[Prop 1\&3]{Arbel2019}
		Let $K\in \mathcal C^1(\R^d\times\R^d)$ be a  positive definite, characteristic kernel that has a
		Lipschitz-continuous gradient in the sense of \eqref{eq:dK-Lip}.
		Then, for any $\nu,\mu \in \mathcal P_2(\R^d)$, there exists a unique Wasserstein gradient flow $\gamma\colon [0,\infty) \to \mathcal P_2(\R^d)$ of the MMD functional \eqref{eq:GK} starting in $\gamma^{(0)} = \mu$.
		For a step size $\tau > 0$, we define the Euler forward iteration by
		\begin{equation}\label{eq:mmd-euler}
			\gamma^{(k+1)} \coloneqq (I - \tau v^{(k)})_\#\gamma^{(k)}
		\end{equation}
		where $v^{(k)}$ is related to $\gamma^{(k)}$,
		$k \in \N$ by \eqref{eq:wgf_mmd}.
		The approximated interpolation path 
		$$
		\gamma^\tau_t \coloneqq (I - (t - k\tau) v^{(k)})_\#{\gamma}^{(k)}, 
		\qquad 
		t \in [k \tau , (k + 1) \tau),
		$$
		satisfies
		$ W_2(\gamma^\tau_t, \gamma_t) \leq \tau\, C_T$  for all
		$t \in [0, T]$, where the constant $C_T$  depends only on $T>0$.
	\end{proposition}

	By Proposition \ref{thm:kme} and Corollary \ref{cor:grad_S_is_lip},  we obtain the following for our smoothed norm kernel.
	
	\begin{corollary}\label{cor:final_WGF}
		The  kernel
		\begin{equation}\label{eq:KF}
			K(x,y)=F(\|x-y\|)-F(\|x\|)-F(\|y\|)
		\end{equation}
		with $F$ from \eqref{def} fulfills the conditions of Proposition~\ref{prop:mmd-wgf}. There exists a Wasserstein gradient flow
		of the corresponding MMD functional \eqref{eq:GK}  and it can be approximated by the Euler forward scheme \eqref{eq:mmd-euler}.
		It holds 
		\begin{equation*}
			v_t =-\int_{\R^d} \nabla_x K(\cdot,y)\d(\gamma_t-\nu)(y)
			=-\int_{\R^d} \nabla_x F(\|x-y\|)\d(\gamma_t-\nu)(y),
		\end{equation*}
		so that $K$ can be replaced by $\tilde K(x,y)=F(\|x-y\|)$ without changing the flow results.
	\end{corollary}
	
	The last corollary remains valid for $F=\mathcal I_{d'}[f]$ with $d'>d$ as follows.
	
	\begin{remark}\label{rem:diff_d_slice}
		Let $d'\ge d$ and $F=\mathcal I_{d'}[f]$ for $f\in \mathrm{CP}_r(\R)$. By Proposition \ref{prop:cpd1}, we have ${F(\|\cdot\|)}\in \mathrm{CP}_r(\R^{d'})$ and hence, also $F(\|\cdot\|) \in \mathrm{CP}_r(\R^{d})$. 
		Similarly, if $K_{d'}\colon \R^{d'}\times \R^{d'}\to \R$ given by $K_{d'}(x,y)\coloneqq F(\|x-y\|)$ is a characteristic kernel in $\R^{d'}$, then also $K_{d}$ is characteristic in~$\R^d$. 
		Any measure $\mu\in \mathcal M_{\nicefrac{1}{2}}(\R^d)$ has the trivial extension $\tilde \mu\coloneqq \mu \otimes \prod_{k=d+1}^{d'} \delta_0\in \mathcal M_{\nicefrac{1}{2}}(\R^{d'})$,
		where $\delta_0$ is the Dirac measure at 0. 
		Then, the kernel mean embedding \eqref{kme_1} satisfies
		\begin{equation*}
			\|m_\mu\|_{\mathcal H_{K_{d}}}^2
			= \hspace{-3pt}\intop_{\R^d\times \R^d} \hspace{-3pt}F(\|x-y\|)\d\mu(x)\d\mu(y)
			= \hspace{-3pt}\intop_{\R^{d'}\times \R^{d'}}\hspace{-3pt} F(\| x- y\|)\d\tilde \mu(x)\d\tilde \mu(y)
			=\|m_{\tilde \mu}\|_{\mathcal H_{K_{d'}}}^2.
		\end{equation*}
		If $\|m_\mu\|_{\mathcal H_{K_{d}}}^2=0$, then $\tilde \mu=0$ as $K_{d'}$ is characteristic, and thus $\mu=0$. Hence, $K_d$ is characteristic.
		Therefore,  we can also use $F=\mathcal I_{d'}[f]$ to smooth the negative distance kernel in Corollary~\ref{cor:final_WGF}.
	\end{remark}
	
	There is a more general theory on Wasserstein gradient flows of
	$\lambda$-convex functionals, $\lambda \in \R$, along generalized geodesics, see \cite[Thm.~11.2.1]{ambrosio2005gruenesbuchgradient}.
	In Appendix~\ref{app:geodesic},  we show  that the functional  $G$ in \eqref{eq:GK}  with our smoothed negative distance kernel fulfills this $\lambda$-convexity with $\lambda < 0$ and establish an analogue to Corollary \ref{cor:final_WGF} for the Euler backward scheme.
	In particular, note that it is only ensured for $\lambda >0$ that the gradient flow converges to the (global) minimizer of $G$ as $t \to \infty$.
	Example \ref{ex:no_target} in Appendix  \ref{app:geodesic}
	shows that convergence to the global minimizer $\nu$ in \eqref{eq:GK} is in general not ensured, and the iteration may become stuck
	in another extreme point.
	
	Finally, let us mention that our new kernel can also be used in the definition 
	of other functionals $G$, 
	e.g., MMD-regularized $f$-divergences, where so far only bounded positive definite, characteristic kernels were applied.
	
	\begin{remark}[MMD-regularized \texorpdfstring{$f$}{f}-Divergence]
		In \cite{NSSR2025}, inspired by \cite{GAG2021},  Wasserstein gradient flows of  
		MMD-regularized $f$-divergences 
		were considered. 
		Unfortunately, the approach requires differentiability of the kernel
		and therefore does not work for negative distance kernels.
		In contrast, using Proposition \ref{prop:RKHS_M_conv}, it can be shown that our new smoothed distance kernel fits into the setting
		of the above papers.
	\end{remark}
	
	{
		
		\subsection{Discretization}
		
		In a discrete setting, we consider probability measures 
		\begin{equation} \label{eq:discrete-measure}
			\mu\coloneqq\frac1N \sum_{n=1}^N \delta_{x_n}
			,\qquad
			\nu\coloneqq  \frac{1}{M} \sum_{m=1}^M \delta_{y_m}
			,\qquad
			x_n,y_m\in \R^d,
		\end{equation}
		where $\delta_x$ is the Dirac measure at $x\in\R^d$.
		The MMD \eqref{mmd_def} between these measures 
		is
		\begin{equation*}
			\mathcal D_K^2(\mu,\nu)
			=
			\frac{1}{N^2}\sum_{n,n'=1}^N K(x_n,x_{n'}) 
			- \frac{2}{MN} \sum_{n,m=1}^{N,M} K(x_n,y_{m}) 
			+ \frac{1}{M^2} \sum_{m,m'=1}^M  K(y_m,y_{m'}).
		\end{equation*}

		The Wasserstein gradient flow of the MMD with a kernel fulfilling the assumptions of Proposition \ref{prop:mmd-wgf} keeps the empirical measure structure and moves just the positions of the Dirac measures. 
		Now let additionally $K$ be radial $K(x,y)=F(\|x-y\|)$ with some even function $F\in \mathcal{C}^2(\R)$.
		Then the forward Euler scheme \eqref{eq:mmd-euler} reads as
		\begin{equation} \label{eq:mmd-flow}
			\begin{aligned}
				x^{(k+1)}_i
				=
				x^{(k)}_i
				- \tau \Big( & \frac{1}{2N}\sum_{\substack{n=1}}^N (x^{(k)}_i-x^{(k)}_n) 
				\frac{F'(\|x^{(k)}_i-x^{(k)}_n\|)}{\|x^{(k)}_i-x^{(k)}_n\|} 
				\\&
				- \frac{1}{M} \sum_{m=1}^M  (x^{(k)}_i-y_m)  \frac{F'(\|x^{(k)}_i-y_m\|)}{\|x^{(k)}_i-y_m\|}\Big).
			\end{aligned}
		\end{equation}
		Because $F\in \mathcal C^2(\R)$ is even, we have $F'(0)=0$ and by L'Hôpital's rule 
		$$F''(0)=\lim_{s \to 0} \frac{F'(s)}{s}$$
		is well-defined.

		For the negative distance kernel $K(x,y)=F(\|x-y\|)$ with $F(s)=-|s|$, Wasserstein gradient flows are not known to exist for dimension $d\ge2$, because the squared MMD functional $G$ in \eqref{eq:GK} is not geodesically $\lambda$-convex, cf.\ \cite{HGBS2024}.
		However, we can replace $G(\mu)$ by $+\infty$ if $\mu$ is not an empirical measure, see, e.g., \cite{HWAH2024},
		and use the Euler scheme \eqref{eq:mmd-euler}.
		Then the summands in \eqref{eq:mmd-flow} have just the form  $\frac{x}{\|x\|}$ with 
		$x \in \{x_i^{(k)} - x_n^{(k)},x_i^{(k)} - y_m:i,n = 1,\ldots,N; m=1,\ldots,M \}$ if $x \not = 0$, 
		and we set $\frac{x}{\|x\|} \coloneqq 0$ for $x=0$.
		
		The following Proposition \ref{prop:dirac_flow} is for the flow of a single Dirac, but gives an intuition when $x_i^{(k)}$ is already close to~$y_i$. 
		It shows that with fixed $\tau$, the flow for the negative distance kernel tends to oscillate near the target, while it converges for the smoothed kernel \eqref{def}.
		In \eqref{eq:mmd-flow}, the repulsion term of $x_i$ and $x_j$ approximately cancels with the attraction term of $x_i$ and $y_j$, leaving only the attraction of $x_i$ and $y_i$ in form of gradient descent \eqref{eq:dirac_up} of $-F(\|\cdot-y_i\|)$ on $\R^d$.
		
		\begin{proposition}\label{prop:dirac_flow}
			Let $F\in \mathcal C^1((0,\infty))$.
			For the target measure $\nu = \delta_y$ and the initial measure $\gamma^{(0)} = \delta_{x^{(0)}}$ with $y, x^{(0)} \in \mathbb{R}^d$, the sequence $x^{(k)}$ from \eqref{eq:mmd-flow} simplifies to
			\begin{equation} \label{eq:dirac_up}
				x^{(k+1)} = x^{(k)} + \tau (x^{(k)} - y) \frac{F'(\|x^{(k)} - y\|)}{\|x^{(k)} - y\|},
			\end{equation}
			and we have the following:
			\begin{enumerate}[label=\roman*)]
				\item If $F = - \frac12 \abs$ with step size $\tau > 0$ and $0<\|y - x^{(0)}\| < \frac{\tau}{2}$, then $x^{(k)} = x^{(0)}$ for even $k$ and $x^{(k)} = x^{(1)}$ for odd $k$. In particular, $(x^{(k)})_k$ does not converge to $y$.
				\item If $F = -\mathcal{I}_d[\abs*u]$ with $u \in \mathcal{U}^0(\mathbb{R})$, then for sufficiently small $\tau$ and $\|x^{(0)} - y\| < \tau$, the sequence $(x^{(k)})_k$ converges exponentially to $y$.
			\end{enumerate}
		\end{proposition}
		The proof of Proposition \ref{prop:dirac_flow} is given in Appendix \ref{app:geodesic}.

	}
	
	\section{Numerical Results} \label{sec:numerics}
	{
		We compare the gradient flows \eqref{eq:mmd-flow} of $G = \frac12 \mathcal D_K^2(\cdot,\nu)$ for $K(x,y) = F(\|x-y\|)$ with different functions $F$.
		For the first part of the numerics, we use the two-dimensional targets $\nu$ as in Figure \ref{fig:datasets}. In the second part, we use the high-dimensional MNIST dataset.
		All computations are performed using PyTorch on an Intel i7-10700 CPU with 32\,GB memory and an NVIDIA GeForce RTX 2060 GPU.

		\subsection{Examples in 2D}
		
		For the two-dimensional examples ($d=2$), we use the following kernels:
		\begin{enumerate}
			\item[a)] \textbf{Gaussian}  $F(s)\coloneqq\exp(\frac{-s^2}{2\sigma^2})$ for $\sigma>0$, 
			\item[b)]  \textbf{SND}: smoothed negative distance  
			\begin{equation} \label{eq:snd}
				F\coloneqq -\mathcal I_3[\abs*u_\varepsilon]
				\text{ with }
				u_\varepsilon (x) \coloneqq \tfrac{1}{\varepsilon}\,M_2(\tfrac{x}{\varepsilon}),
			\end{equation}
			\item[c)] \textbf{ND}: negative distance  
			$F \coloneqq  -\frac12 \text{abs}$.
		\end{enumerate}
		
		The usage of $\mathcal I_3$ instead of $\mathcal I_2$ for the SND kernel is justified by Remark~\ref{rem:diff_d_slice}.
		The reason is the simple structure of $\mathcal I_d[\abs * M_m]$ for $d=3$, see Example~\ref{ex:spline_2}, as opposed to $d=2$. 
		The constant in the ND kernel is chosen so that it is the limit of the SND kernel for $\varepsilon\to0$, see Proposition~\ref{prop:of_F}.
		
		For the SND kernel, we found choices $\varepsilon\in [10^{-4}, 10^{-2}]$ to work generally well. During testing, we did not encounter numerical issues due to small $\varepsilon$,
		instead the behavior of the SND approaches to that of the ND kernel. Large $\varepsilon$ over-smooth the kernel, hurting the numerical performance.
	}

	\subsubsection{Three-Rings Target}
	
	\begin{figure}
		\centering
		\begin{subfigure}{0.44\textwidth}
			\centering
			\includegraphics[scale=0.27]{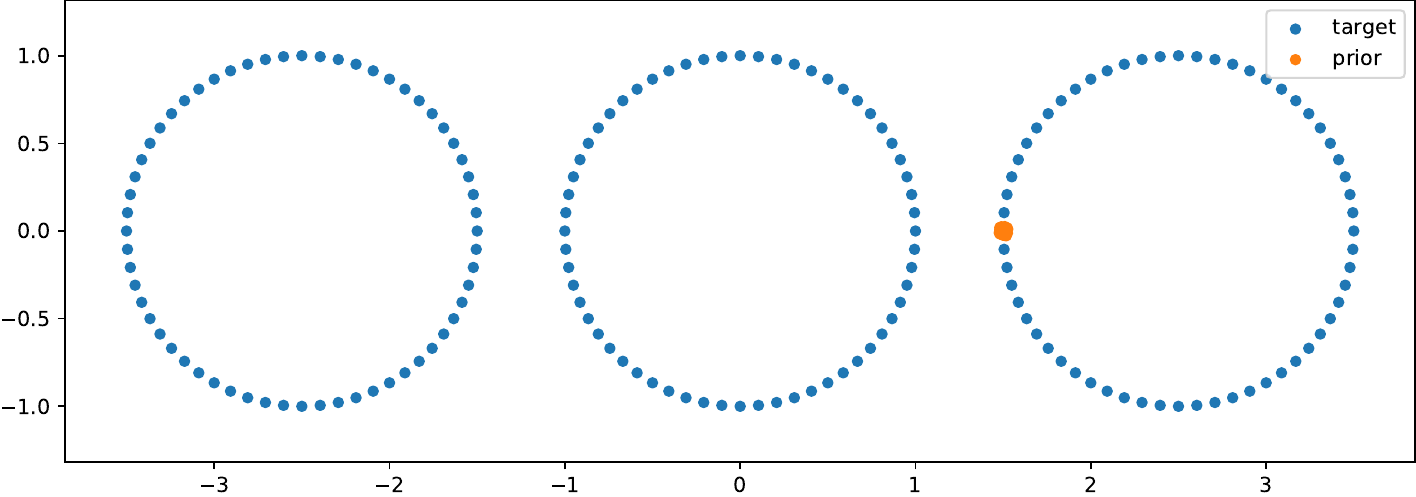}
			\caption{Three-Rings}  \label{fig:3Rings}
		\end{subfigure}
		\hfill
		\begin{subfigure}{0.27\textwidth}
			\centering
			\includegraphics[scale=0.27]{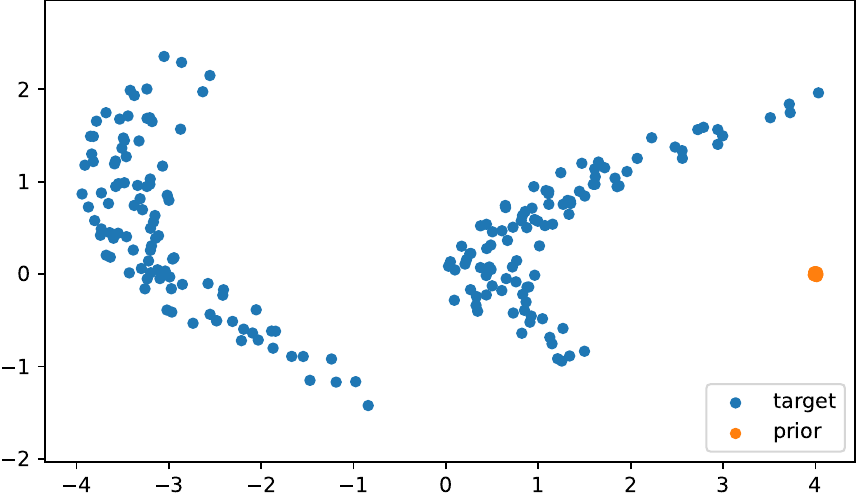}
			\caption{Bananas}  \label{fig:Bananas}
		\end{subfigure}
		\hfill
		\begin{subfigure}{0.26\textwidth}
			\centering
			\includegraphics[scale=0.27]{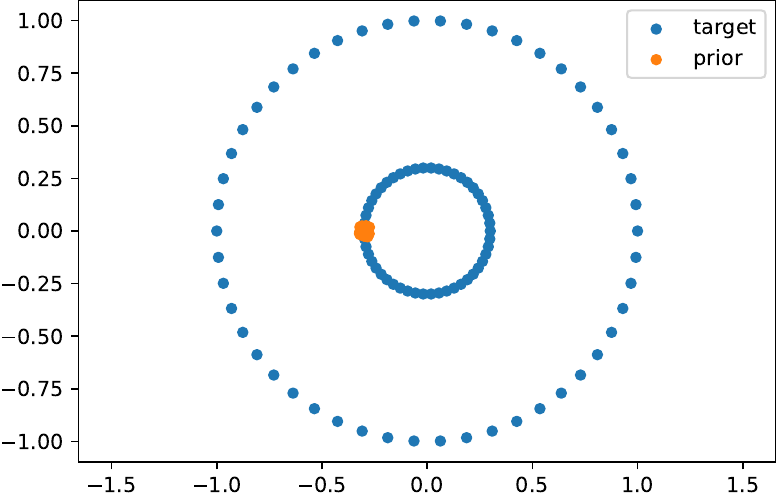}
			\caption{Annulus}  \label{fig:Annulus}
		\end{subfigure}
		\caption{Target measures $\nu$ (blue) and initialization $\gamma^{(0)}$ (orange).}
		\label{fig:datasets}
	\end{figure}
	The Three-Rings target $\nu$ in Figure~\ref{fig:3Rings} from \cite[Fig.~1]{GAG2021} consists of three circles in $\R^2$ with radius $1$ and midpoints $(-2.5,0),(0,0)$ and $(2.5,0)$ discretized with $M= {3\cdot 40} = 120$ points.
	The initialization $\gamma^{(0)}$ is a highly localized Gaussian with standard deviation $10^{-4}$, see Figure~\ref{fig:3Rings}.
	
	We computed the iteration \eqref{eq:mmd-euler}
	with step size $\tau=0.01$ in double precision and display the flow after $k\in \{1\, 000,5\, 000,10\, 000, 50\, 000\}$ iterations or equivalently after time $t=\tau k$.
	Figure~\ref{fig:3Rings_GAUSS} shows the flows for the Gaussian kernel with standard deviation $\sigma\in \{0.06,\, 0.3,\, 1\}$. Here, the quality of the result heavily depends on the choice of~$\sigma$. If $\sigma$ is too small, the points cover only two circles; if too large, the points do not lie on the circles. The sweet spot is around $\sigma=0.3$, but even then some particles get stuck far from the target. 
	Figure~\ref{fig:3Rings_SND} shows the flows for the SND kernel \eqref{eq:snd} for $\varepsilon\in \{1,\, 0.1,\, 0.01\}$. Here, it is preferable to choose a small~$\varepsilon$, then all three circles are recovered well.
	Figure \ref{fig:3Rings_ND} depicts the results with the ND kernel.
	The flows in Figure~\ref{fig:3Rings_ND} and Figure \ref{fig:3Rings_SND} for $\varepsilon=0.01$ are almost identical.

	\begin{figure}
		\centering
		\begin{subfigure}{\textwidth}
			\centering\footnotesize
			\begin{tabular}{c c c c c}
				\vspace{-4pt}
				\rotatebox{90}{\parbox{2cm}{\centering $\sigma=0.06$}} &
				\includegraphics[width=3cm]{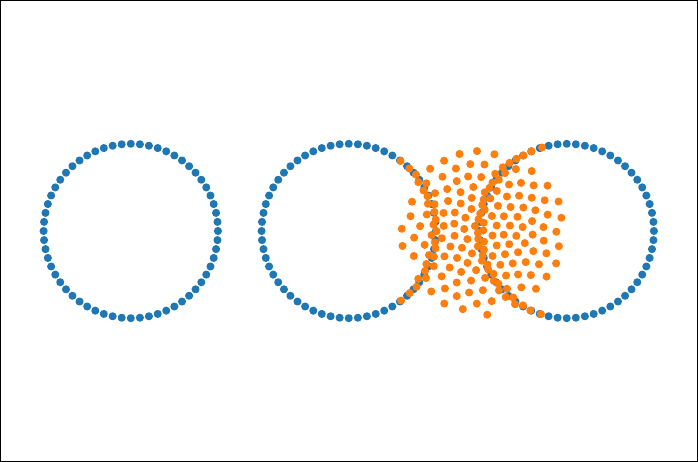} &
				\includegraphics[width=3cm]{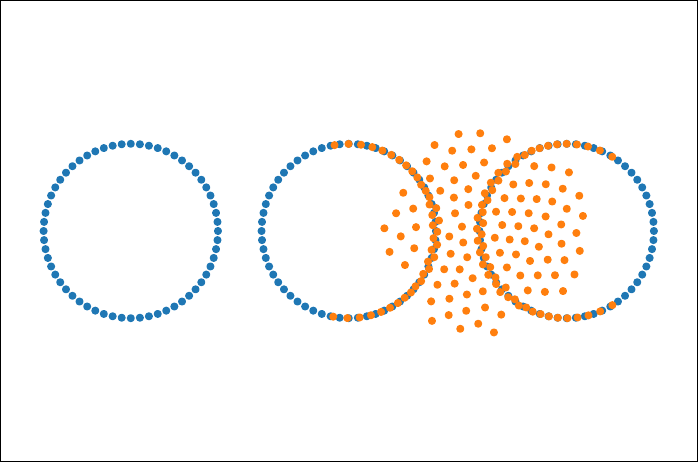} &
				\includegraphics[width=3cm]{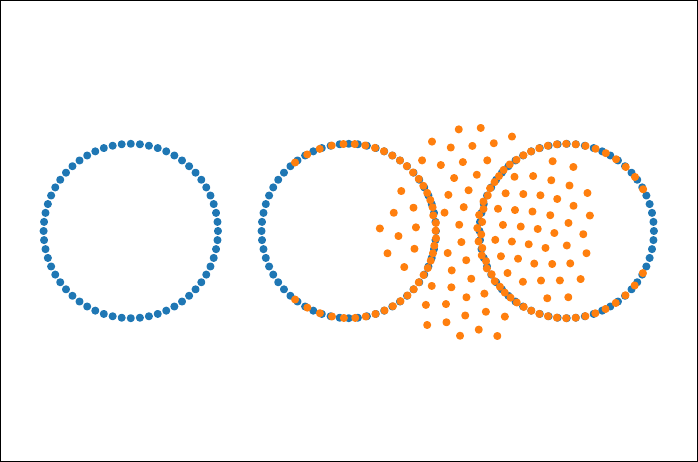} &
				\includegraphics[width=3cm]{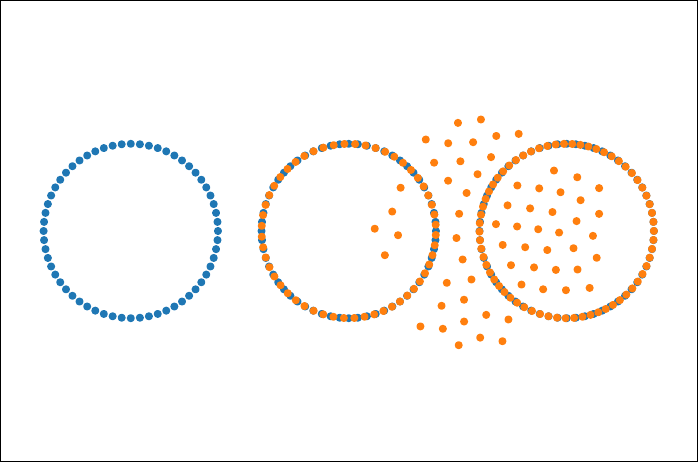} \\
				\vspace{-4pt}
				\rotatebox{90}{\parbox{2cm}{\centering $\sigma=0.3$}} &
				\includegraphics[width=3cm]{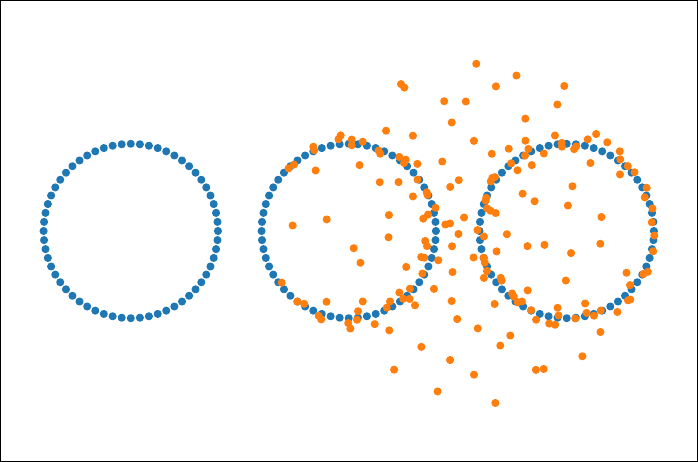} &
				\includegraphics[width=3cm]{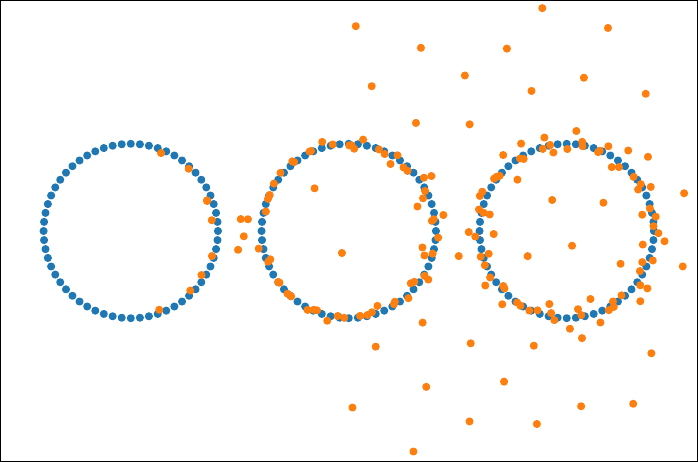} &
				\includegraphics[width=3cm]{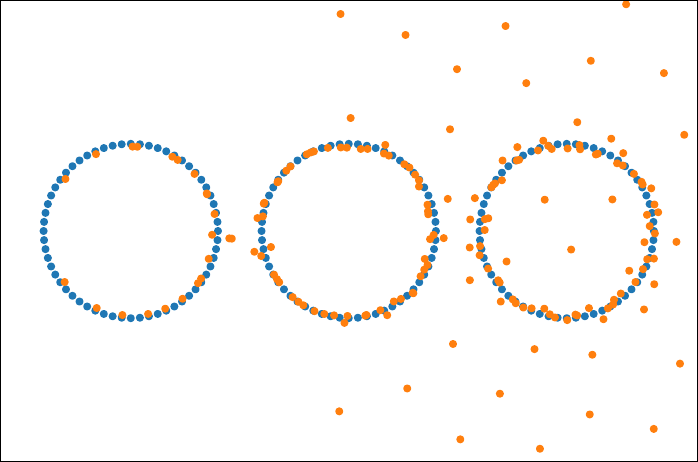} &
				\includegraphics[width=3cm]{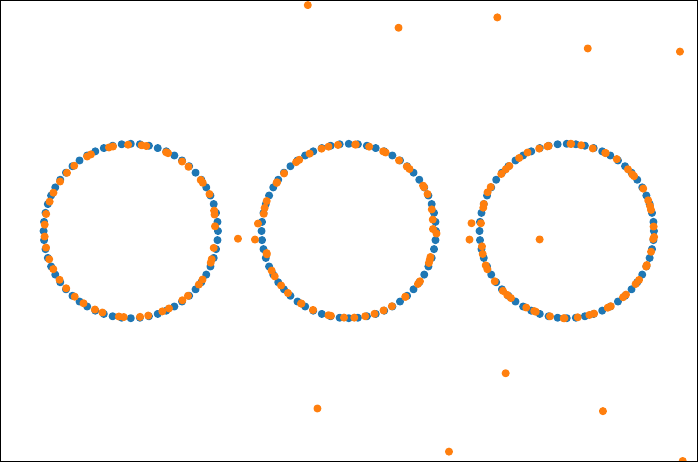} \\
				\vspace{-2pt}
				\rotatebox{90}{\parbox{2cm}{\centering $\sigma=1$}} &
				\includegraphics[width=3cm]{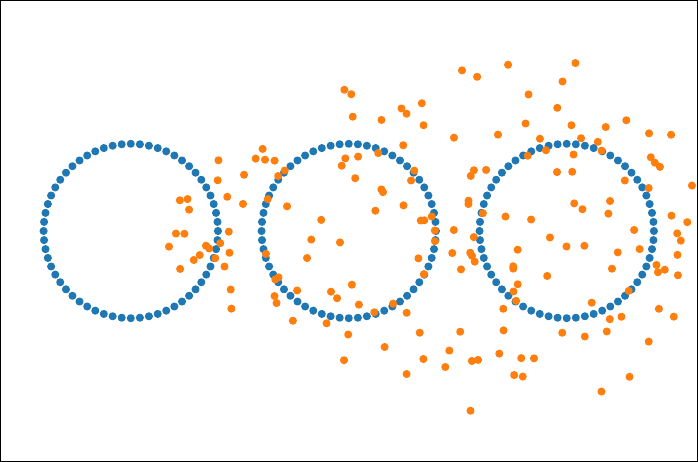} &
				\includegraphics[width=3cm]{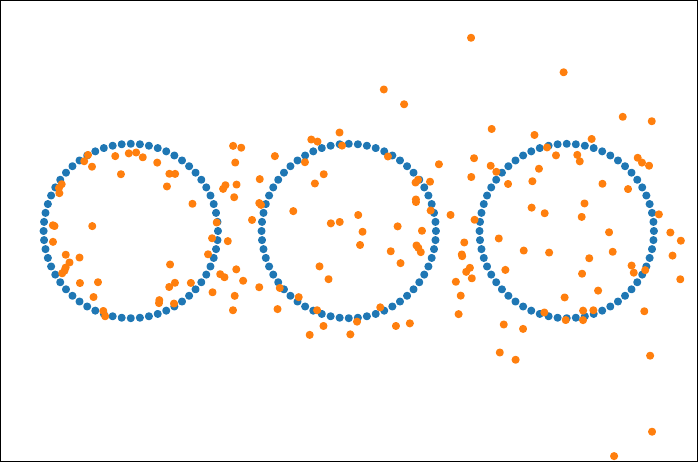} &
				\includegraphics[width=3cm]{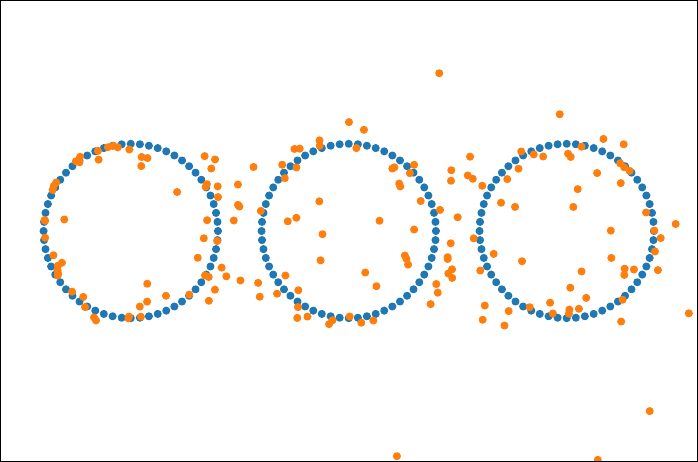} &
				\includegraphics[width=3cm]{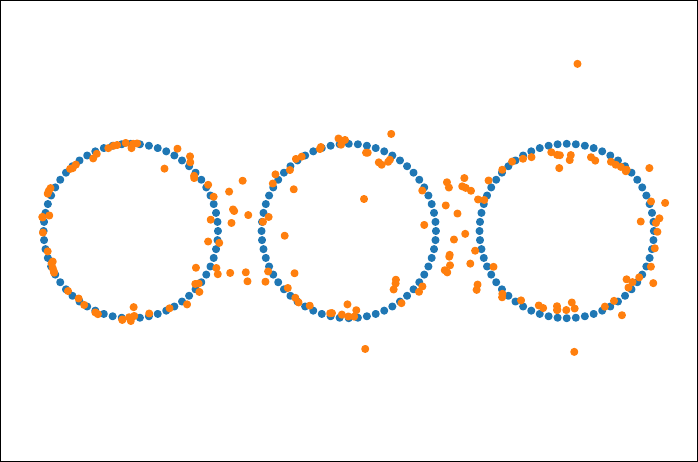} \\

				& \multicolumn{1}{c}{$t=10$} & \multicolumn{1}{c}{$t=50$} & \multicolumn{1}{c}{$t=100$} & \multicolumn{1}{c}{$t=500$} \\
			\end{tabular}
			\caption{Gaussian}
			\label{fig:3Rings_GAUSS}
		\end{subfigure}
		
		\begin{subfigure}{\textwidth}
			\centering\footnotesize
			\begin{tabular}{c c c c c}
				\vspace{-4pt}
				\rotatebox{90}{\parbox{2cm}{\centering $\varepsilon=1$}} &
				\includegraphics[width=3cm]{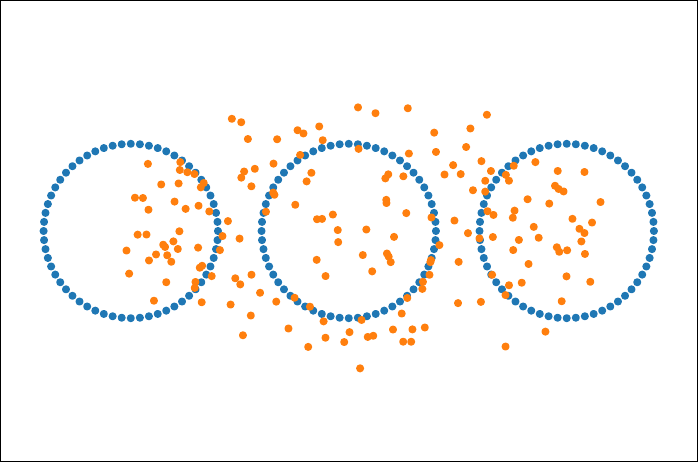} &
				\includegraphics[width=3cm]{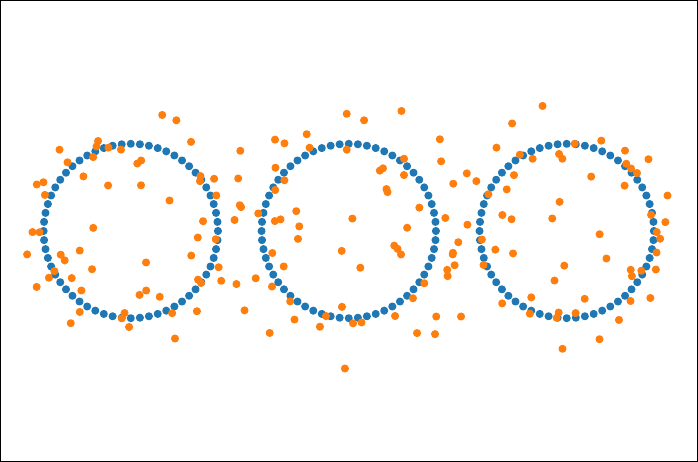} &
				\includegraphics[width=3cm]{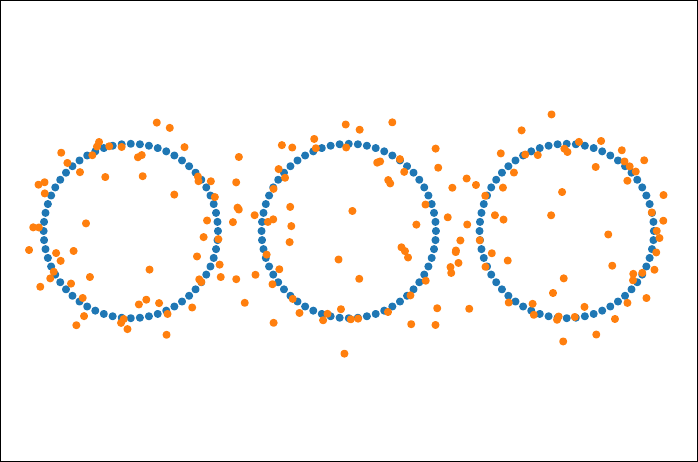} &
				\includegraphics[width=3cm]{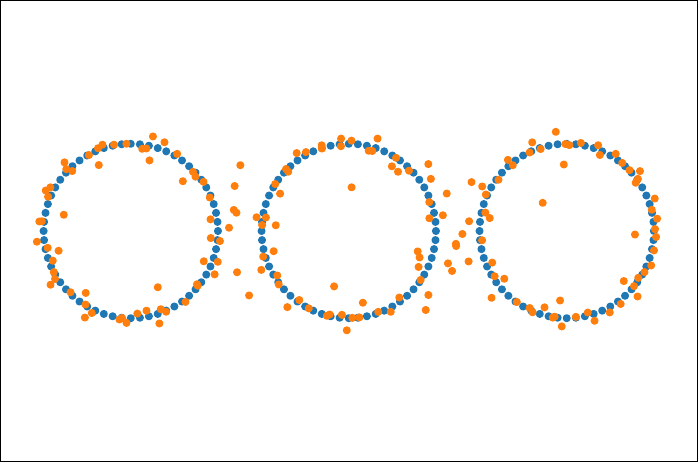} \\
				\vspace{-4pt}
				
				\rotatebox{90}{\parbox{2cm}{\centering $\varepsilon=0.1$}} &
				\includegraphics[width=3cm]{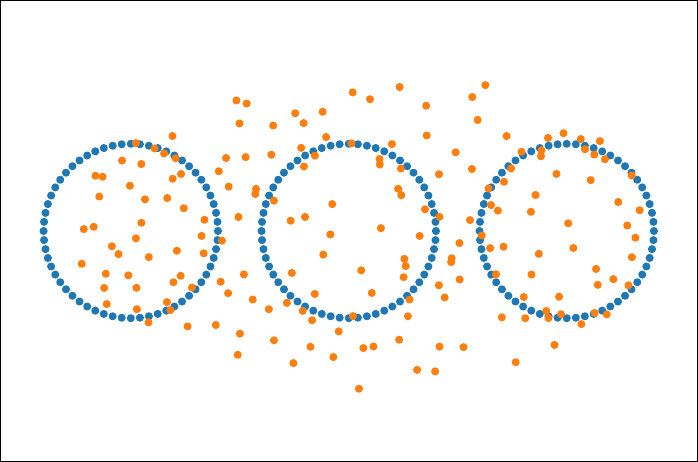} &
				\includegraphics[width=3cm]{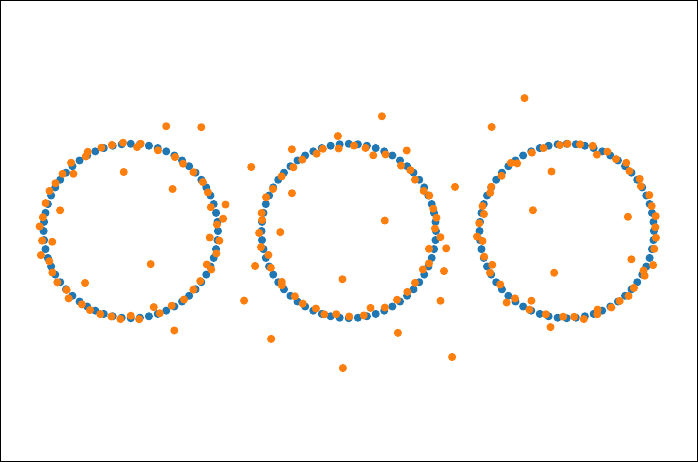} &
				\includegraphics[width=3cm]{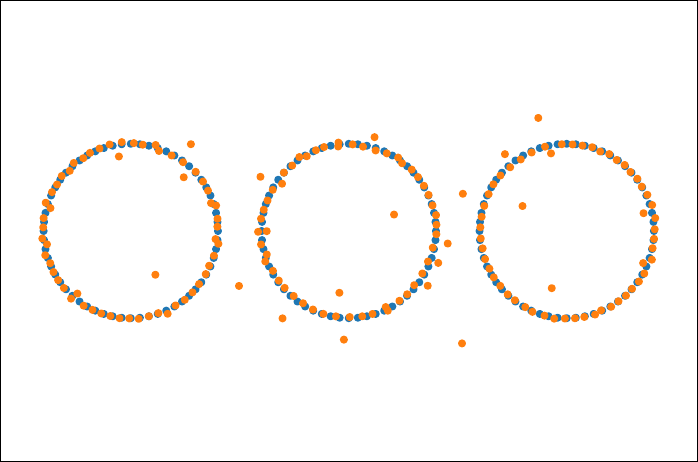} &
				\includegraphics[width=3cm]{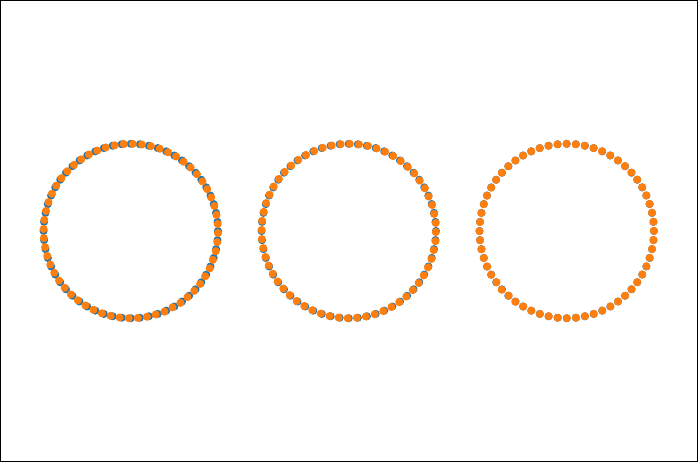} \\
				\vspace{-2pt}
				
				\rotatebox{90}{\parbox{2cm}{\centering $\varepsilon=0.01$}} &
				\includegraphics[width=3cm]{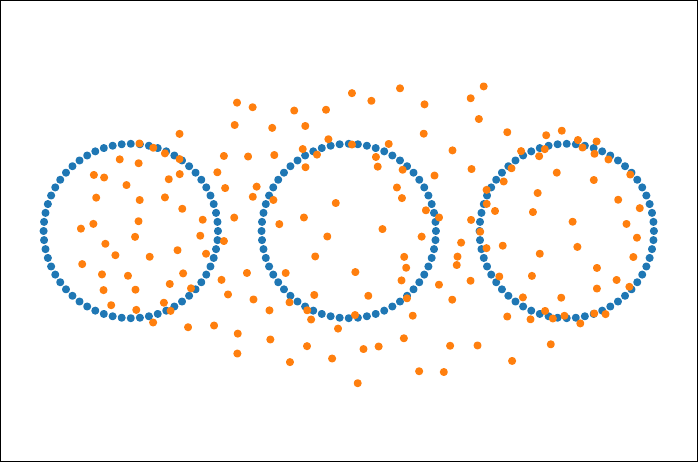} &
				\includegraphics[width=3cm]{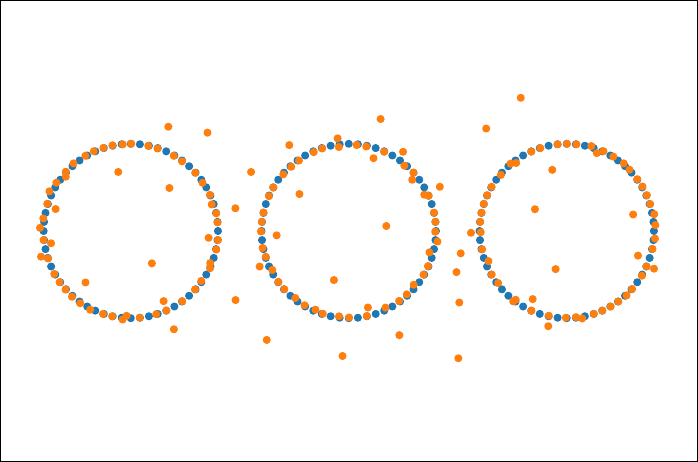} &
				\includegraphics[width=3cm]{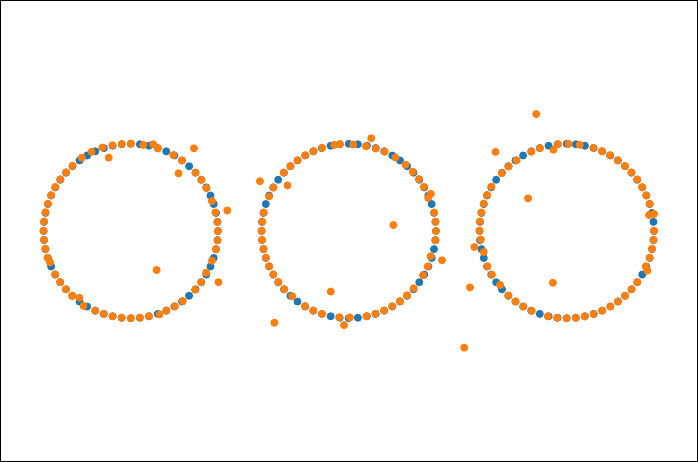} &
				\includegraphics[width=3cm]{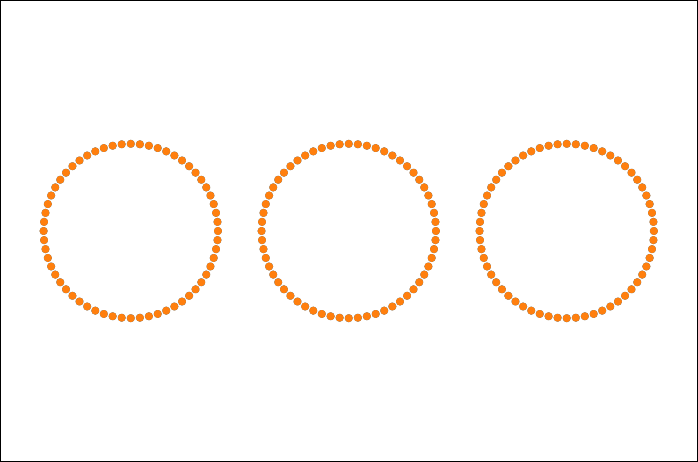} \\
				
				& \multicolumn{1}{c}{$t=10$} & \multicolumn{1}{c}{$t=50$} & \multicolumn{1}{c}{$t=100$} & \multicolumn{1}{c}{$t=500$} \\
			\end{tabular}
			\caption{SND}
			\label{fig:3Rings_SND}
		\end{subfigure}
		
		\begin{subfigure}{\textwidth}
			\centering\footnotesize
			\begin{tabular}{c c c c c} 
				\rotatebox{90}{\parbox{2cm}{}} &
				\includegraphics[width=3cm]{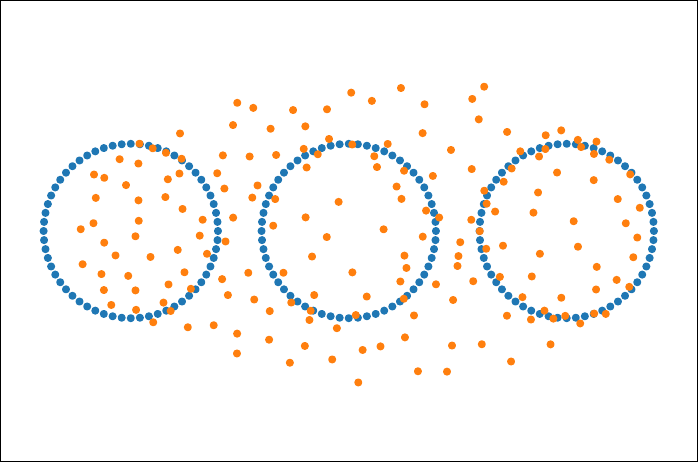} &
				\includegraphics[width=3cm]{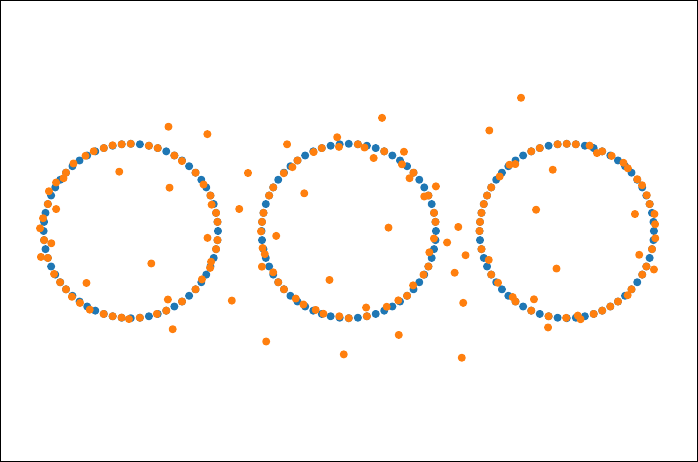} &
				\includegraphics[width=3cm]{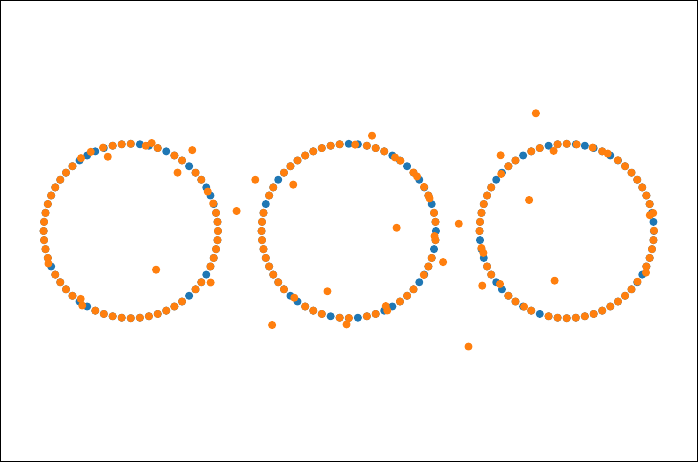} &
				\includegraphics[width=3cm]{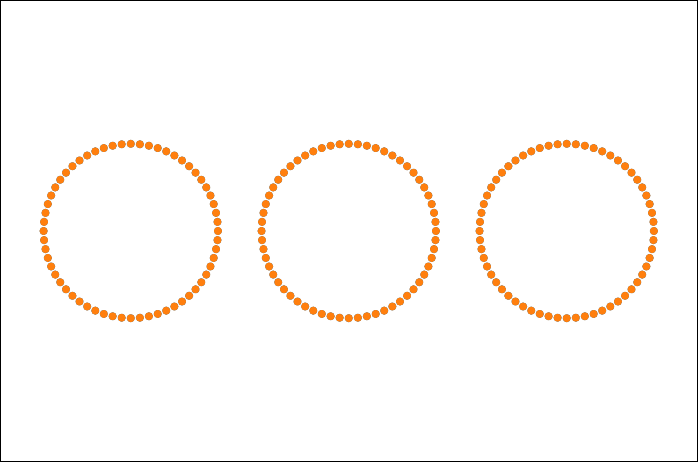} \\
				
				& \multicolumn{1}{c}{$t=10$} & \multicolumn{1}{c}{$t=50$} & \multicolumn{1}{c}{$t=100$} & \multicolumn{1}{c}{$t=500$} \\
			\end{tabular}
			\caption{ND}
			\label{fig:3Rings_ND}
		\end{subfigure}
		
		\caption{MMD flow \eqref{eq:mmd-flow} with step size $\tau=0.01$. For the Gaussian kernel, the result depends heavily on the choice of the parameter $\sigma$. For our SND kernel with small $\varepsilon$, the performance is as good as for the ND kernel, which is better than for the Gaussians.}
		\label{fig:3Rings_Flow}
	\end{figure}
	
	In Figure~\ref{fig:3Rings_W2}, we plot the Wasserstein error $W_2(\gamma_t^\tau, \nu)$ between the Three-Rings target measure $\nu$ and the discretized Wasserstein gradient flow $\gamma_t^\tau$ at time $t$ computed with PythonOT \cite{flamary2021pot}.  
	The first plot in Figure~\ref{fig:3Rings_W2} corresponds to the flow $\gamma_t^\tau$ shown in Figures \ref{fig:3Rings_GAUSS}, \ref{fig:3Rings_SND}, and \ref{fig:3Rings_ND}.  
	The remaining plots in Figure~\ref{fig:3Rings_W2} depict the same experiment with different step sizes $\tau$ and machine precision, where we always used the same random seed.  
	Regardless of precision, step size, or bandwidth $\sigma$, the Gauss kernel stagnates away from the target measure $\nu$.
	
	Proposition \ref{prop:dirac_flow} gives an intuition for the behavior of the ND and SND kernel close to the target.
	For the ND kernel, Proposition \ref{prop:dirac_flow} i) indicates that $\mu_t^\tau$ oscillates around the target for any $\tau > 0$ without convergence. 
	In contrast, Proposition \ref{prop:dirac_flow} ii) states that for the SND kernel, $W_2(\mu_t^\tau, \nu)$ decays exponentially if $\tau$ is sufficiently small.
	Numerically, Figure~\ref{fig:3Rings_W2} confirms this behavior. In single precision, ND and SND with $\varepsilon = 0.01$ plateau at $\approx  10^{-3}$.
	With double precision, ND oscillates at the same error, while SND drops to ${\approx 10^{-7}}$. Thus, SND matches ND globally but exhibits better local convergence for fixed $\tau > 0$ due to its smoothness.
	
	In Appendix~\ref{app:annulus}, we 
	consider the SND with $M_4$ instead of $M_2$,
	provide an additional example with two concentric circles,
	and report computation times.

	\begin{figure}[!htp]
		\centering\small
		\begin{tabular}{c c c}
			\rotatebox{90}{\parbox{5cm}{\centering float32}} &
			\includegraphics[width=6cm]{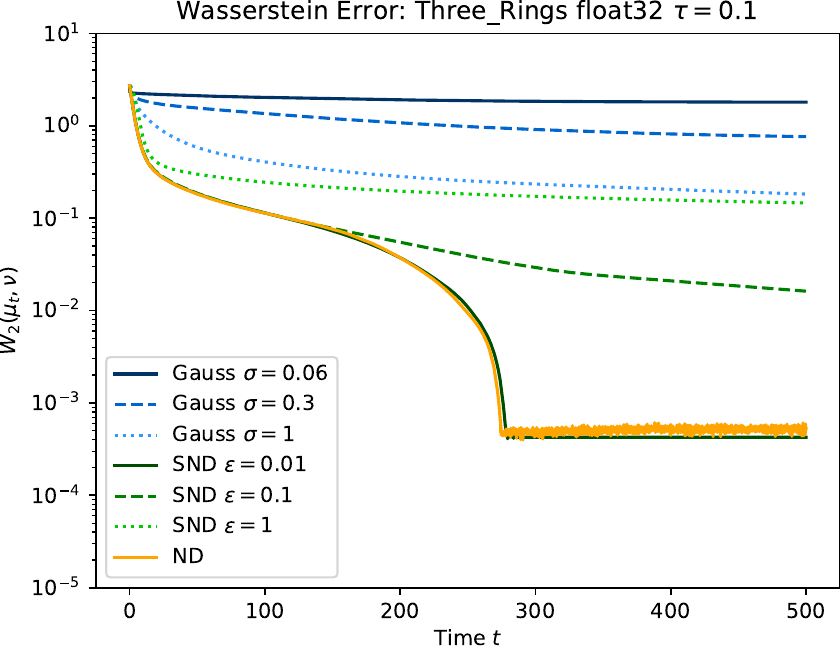} &
			\includegraphics[width=6cm]{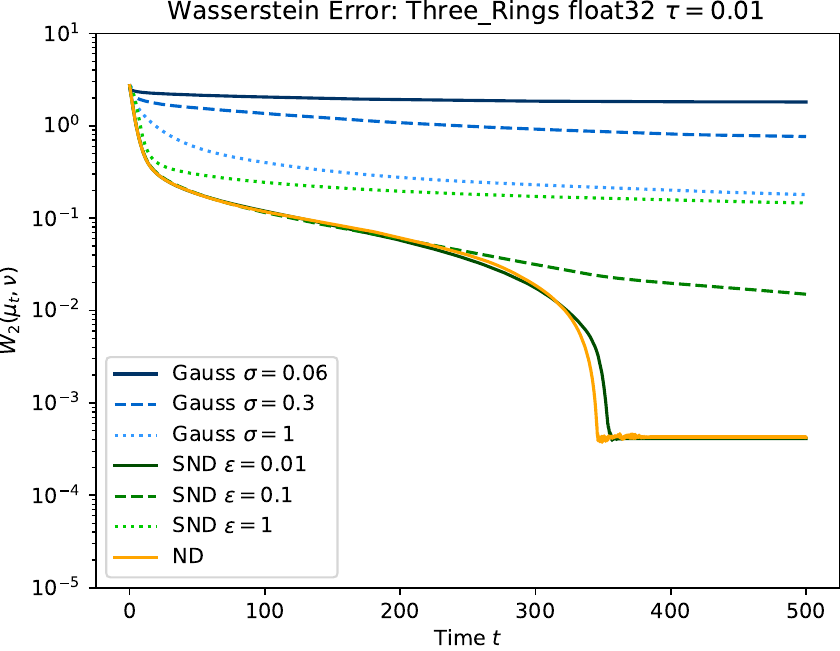} \\
			
			\rotatebox{90}{\parbox{5cm}{\centering float64}} &
			\includegraphics[width=6cm]{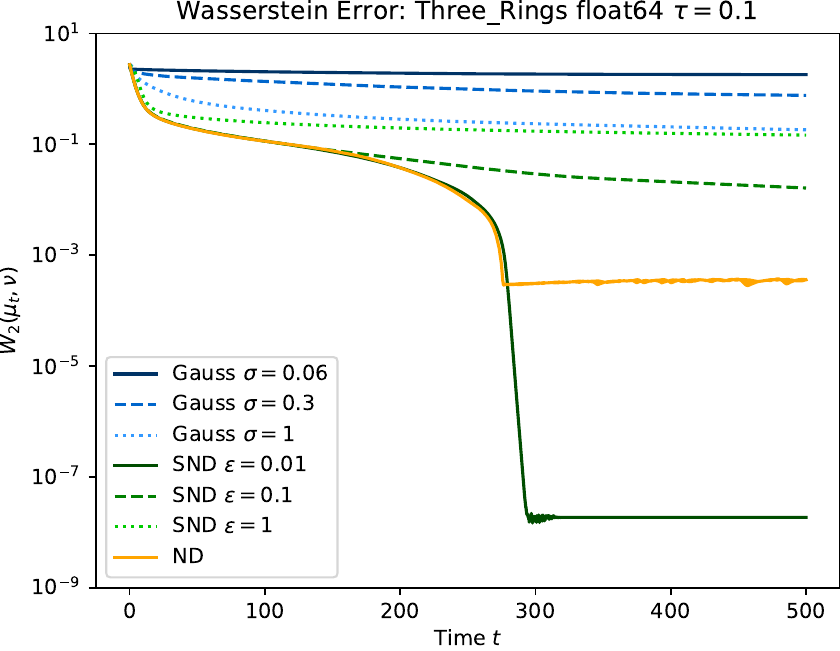}&
			\includegraphics[width=6cm]{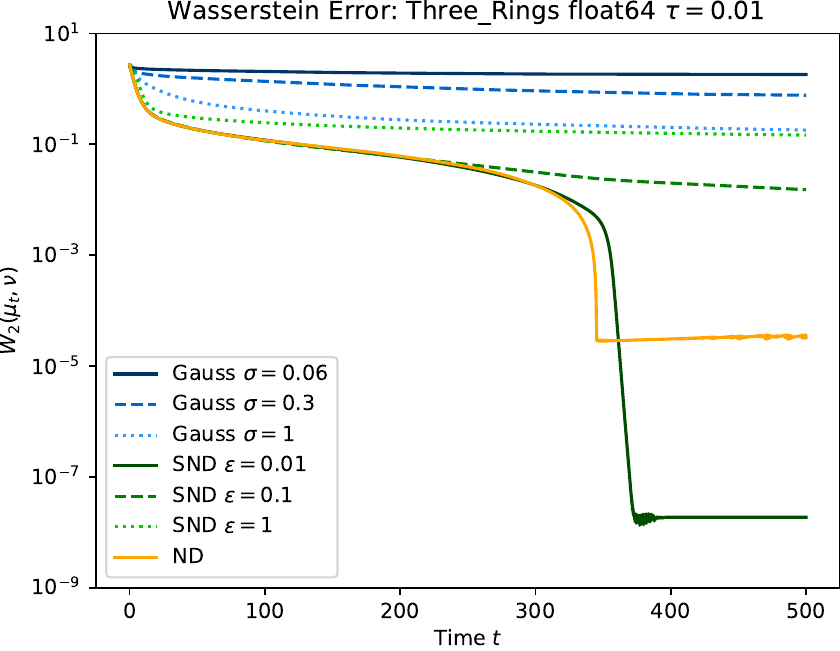}\\
			
			& \multicolumn{1}{c}{$\tau =0.1$} & \multicolumn{1}{c}{$\tau =0.01$}
		\end{tabular}
		\caption{$W_2$ error between Three-Rings target $\nu$ and flow $\gamma_t^\tau$ after $k$ with time $t=\tau k$. We compare single precision (first row) and double precision (second row) for step sizes $\tau=0.1$ (left) and $\tau=0.01$ (right).
			In single precision, SND with $\varepsilon = 0.01$ and ND have the smallest error which gets stuck in $\approx 10^{-3}$. In double precision, SND with $\varepsilon = 0.01$ even outperforms ND. For some explanation, see Proposition \ref{prop:dirac_flow}.  }
		\label{fig:3Rings_W2}
	\end{figure}
	
	\subsubsection{Bananas Target}
	The Bananas target $\nu$ in Figure~\ref{fig:Bananas} is inspired by the talk from Aude Genevay\footnote{MIFODS Workshop on Learning with Complex Structure 2020, see \url{https://youtu.be/TFdIJib_zEA}.} and its implementation by Viktor Stein\footnote{\href{https://github.com/ViktorAJStein/Regularized_f_Divergence_Particle_Flows}{https://github.com/ViktorAJStein/Regularized\_f\_Divergence\_Particle\_Flows}.}.
	The target consists of two banana shaped clusters in $\R^2$, where each banana consists of $100$ points, so $M= 200$.
	
	We compute the flows with step size $\tau=0.02$ in double precision for the Gauss kernel with $\sigma\in \{0.06,0.3,1\}$, the SND kernel with $\varepsilon\in \{0.1,0.01,0.001\}$, and the ND kernel, see Figure~\ref{fig:Bananas_flow}. For small $\sigma=0.06$ the Gauss kernel struggles to reach the bananas. When $\sigma=0.3$, the right banana is reached, but some particles blow up and leave the frame. For $\sigma=1$, the flows reaches the bananas, but collapses in the modes and do not recover the structure of the target.
	In contrast, the SND flow always manages to reach both bananas without blowing up, while a smaller $\varepsilon$ again gives more desirable results.
	The respective Wasserstein errors in Figure~\ref{fig:Bananas_W2} show a similar behavior as for the rings.

	\begin{figure}
		\centering
		\begin{subfigure}{\textwidth}
			\centering\footnotesize
			\begin{tabular}{c c c c c}
				\vspace{-4pt}
				\rotatebox{90}{\parbox{2cm}{\centering $\sigma=0.06$}} &
				\includegraphics[width=3cm]{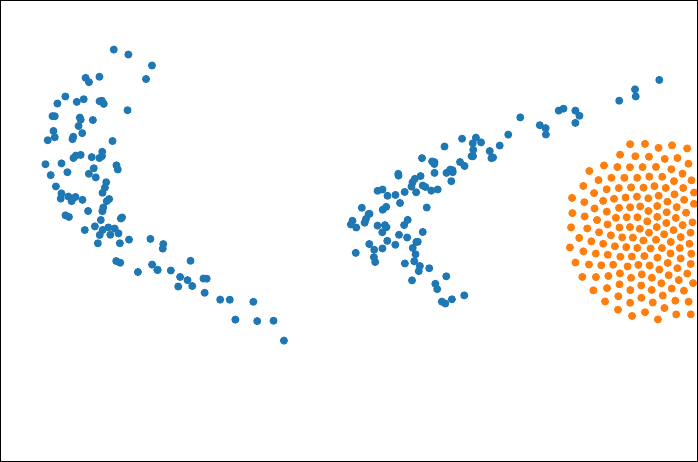} &
				\includegraphics[width=3cm]{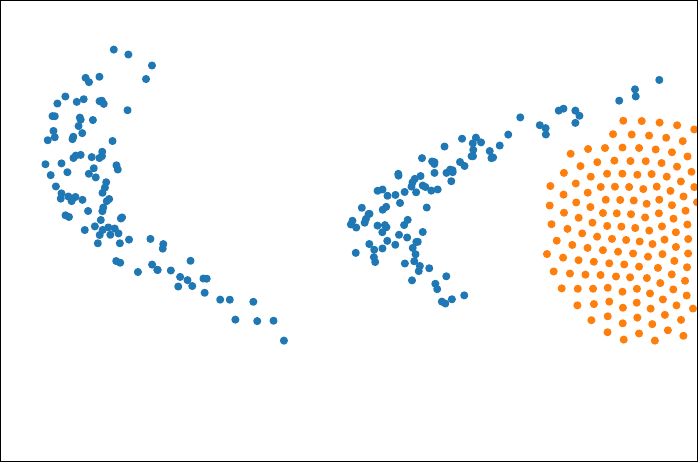} &
				\includegraphics[width=3cm]{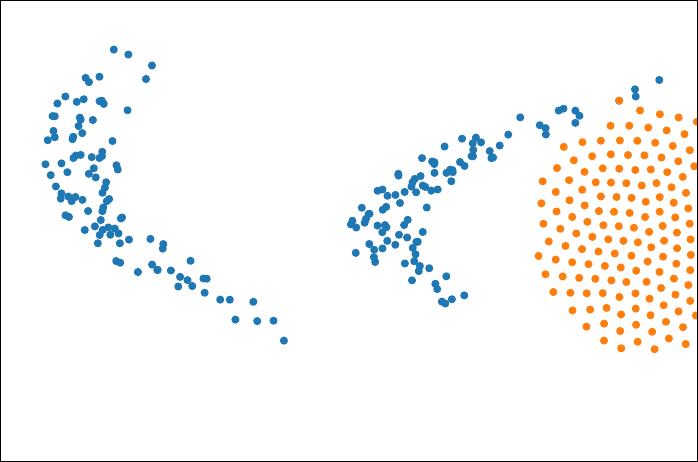} &
				\includegraphics[width=3cm]{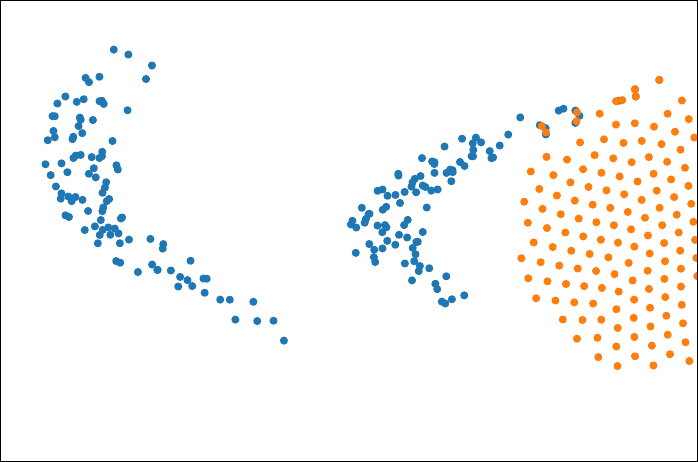} \\
				\vspace{-4pt}
				
				\rotatebox{90}{\parbox{2cm}{\centering $\sigma=0.3$}} &
				\includegraphics[width=3cm]{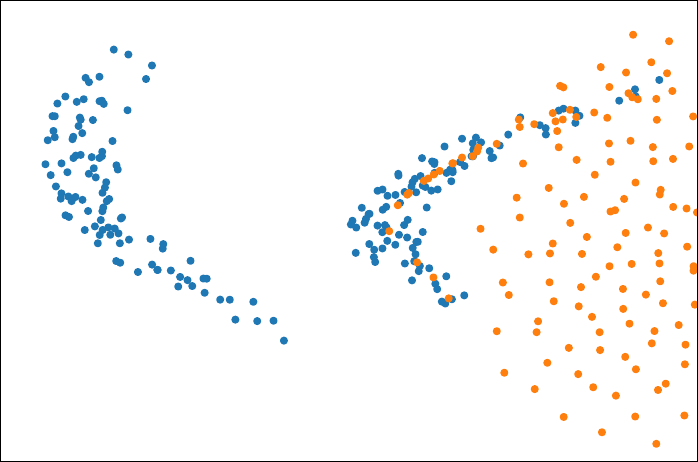} &
				\includegraphics[width=3cm]{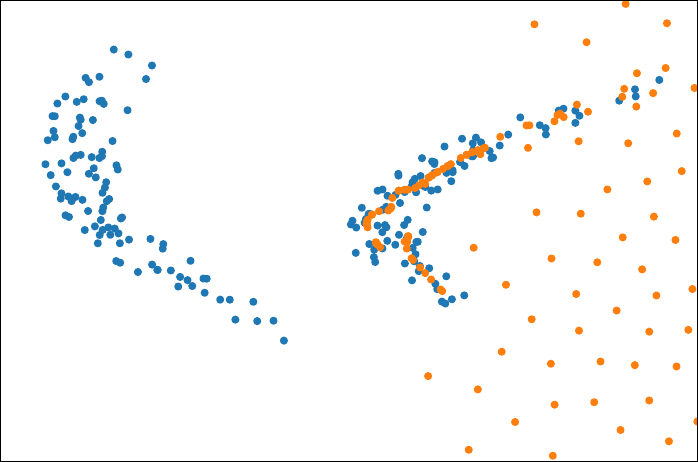} &
				\includegraphics[width=3cm]{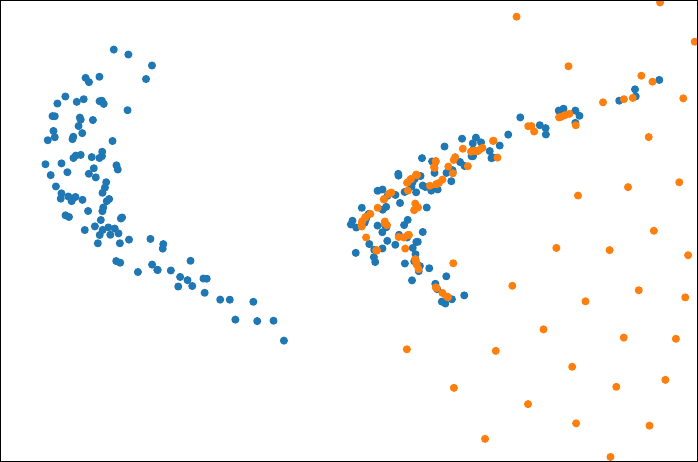} &
				\includegraphics[width=3cm]{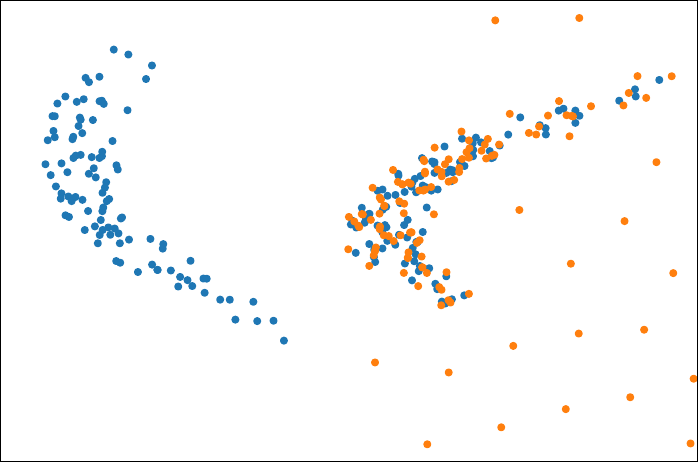} \\
				\vspace{-2pt}

				\rotatebox{90}{\parbox{2cm}{\centering $\sigma=1$}} &
				\includegraphics[width=3cm]{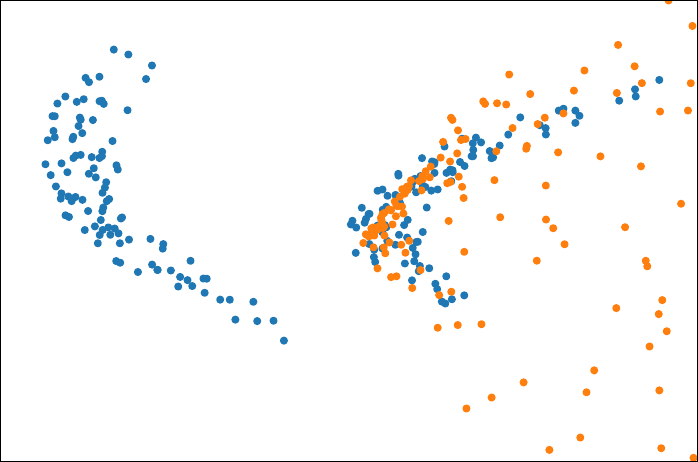} &
				\includegraphics[width=3cm]{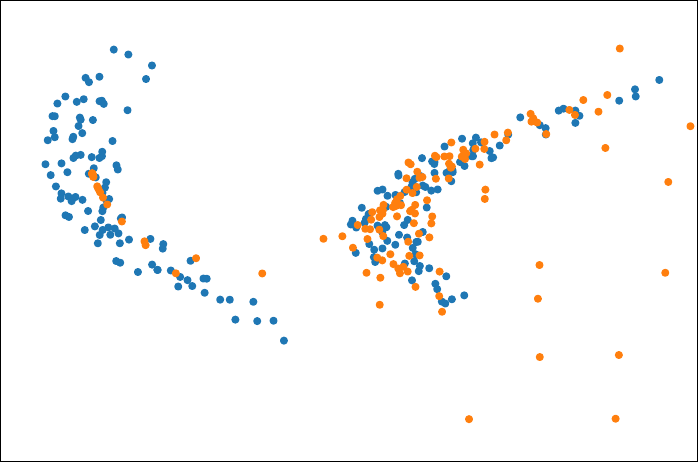} &
				\includegraphics[width=3cm]{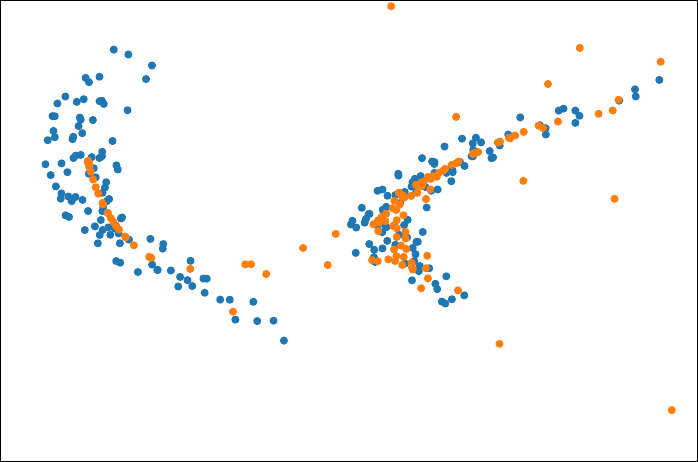} &
				\includegraphics[width=3cm]{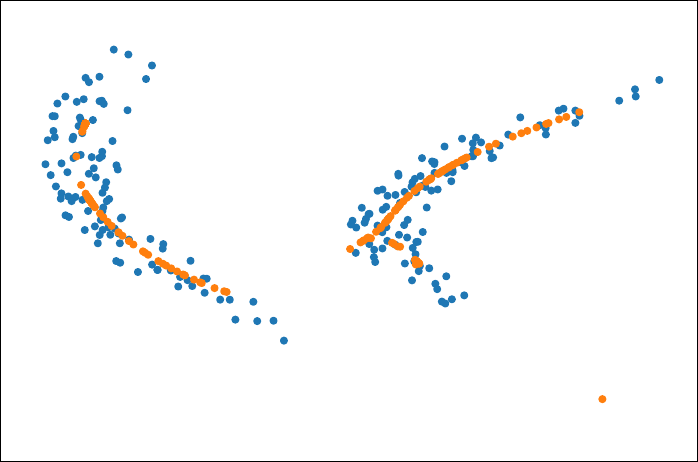} \\

				& \multicolumn{1}{c}{$t=20$} & \multicolumn{1}{c}{$t=100$} & \multicolumn{1}{c}{$t=200$} & \multicolumn{1}{c}{$t=1\,000$} \\
			\end{tabular}
			\caption{Gaussian}
			\label{fig:Bananas_GAUSS}
		\end{subfigure}
		
		\begin{subfigure}{\textwidth}
			\centering\footnotesize
			\begin{tabular}{c c c c c}  %
				\vspace{-4pt}
				\rotatebox{90}{\parbox{2cm}{\centering $\varepsilon=1$}} &
				\includegraphics[width=3cm]{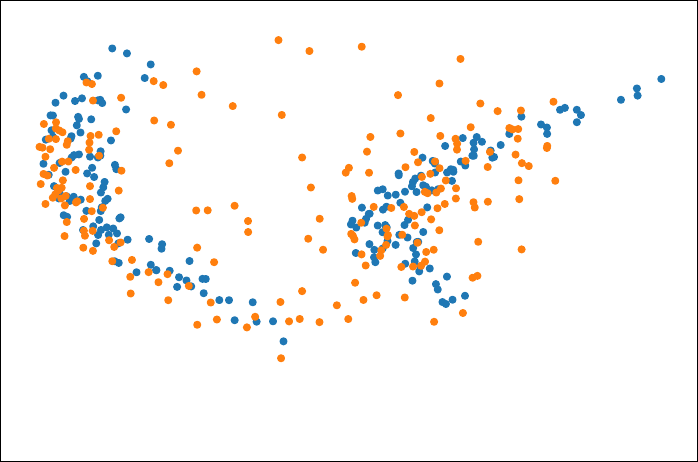} &
				\includegraphics[width=3cm]{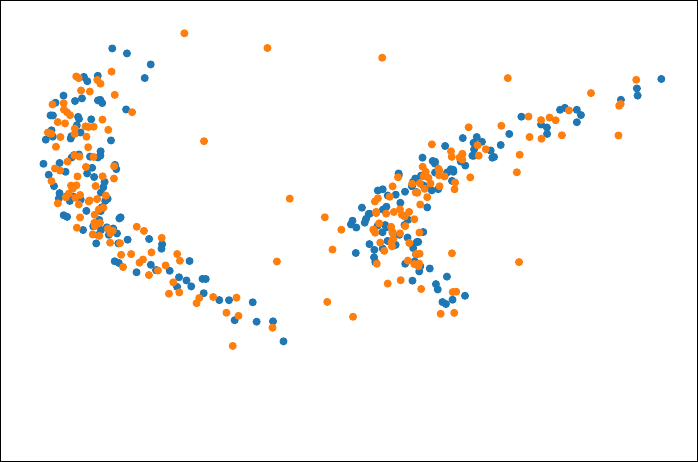} &
				\includegraphics[width=3cm]{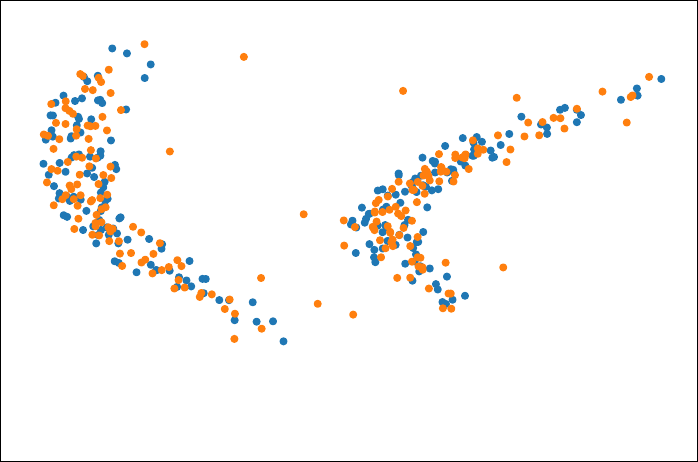} &
				\includegraphics[width=3cm]{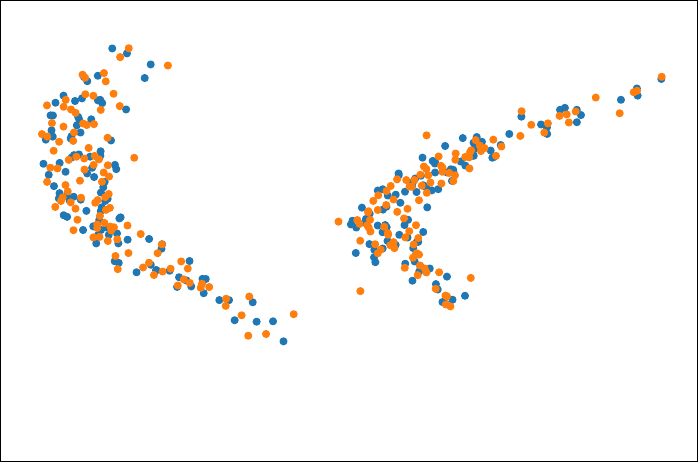} \\
				\vspace{-4pt}
				
				\rotatebox{90}{\parbox{2cm}{\centering $\varepsilon=0.1$}} &
				\includegraphics[width=3cm]{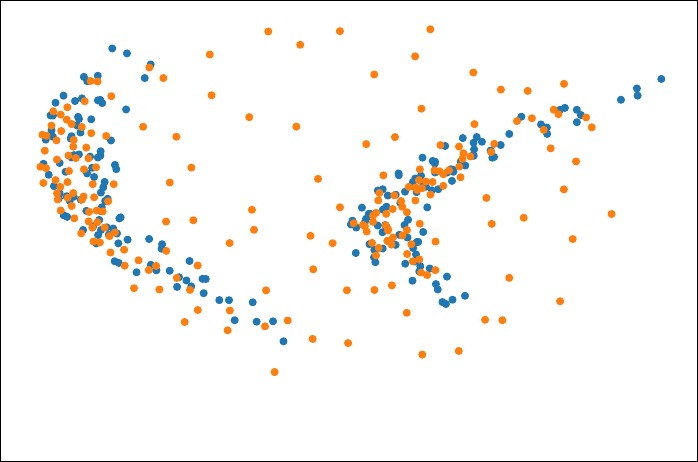} &
				\includegraphics[width=3cm]{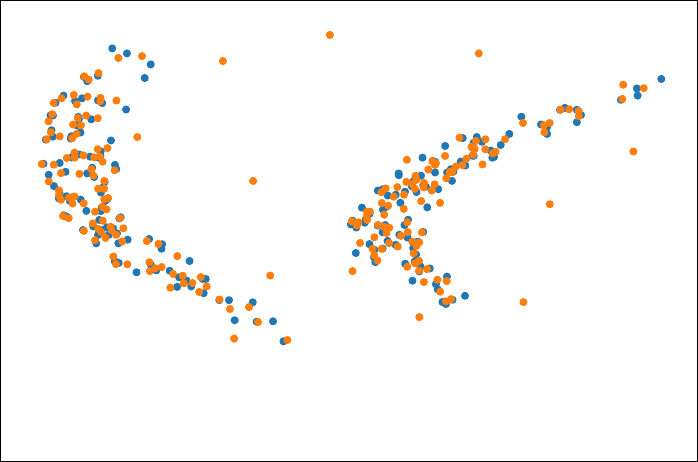} &
				\includegraphics[width=3cm]{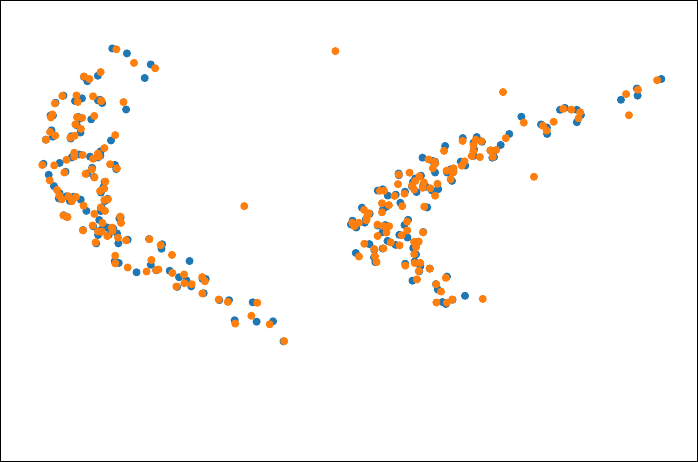} &
				\includegraphics[width=3cm]{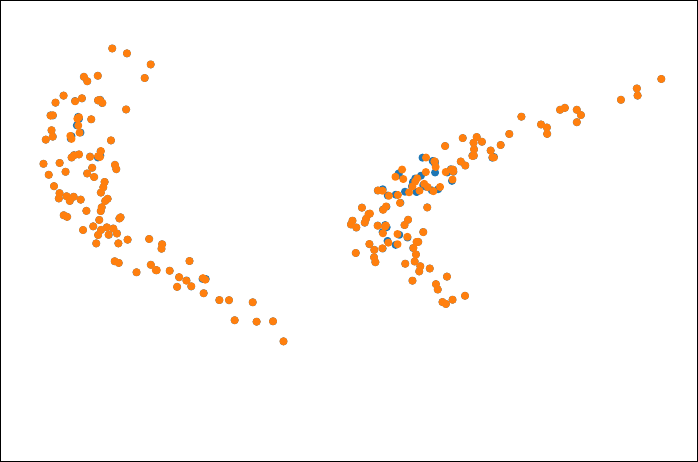} \\
				\vspace{-2pt}
				
				\rotatebox{90}{\parbox{2cm}{\centering $\varepsilon=0.01$}} &
				\includegraphics[width=3cm]{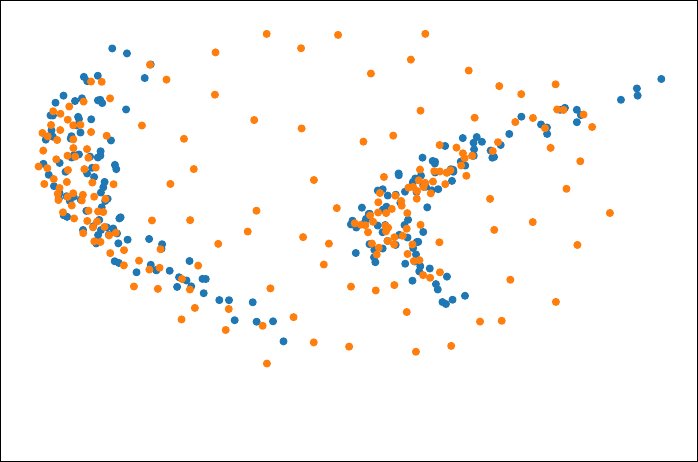} &
				\includegraphics[width=3cm]{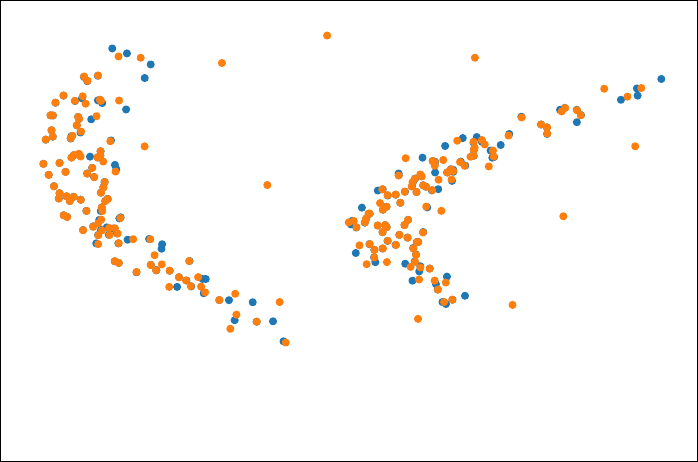} &
				\includegraphics[width=3cm]{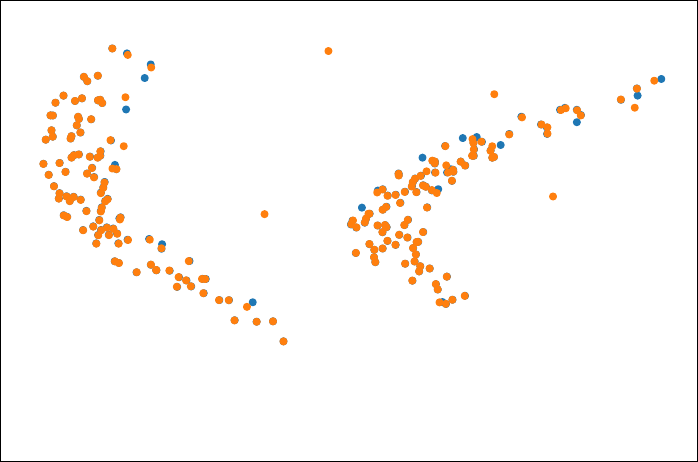} &
				\includegraphics[width=3cm]{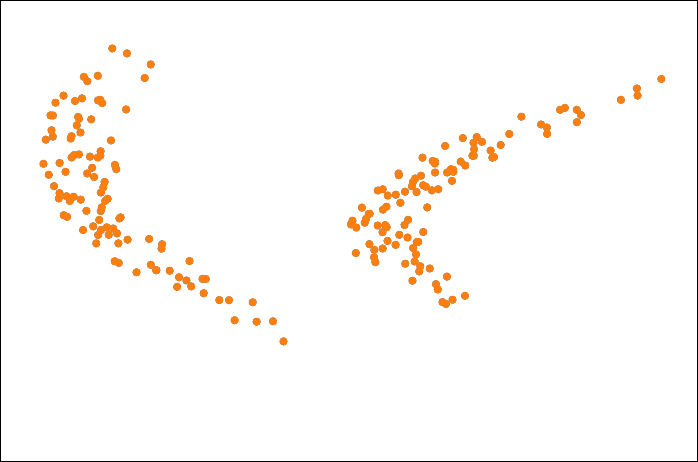} \\

				& \multicolumn{1}{c}{$t=20$} & \multicolumn{1}{c}{$t=100$} & \multicolumn{1}{c}{$t=200$} & \multicolumn{1}{c}{$t=1\,000$} \\
			\end{tabular}
			\caption{SND}
			\label{fig:Bananas_SND}
		\end{subfigure}
		
		\begin{subfigure}{\textwidth}
			\centering\footnotesize
			\begin{tabular}{c c c c c}  
				\rotatebox{90}{\parbox{2cm}{}} &
				\includegraphics[width=3cm]{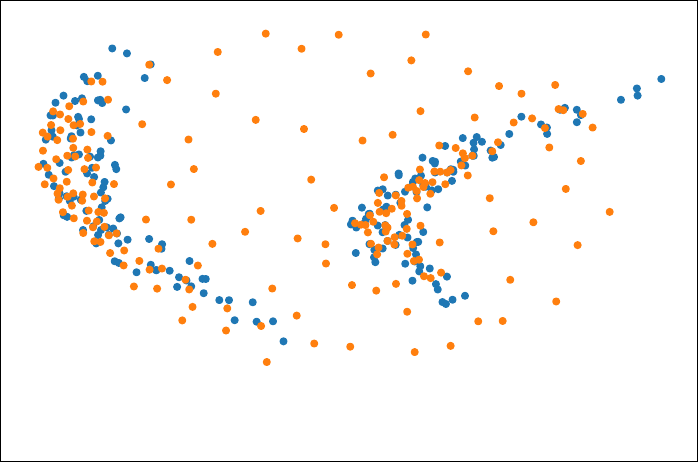} &
				\includegraphics[width=3cm]{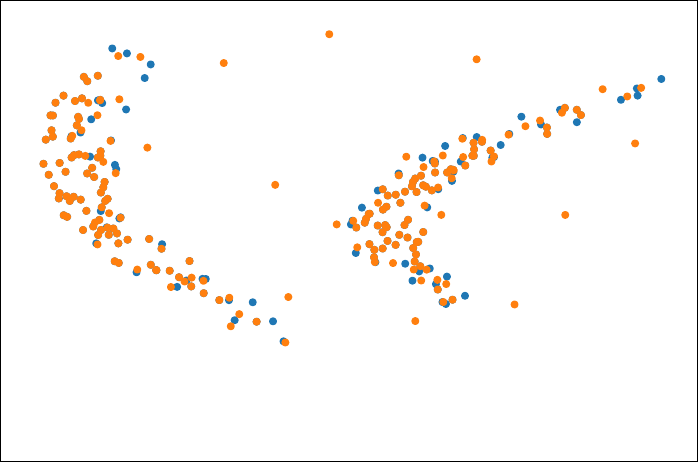} &
				\includegraphics[width=3cm]{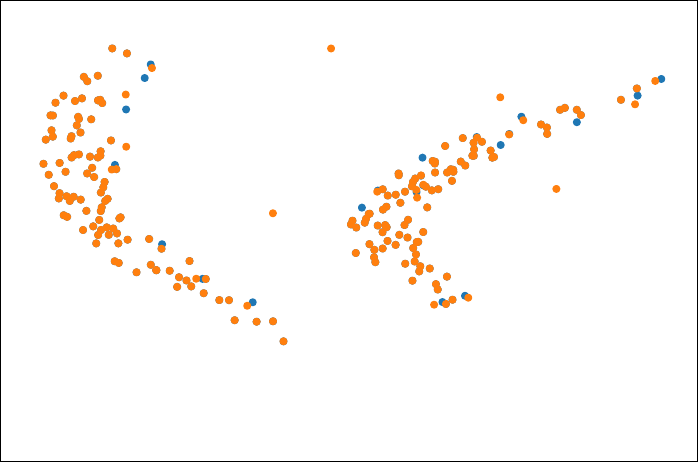} &
				\includegraphics[width=3cm]{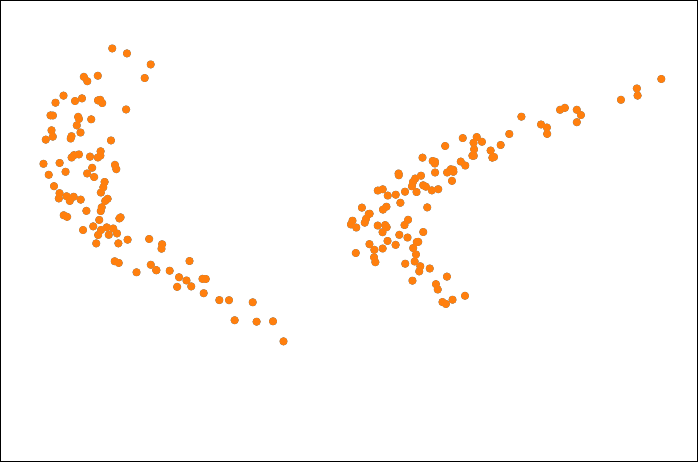} \\
				
				& \multicolumn{1}{c}{$t=20$} & \multicolumn{1}{c}{$t=100$} & \multicolumn{1}{c}{$t=200$} & \multicolumn{1}{c}{$t=1\,000$} \\
			\end{tabular}
			\caption{ND}
			\label{fig:Bananas_ND}
		\end{subfigure}
		\caption{MMD flow \eqref{eq:mmd-flow} with step size $\tau=0.02$. For the Gaussian kernel, the result depends heavily on the choice of the parameter $\sigma$. For our SND kernel with small $\varepsilon$, the performance is as good as for the ND kernel, which is better than for the Gaussians.}
		\label{fig:Bananas_flow}
	\end{figure}

	\begin{figure}
		\centering
		\includegraphics[width=6cm]{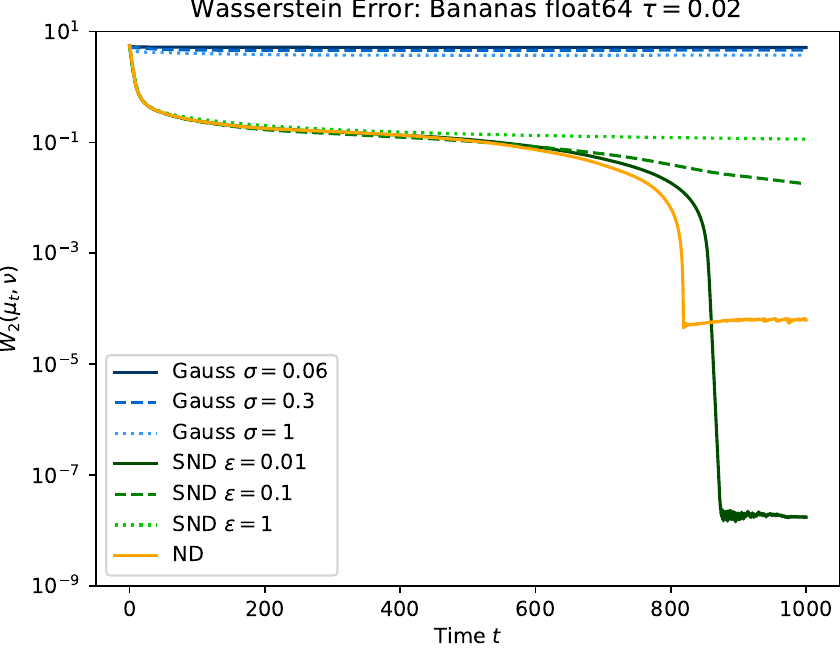}
		\caption{$W_2$ error between Bananas target $\nu$ and flow  $\gamma_k^\tau$ after $k$ iterations with time $t=\tau k$. Computation with double precision with step size $\tau=0.02$ shows a similar behavior as in Figure \ref{fig:3Rings_W2}. }
		\label{fig:Bananas_W2}
	\end{figure}

	\subsection{MNIST Dataset}
	We consider as target the MNIST dataset, where each $28\times28$ image is considered a point $y\in\R^d$ with $d=784=28^2$.
	We use $N=M=100$ images as flow and target.
	\paragraph{Fast Summation by Slicing.}
	The computation of \eqref{eq:mmd-flow} includes the summation of kernel values of the form $s_m=\sum_{n=1}^N w_n F'(\|x_n-y_m\|)$ for $m=1,\dots,M$ with some weights $w_n\in\C$.
	This summation requires $O(NM)$ arithmetic operations.
	If $F=\mathcal I_d[f]$, then we have by \eqref{eq:slicingRLFI} that
	$
	F'(x) = \E_{\xi\sim\mathcal U_{\mathbb S^{d-1}}} [\xi f'(\langle x,\xi\rangle)].
	$
	In order to speed up the computation, the sum $s_m$ can be approximated by slicing \cite{hertrich2024,HWAH2024} via
	\begin{equation} \label{eq:slicing-sum}
		s_m
		=
		\E_{\xi\sim\mathcal U_{\mathbb S^{d-1}}} \Big[ \sum_{n=1}^N w_n \xi f'(\langle x_n-y_m,\xi\rangle) \Big]
		\approx
		\frac{1}{P} \sum_{p=1}^P \xi_p \sum_{n=1}^N w_n f'(\langle x_n-y_m,\xi_p\rangle),
	\end{equation}
	where $(\xi_p)_{p=1}^P\in(\mathbb S^{d-1})^P$ are equidistributed quadrature nodes on $\mathbb S^{d-1}$.
	This is a collection of $P$ one-dimensional kernel sums. Each of them can be computed efficiently in $O((N+M)\log(N+M))$ operations, e.g.\ via fast Fourier summation \cite{plonka2018numerical} or, if $F$ is the ND kernel just by sorting \cite{hertrich2024}.
	Hence, \eqref{eq:slicing-sum} is more efficient if $P$ is considerably smaller than the number of points.
	More details on the slicing summation and errors estimates are provided in \cite{hertrich2025fast}.
	Furthermore, for our SND kernel, the explicit expression for $F$ in Proposition~\ref{prop:id_spline} is somewhat cumbersome for large dimension~$d$, while the slicing summation \eqref{eq:slicing-sum} only requires to evaluate $f'$.

	\paragraph*{Setup.
	}
	The initialization $x^{(0)}_n$, $n=1,\dots,N$, are iid samples from a uniform distribution on $[0,1]^d$.
	We compute the MMD flows \eqref{eq:mmd-flow} with $2^{15}=32768$ iterations for the SND kernel $F = - C_{784} \mathcal I_{784}[\abs*M_{2,\varepsilon}]$ with $\varepsilon\in\{0.001,0.01,0.1\}$ and the ND kernel $F=- \abs$.
	The step size is $\tau=1$ and the computations are performed in single precision.
	We use slicing summation \eqref{eq:slicing-sum} with $P=785$ directions $\xi_p$ that are the vertices of the centrally symmetric simplex, to which we apply a random rotation in each iteration step, cf.\ \cite{hertrich2025fast}.
	The slicing summation requires only the sliced kernel $\abs*M_{2,\varepsilon}$ given in \eqref{eq:abs*M2}, but not the representation of $F$, which becomes quite clumsy in general, see Proposition~\ref{prop:id_spline}.
	
	The resulting images are shown in Figure~\ref{fig:mnist}, where we see the MMDs for the SND with small $\varepsilon$ and the Riesz kernel work comparably well and converge to the target measure $\nu$.
	For larger smoothing parameter $\varepsilon$, the flow needs more iterations to converge.
	
	The distance to the target measure in the Wasserstein and MMD metrics is shown in Figure~\ref{fig:mnist-dist}. Here we use the sliced approximation of the MMD. Note that the MMD, which is the objective we minimize, still depends on the kernel $K$.
	We see a similar behavior as for the previous low-dimensional examples with the error plateauing at some level, which becomes better for smaller $\varepsilon$ even slightly beating the ND kernel.
	
	\begin{figure}[!ht]
		\begin{subfigure}{.31\textwidth}
			\centering
			\includegraphics[width=\linewidth]{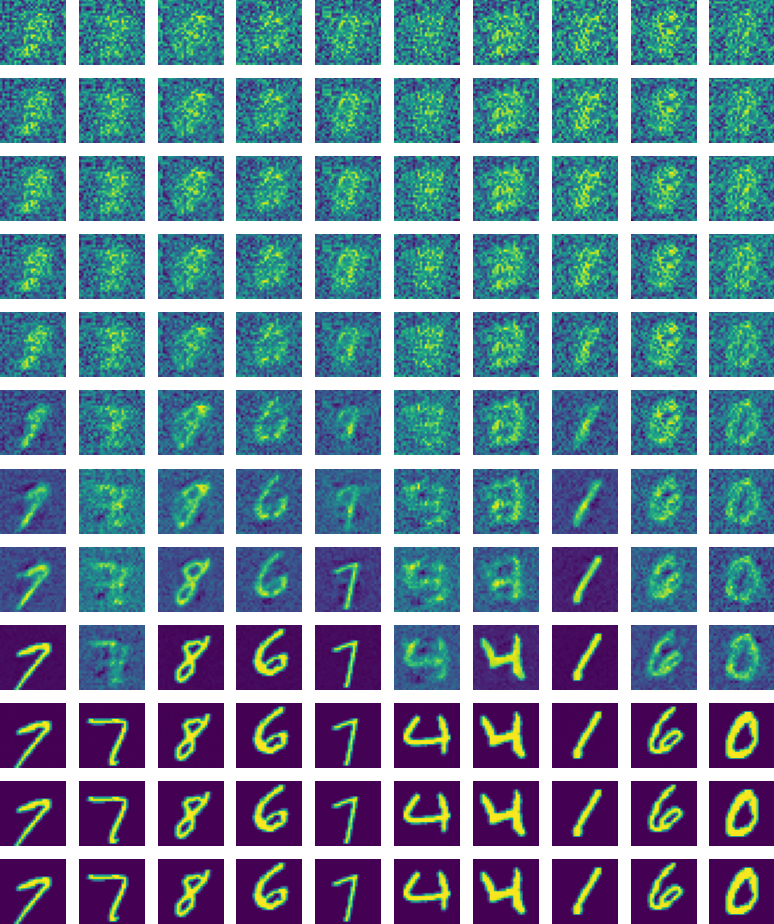}
			\caption{ND}
		\end{subfigure}\hfill
		\begin{subfigure}{.31\textwidth}
			\centering
			\includegraphics[width=\linewidth]{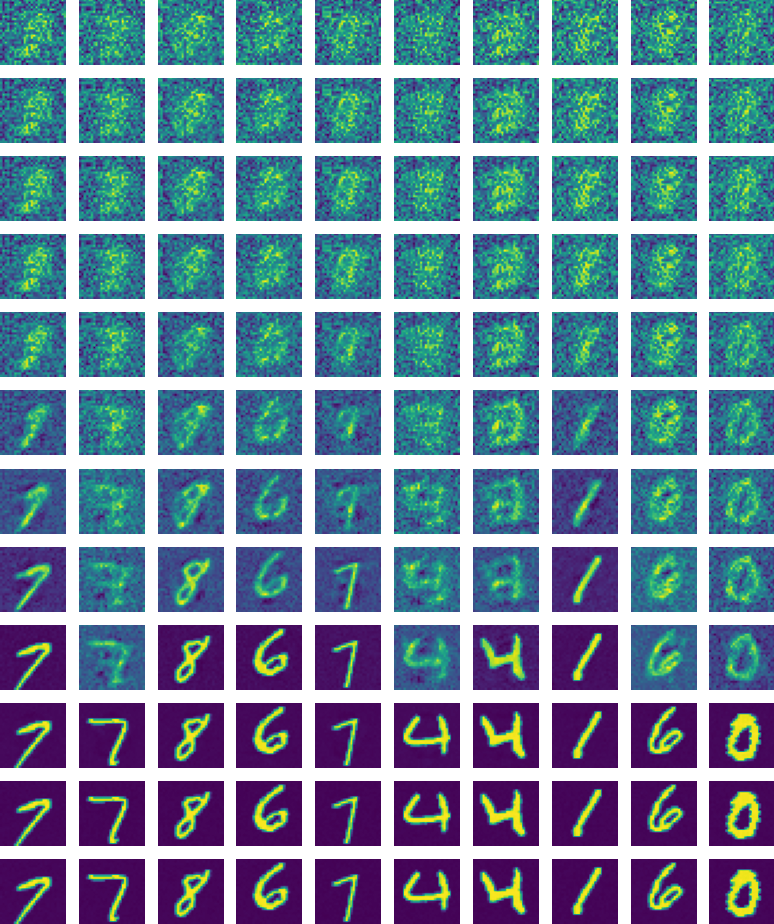}
			\caption{SND  ($\varepsilon=0.01$).}
		\end{subfigure}\hfill
		\begin{subfigure}{.31\textwidth}
			\centering
			\includegraphics[width=\linewidth]{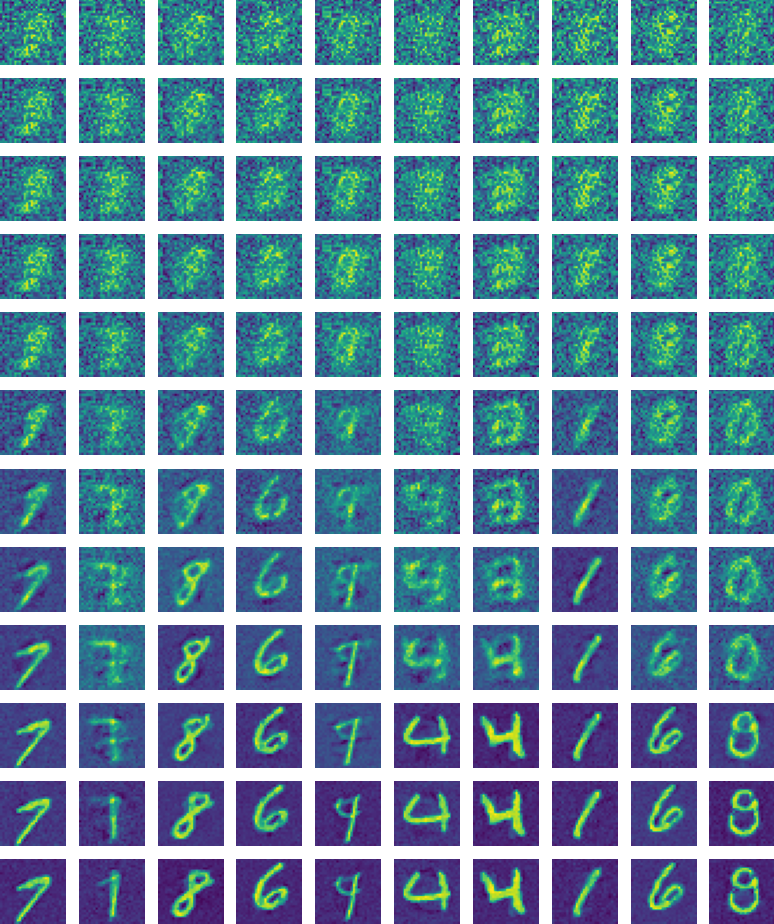}
			\caption{SND ($\varepsilon=0.1$).}
		\end{subfigure}
		\caption{MMD flow for MNIST target with different kernels. Each row shows the first 10 images $x_n\in\R^{28\times28}$, the $\ell$-th row corresponds to the iteration $k=2^{3+\ell}$, $\ell=1,\dots,12$.} 
		\label{fig:mnist}
	\end{figure}
	
	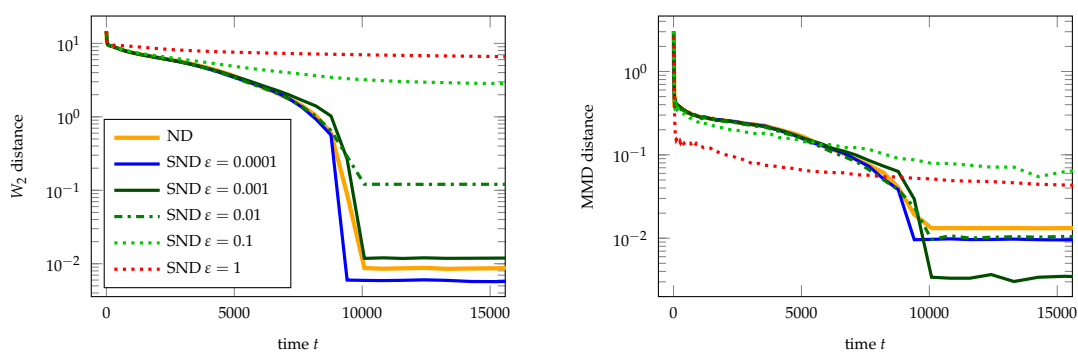
\begin{figure}[!ht]
		\centering
		\begin{tikzpicture}
			\begin{semilogyaxis}[xlabel=time $t$, ylabel=$W_2$ distance, width=.48\textwidth, height=0.36\textwidth, enlargelimits={abs=2mm}, xmin=0,xmax=15000, scaled ticks=false, /pgf/number format/.cd, 1000 sep={},
				legend pos=south west, legend style={cells={anchor=west}},
				]
				\addplot[ultra thick, color=orangenormal, mark=none 
				] table[x index=0,y index=2] {mnist_N100_it32769_m0_P785_loss.txt};
				\addplot[very thick, color=blue, mark=none 
				] table[x index=0,y index=2] {mnist_N100_smootheps0.0001_it32769_m0_P785_loss.txt};
				\addplot[very thick, color=greendark, mark=none 
				] table[x index=0,y index=2] {mnist_N100_smootheps0.001_it32769_m0_P785_loss.txt};
				\addplot[very thick, color=greennormal, dashdotted, mark=none 
				] table[x index=0,y index=2] {mnist_N100_smootheps0.01_it32769_m0_P785_loss.txt};
				\addplot[very thick, color=greenlight, dotted, mark=none 
				] table[x index=0,y index=2] {mnist_N100_smootheps0.1_it32769_m0_P785_loss.txt};
				\addplot[very thick, color=red, dotted, mark=none 
				] table[x index=0,y index=2] {mnist_N100_smootheps1_it32769_m0_P785_loss.txt};
				\legend{ND, SND $\varepsilon=0.0001$, SND $\varepsilon=0.001$, SND $\varepsilon=0.01$, SND $\varepsilon=0.1$, SND $\varepsilon=1$};
			\end{semilogyaxis}
			\begin{scope}[xshift=7.5cm]
				\begin{semilogyaxis}[xlabel=time $t$, ylabel=MMD distance, width=.48\textwidth, height=0.36\textwidth, xmin=10, xmax=68000, enlargelimits={abs=2mm}, xmin=0,xmax=15000, scaled ticks=false, /pgf/number format/.cd, 1000 sep={},
					legend pos=south west, legend style={cells={anchor=west}}, 
					]
					\addplot[ultra thick, color=orangenormal, mark=none 
					] table[x index=0,y index=1] {mnist_N100_it32769_m0_P785_loss.txt};
					\addplot[very thick, color=blue, mark=none 
					] table[x index=0,y index=1] {mnist_N100_smootheps0.0001_it32769_m0_P785_loss.txt};
					\addplot[very thick, color=greendark, mark=none 
					] table[x index=0,y index=1] {mnist_N100_smootheps0.001_it32769_m0_P785_loss.txt};
					\addplot[very thick, color=greennormal, dashdotted, mark=none  
					] table[x index=0,y index=1] {mnist_N100_smootheps0.01_it32769_m0_P785_loss.txt};
					\addplot[very thick, color=greenlight, dotted, mark=none 
					] table[x index=0,y index=1] {mnist_N100_smootheps0.1_it32769_m0_P785_loss.txt};
					\addplot[very thick, color=red, dotted, mark=none 
					] table[x index=0,y index=1] {mnist_N100_smootheps1_it32769_m0_P785_loss.txt};
				\end{semilogyaxis}
			\end{scope}
		\end{tikzpicture}
		\caption{MMD flow for MNIST target.
			Left: Wasserstein distance $W_2(\gamma^{(k)},\nu)$.
			Right:  MMD distance $\frac12 \mathcal D_K^2(\gamma^{(k)},\nu)$ for the respective kernels $K$.
		}
		\label{fig:mnist-dist}
	\end{figure}

	\section{Conclusions}\label{sec:conclusions}
	We introduced a smoothed negative distance kernel as an alternative to the negative distance kernel in MMDs.
	The novel kernel retains desired numerical properties of the negative distance kernel, but comes with well-defined
	gradient expressions and theoretical convergence guarantees 
	of the corresponding gradient flow schemes. Therefore our novel kernel appears to be well suited for various applications.
	
	Concerning our future work, it may be interesting to examine if our kernel can be also used in Stein variational gradient descent \cite{kama2021,NR2021},
	where negative distance kernels do neither theoretically nor
	practically work.

	\bibliographystyle{abbrv}
	\bibliography{ref}
	
	\appendix
	\section{Fourier Transform of Tempered Distributions}\label{app:fourier}
	Let  $\mathcal S'(\R^d)$ denote the space of tempered distributions, i.e., of linear
	functionals $T$ on $\mathcal S(\R^d)$ fulfilling
	$$\varphi_k \xrightarrow[{\mathcal S}]{~} \varphi 
	\quad
	\Longrightarrow
	\quad
	\lim_{k\to \infty}\langle T, \, \varphi_k \rangle = \langle T, \, \varphi \rangle,
	$$
	where $\xrightarrow[{\mathcal S}]{~}$ denotes the convergence with respect to  \begin{equation} \label{eq:seminormsS}
		\| \varphi \|_m :=
		\max\limits_{|{ \beta}| \le m}
		\|(1 + \|  x \|_2)^m \,D^{ \beta} \varphi ( x)\|_{\mathcal C_0(\mathbb R^d)} \quad \text{for all} \quad 
		m\in \N.
	\end{equation}
	In particular, $\mathcal S'(\R^d)$ contains all slowly increasing functions $f$, i.e. the functions fulfilling $|f(x)| \le C(1 + \|x\|^N)$ for some $N \in \N$
	and all functions in $L^p(\R^ d)$, $p \in [1,\infty)$.
	As usual, for distributions of function type, the distribution $T_f$ is identified with the function itself and the dual pairing becomes 
	$$\langle T_f,\varphi\rangle = \int_{\R^d} f \, \varphi \d x
	\quad \text{for all} \quad 
	\varphi \in \mathcal S(\R^d).
	$$
	The Fourier transform  
	$\mathcal F:  \mathcal S'(\R^d) \to  \mathcal S'(\R^d)$,
	$T \mapsto \hat T$
	is defined by
	\begin{equation} \label{eq:FS'}
		\langle T , \hat \varphi \rangle = \langle \hat T ,  \varphi \rangle 
		\quad \text{for all} \quad 
		\varphi \in \mathcal S(\R^d).
	\end{equation}
	In particular, we have for $f \in L^1(\R^d)$ in the above sense that $\hat T_f = T_{\hat f}$ with $\hat f$ given by \eqref{eq:fourier_trafo}.
	
	\begin{example}[Distributional versus Generalized Fourier transform of polynomials]
		Let $p\colon\R\to\R$, $p(x) \coloneqq \sum_{k=0}^{r-1} p_k x^k$. 
		Then
		$$
		\int_{\R^d} p(x) \hat \varphi(x) \d x = 0 
		\quad \text{for all} \quad
		\varphi \in \mathcal S_r,
		$$
		so the Generalized Fourier transform of $p$ 
		of order $r$ is the zero function,
		see \cite[Prop. 8.10]{Wendland2004}.
		In contrast, the distributional Fourier transform of $p$ is given by
		$$
		\hat p =  \sum_{k=0}^{r-1} \left(\frac{i}{2\pi}\right)^k p_k
		\delta^{(k)} .
		$$
		If we test only against functions in $\mathcal S_r$ both approaches coincide.
	\end{example}

	\begin{example}[Distributional Fourier transform of $\abs$]
		Since $\abs$ is slowly increasing, it is a tempered distribution. Its distributional Fourier transform can be written as the distributional derivative of the Cauchy principal value, 
		$$\widehat\abs = \frac{1}{2\pi^2} \left(\operatorname{pv}\left(\frac{1}{\cdot} \right)\right)',
		$$
		where
		$$
		\left\langle\operatorname{pv}\left(\frac1\cdot\right),\varphi\right\rangle 
		\coloneqq \lim_{\varepsilon\searrow0}\int_{|x|>\varepsilon} \frac{\varphi(x)}{x}\d x
		=
		\int_{\R} \frac{\varphi(x)-\varphi(0)}{x} \d x
		,\qquad \varphi\in\mathcal S(\R),
		$$
		see \cite[Sect 4.3]{plonka2018numerical}, \cite{gelfand1964generalized}.
		This can also be represented as the so-called Hadamard finite part 
		$\operatorname{H}\left(\frac{-1}{2\pi^2(\cdot)^2}\right)$, see \cite{criscuolo1997}, given by
		$$
		\left\langle{\widehat\abs,\varphi}\right\rangle
		=
		\left\langle\operatorname{H}\left(\frac{-1}{2\pi^2(\cdot)^2}\right), \varphi\right\rangle
		\coloneqq
		-\int_{\R} \frac{\varphi(\omega)-\varphi(0)-\varphi'(0)\omega}{2\pi^2\omega^2} \d\omega.
		$$
		If we test only against functions from $\varphi \in \mathcal S_2(\R)$,
		we have $\varphi(0) = \varphi'(0) = 0$, so that
		this coincides 
		with the generalized Fourier transform \eqref{fabs}. 
		\hfill $\Box$
	\end{example}
	
	Another special case of tempered distributions are finite Borel 
	measures $\mathcal M(\R^d)$, see \cite[Sect. 4.4]{plonka2018numerical}.
	More precisely, since $\mathcal S(\mathbb R^d)$ is a dense subspace of $\left(\mathcal C_0(\mathbb R^d),\|\cdot\|_\infty\right)$, we know by the Riesz representation theorem
	that $\mu \in \mathcal M(\mathbb R^d)$ can be identified  with 
	a tempered distribution $T_\mu \colon \mathcal S(\mathbb R^d) \to \mathbb C$ which acts on any $\varphi \in \mathcal S(\mathbb R^d)$ by 
	\begin{equation*}
		\langle T_\mu, \varphi \rangle
		\coloneqq
		\int_{\mathbb R^d} \varphi \d \mu. 
	\end{equation*}
	The \emph{Fourier transform}
	on $\mathcal M(\mathbb R^d)$ is defined by
	$\mathcal F \colon \mathcal M(\mathbb R^d) \to \mathcal C_b(\mathbb R^d)$ with
	\begin{equation}\label{f_measure}
		\mathcal{F} \mu( \omega)
		=
		\hat{\mu} ( \omega)
		:=
		\int_{\mathbb R^d} \e^{- 2 \pi i \omega \cdot} \d \mu
		,\qquad
		\omega \in \mathbb R^d,
	\end{equation}
	and we have $\hat T_{\mu} = T_{\hat \mu}$.
	For positive measures $\mu \in \mathcal M(\mathbb R^d)$, i.e.\
	$\mu(B) \ge 0$ for all Borel sets $B \subseteq \mathbb R^d$, we obtain a
	one-to-one mapping to positive definite functions by Bochner's theorem.
	
	\begin{theorem}[Bochner]  \label{thm:bochner}
		Any positive definite function $f \colon \R^d \to \C$ is the Fourier
		transform of a positive measure and conversely. If in addition $f(0) = 1$, then it is the Fourier transform of a probability measure.
	\end{theorem}
	Note that, by our definition, positive definite functions automatically are continuous.

	\section{Relation with Moreau Envelopes}\label{app:huber}
	For a proper, convex, lower semi-continuous function $g\colon \R^d \to \R$
	and $\lambda >0$, the \emph{proximal function} $\text{prox}\colon\R^d \to \R^d$ is defined by 
	$$
	\prox_{\lambda g} (x) = \argmin_{y \in \R^d} \left\{\frac12 \|x-y\|^2+ \lambda g(y) \right\}
	$$
	and its \emph{Moreau envelope} by 
	$$
	H_{\lambda g} (x) = \min_{y \in \R^d} \left\{\frac12 \|x-y\|^2 + \lambda g(y)\right\}.
	$$
	The Moreau envelope is differentiable and
	\begin{align} \label{moreau_diff}
		\nabla H_{\lambda g}(x) &= x - \prox_{\lambda g} (x),
	\end{align}
	so that
	\begin{align} \label{needed}
		\prox_{\lambda g} (x) &= x - \nabla H_{\lambda g}(x) = \nabla 
		\underbrace{\left( \frac12 \|x\|^2 - H_{\lambda g}(x)\right)}_{\psi(x)}
		.
	\end{align}
	Conversely, we have the following  result of Moreau \cite[Cor~10c]{moraeu1965}.
	
	\begin{proposition}\label{prop:moreau}
		A function $G\colon \R^d \to \R^d$ is the proximal function of a proper, convex, lower semi-continuous function if and only if i) there exists a convex differentiable function~$\psi$ such that $G = \nabla \psi$, and ii) $G$ is nonexpansive, i.e., $\|G(x)-G(y)\|\le\|x-y\|$ for all $x,y\in\R^d$.
	\end{proposition}
	
	In particular, we obtain for $g=\abs$ that
	$$
	\prox_{\lambda \abs} (x) = \left\{
	\begin{array}{ll}
		x-\lambda& x > \lambda,\\
		0 & x \in [-\lambda,\lambda],\\
		x+\lambda&x <-\lambda,
	\end{array}
	\right.
	$$
	and the Moreau envelope 
	$$
	H_{\lambda \abs} (x) = 
	\left\{
	\begin{array}{ll}
		\frac12 x^2 & |x| \le \lambda\\
		\lambda(|x| - \frac{\lambda}{2})& \text{otherwise}
	\end{array}
	\right.
	$$
	is known as the \emph{Huber function}.
	
	On the other hand, we have 
	$$
	(\abs*M_1)(x) = 
	\left\{
	\begin{array}{ll}
		x^2 + \frac14 & |x| \le \frac12,\\[0.5ex]
		|x|& |x|>\frac12,
	\end{array}
	\right.
	$$
	so that by Proposition \ref{prop:1}, for 
	$\varepsilon = 2 \lambda$,  
	$$
	(\abs*M_{1,2\lambda})(x)
	=
	\frac{1}{\lambda} H_{\lambda \abs} (x) + \frac{\lambda}{2} 
	= 
	\left\{
	\begin{array}{ll}
		\frac{x^2}{2\lambda} + \frac{\lambda}{2}  
		& |x| \le  \lambda,\\[0.5ex]
		|x| & |x| > \lambda.
	\end{array}
	\right.
	$$
	Thus, by \eqref{fabs} and
	Propositions \ref{lem:four_conv} and \ref{prop:splines}, the
	Generalized Fourier transform of the Huber function
	is given by
	\begin{equation}\label{fhuber}
		\hat H_{\lambda \abs} (\omega)= - \frac{\lambda}{2\pi^2 \omega^2} \, \sinc(2\lambda \omega).
	\end{equation}
	Since this function has positive and negative values, we conclude by Proposition \ref{thm:cond_bochner} that the negative Huber function is not conditionally positive definite of any order.
	
	Let us see if $\abs*M_{1,\varepsilon}$ is the Moreau envelope of some 
	function.
	Regarding \eqref{needed}, we consider
	$$
	\psi(x) \coloneqq \frac12 x^2 - (\abs*M_{1,2\lambda})(x)
	= \left\{
	\begin{array}{ll}
		\frac12 \left(1- \frac{1}{\lambda} \right)x^2 - \frac{\lambda}{2}  
		&   |x| \le  \lambda,\\[0.5ex]
		\frac{x^2}{2} - |x| & |x| > \lambda,
	\end{array}
	\right.
	$$
	which is convex for $\lambda \ge 1$ and 
	has a nonexpansive derivative 
	$$
	\psi'(x) = x- (\abs*M_{1,\varepsilon})'(x)
	=
	\left\{
	\begin{array}{ll}
		\left(1- \frac{1}{\lambda} \right) x
		& |x| \le  \lambda,
		\\[0.5ex]
		x-1 & x > \lambda,\\
		x+1 & x < -\lambda.
	\end{array}
	\right.
	$$
	Thus, by Moreau's Proposition \ref{prop:moreau}, we see that
	$\abs*M_{1,2\lambda}$ is a Moreau envelope if and only if $\varepsilon = 2 \lambda \ge 2$.
	More general, we have the following proposition
	
	\begin{proposition}
		For $m \in \N$, $m\ge 1$, the function
		$\abs*M_{m,\varepsilon}$
		is the Moreau envelope of a proper, convex, lower semi-continuous function if and only if
		$\varepsilon \ge 2 M_m(0)$.
	\end{proposition}
	
	\begin{proof}
		By the above considerations, the assertion is true for $m=1$.
		By Proposition \ref{prop:1}, the function
		$$
		\psi(x) \coloneqq \frac12 x^2 - (\abs*M_{m,\varepsilon})(x)
		$$
		fulfills
		\begin{align}
			\psi''(x) &= 1 - \frac{2}{\varepsilon} M_m\left(\frac{x}{\varepsilon} \right) 
			\ge 1 - \frac{2}{\varepsilon} M_m(0)
			\ge 0
		\end{align}
		if and only if $\varepsilon \ge 2 M_m(0)$,
		and exactly in this case $\psi$ is convex.
		Further, because $\psi''\le1$, we see that $\psi'$
		is nonexpansive and by Moreau's Proposition \ref{prop:moreau}, the function $\abs*M_{m,\varepsilon}$ is a Moreau envelope.
	\end{proof}
	\section{Proofs}\label{app:a}
	\subsection*{Proofs from Section 2}
	\textbf{Proof of Proposition \ref{lem:four_conv}.}
	Since $f\in \mathcal{C}(\Rd)$ is slowly increasing
	and $u\in \mathcal{C}_c(\Rd)$, we conclude
	by straightforward computations that $f*u$ is continuous and slowly increasing, too. Therefore, $\langle f*u,\hat \varphi\rangle$ exists for all $\varphi \in \mathcal S(\R^d)$. 
	Using Fubini's theorem, we obtain 
	\begin{align*}
		\langle f*u,\hat \varphi\rangle &
		=\int_\Rd  (f*u)(x)\hat \varphi(x)\d x
		=\int_\Rd \int_\Rd u(y)f(x-y)\d y \hat \varphi(x)\d x\\&
		=\int_\Rd u(y)\int_\Rd   f(x-y)\hat \varphi(x)\d x\d y.
	\end{align*}
	By the translation-modulation theorem, we know that
	\begin{align*}
		\hat \varphi(x)=\mathcal F[\e^{-2\pi \i\langle \cdot ,y\rangle}\varphi](x-y),
	\end{align*}
	so that
	\begin{align*}
		\langle f*u,\hat \varphi\rangle 
		&=
		\int_\Rd u(y) \int_\Rd f(x-y) \mathcal F[\e^{-2\pi \i\langle \cdot,y\rangle}\varphi](x-y)\d x\d y\\
		&=\int_\Rd u(y) \int_\Rd f(x) \mathcal F[\e^{-2\pi \i\langle \cdot,y\rangle}\varphi](x)\d x\d y.
	\end{align*}        
	Since $f$ has a generalized Fourier transform $\hat f$ of order $r$, this implies for $\varphi \in \mathcal{S}_{2r}(\R^d)$   that
	\begin{align*}
		\langle f*u,\hat \varphi\rangle 
		&=\int_\Rd \hat f(x) \varphi(x)\int _\Rd u(y)\e^{-2\pi \i\langle x,y\rangle} \d y\d x\\
		&= \int_\Rd \hat f(x) \varphi(x) \hat u(x)\d x.
	\end{align*}
	Hence, $f*u$ has a generalized Fourier transform of order $r$, namely $\hat f \, \hat u$.
	\hfill\qedsymbol{\parfillskip0pt\par}
	
	\subsection*{Proofs from Section 3}
	\textbf{Proof of Proposition \ref{prop:1}.}
	i) follows directly by definition of $f$ and since $u$ is even.
	
	To show ii), let $x>R \coloneqq \textup{diam} (\supp u)/2$. The case $x<-R$ follows similarly. 
	Then we obtain
	\begin{align*}
		(\abs* u)(x)&
		=\int\limits_{-R}^R  u(y)(x-y)  \d y
		=x\int\limits_{-R}^R  u(y)  \d y -\int\limits_{-R}^R y \,  u(y)  \d y
		= x\cdot 1-0 =x.
	\end{align*}
	
	In iii), we only have to show that $f''= 2u$. Then the smoothness of $f$ follows by $u \in \mathcal C^n(\R)$.
	Using Lebesgue's dominated convergence theorem, we conclude
	\begin{align*}
		\frac{\d}{\d x}(\abs*u)(x)
		&= \lim_{h \to 0} \int\limits_{-R}^R \underbrace{\frac{|x+h-y| - |x-y|}{h} u(y) }_{|g_{x,h}| \le u } \d y
		=\int\limits_{-R}^R \sgn(x-y)u(y)\d y \\
		&=(\sgn *u)(x),
	\end{align*}
	where 
	$$\sgn(x) \coloneqq \left\{
	\begin{array}{rl}
		1& x \ge 0,\\
		-1& x < 0.
	\end{array}
	\right.$$
	The right derivative of $\sgn$ is given by
	\begin{equation*}
		\lim_{h\searrow 0} \frac{\sgn(x+h)-\sgn(x)}{h}=\lim_{h\searrow 0}\frac{2}{h}\1_{[-h,0)}(x).
	\end{equation*}
	Therefore, we have by continuity of $u$ that
	\begin{align*}
		\lim_{h\searrow 0} \frac{(\sgn*u)(x+h)-(\sgn*u)(x)}{h}&
		=\lim_{h\searrow 0} \int_\R \frac{2}{h}\1_{[-h,0)}(y)u(x-y)\d y\\&
		=2\lim_{h\searrow 0} \frac{1}{h}\int_{-h}^0 u(x-y)\d y
		=2u(x).
	\end{align*}
	We obtain the same result for the left  derivative. Since $f''=2u\ge0$, the function $f$ is convex.
	
	For iv), we have by Lemma \ref{lem:four_conv}
	and \eqref{fabs} that 
	\begin{equation}
		-\hat f= \mathcal F[-\abs * u] = (-\widehat{ \abs}) \, \hat u \ge 0.
	\end{equation}
	Therefore $-f$ is conditionally positive definite of order $r=1$ by Theorem~\ref{thm:cond_bochner}.
	
	Assertion v) follows by straightforward computation.
	
	Finally, we show vi). Note that  we cannot apply the usual convergence theorems for approximate identities in vi), because $\abs\notin  \mathcal C_0(\R)$.
	
	Let $\supp u\subseteq [-R,R]$ for some $R>0$. Then, we have by v) and ii) that  
	\begin{equation} \label{eq:f_zeta=abs}
		(\abs * u_\varepsilon)(x) = |x| \quad \text{for} \quad
		|x| \ge \varepsilon R.
	\end{equation}
	For $|x|\le \varepsilon R$, we conclude using  $\int_\R u_\varepsilon (y) \d y=1$ that
	\begin{align*}
		\big| |x|- (\abs*u_\varepsilon)(x) \big|
		&=
		\left|\int_\R |x| u_\varepsilon (x-y)\d y-\int_\R |y|u_\varepsilon (x-y)\d y\right|\\
		&\le \int_\R |x-y| u_\varepsilon(x-y)\d y
		=
		\int_{-\varepsilon R}^{\varepsilon R} |y| \, u_\varepsilon(y) \d y\\
		&\le \varepsilon R \int_{-\varepsilon R}^{\varepsilon R} u_\varepsilon(y)\d y
		= \varepsilon R,
	\end{align*}
	which shows the uniform convergence.
	\hfill\qedsymbol
	\\[2ex]
	\textbf{Proof of Corollary \ref{cor:1}.}
	By \eqref{bspline} and since $f'' = 2 M_m$, we obtain 
	$$
	f(x)= \frac{2}{(m+1)!}
	\sum_{k=0}^{m} (-1)^k \binom{m}{k} \Big(x-k + \frac{m}{2} \Big)_+^{m+1} + ax + b
	$$
	with some $a,b\in\R$.
	Further, for $x \le - \frac{m}{2}$, we conclude by $f(x) = -x$ that
	$$
	f(x) =  ax + b = - x,
	$$
	so that $a=-1$ and $b=0$.
	\hfill\qedsymbol{\parfillskip0pt\par}
	\subsection*{Proofs from Section 4}
	\textbf{Proof of Lemma \ref{ex:not_cond_pos}.}
	The function $f \coloneqq \abs*M_2$ has the generalized Fourier transform of order $1$ given by 
	$\hat f(r)=- \frac{\sinc^2(r)}{2\pi^2 r^2}$.
	We define
	\begin{equation}\label{eq:FT_of_F3}
		g(r)\coloneqq -\frac{1}{2\pi}\frac{1}{r}\frac{\d}{\d r}\hat f(r) .
	\end{equation}
	For integrable functions it was proven in \cite[Thm.~1.1]{Grafakos_2012} that $g(\|\cdot\|)$ is the $3$-dimensional Fourier transform of $f(\|\cdot\|)$.
	However, since $f$ is not integrable, we  use the generalized Fourier transform to argue that $\mathcal F_3[f(\|\cdot\|)]=g(\|\cdot\|)$:
	for an even test function $\psi\in \mathcal S_{2}(\R)$, we apply \cite[Thm.~1.1]{Grafakos_2012} to obtain
	\begin{equation}
		\mathcal F_3[\psi(\|\cdot\|)](\|x\|)=-\frac{1}{2\pi}\frac{1}{\|x\|}\hat \psi'(\|x\|),
	\end{equation}
	and with the surface area $\omega_{2}=4\pi$, integration by parts gives
	\begin{align*}
		\int_{\R^3} g(\|x\|)\psi(\|x\|)\d x&
		=\omega_2 \int_0^\infty g(r)\psi(r)r^2\d r
		=-\int_0^\infty 2 \hat f'(r)\psi(r)r\d r\\&
		=-\left[2\hat f(r)\psi(r)r\right]_0^\infty+\int_0^\infty 2\hat f(r)\frac{\d}{\d r}(\psi(r)r)\d r.
	\end{align*}
	Since $\psi\in \mathcal S_2(\R)$, the first summand $[2\hat f(r)\psi(r)r]_0^\infty$ vanishes. 
	The derivative of the odd function $\psi(r)r$ is even and still in $\mathcal S_2(\R)$. Thus, we get  by \eqref{g_fourier} that
	\begin{align*}
		\int_{\R^3} g(\|x\|)\psi(\|x\|)\d x&
		= \int_\R \hat f(r) \frac{\d}{\d r}(\psi(r)r)\d r
		=-\int_\R f(r) r\frac{\d}{\d r}\hat \psi(r)\d r\\&
		=4\pi \int_0^\infty f(r) \frac{-1}{2\pi r}\frac{\d}{\d r}\hat \psi(r)r^2\d r
		= \int_{\R^3} f(\|x\|) \mathcal F_3[\psi(\|\cdot\|)](\|x\|)\d x,
	\end{align*}
	i.e. $\mathcal F_3[f(\|\cdot\|)]=g(\|\cdot\|)$.
	Next we show that for all test functions $\varphi\in \mathcal S_2(\R^3)$ it holds
	\begin{equation*}
		\int_{\R^3} g(\|x\|)\varphi(x)\d x = \int_{\R^3} f(\|x\|)\hat \varphi(x)\d x.
	\end{equation*}
	For an arbitrary  $\varphi \in \mathcal S(\R^d)$, define the radial test function 
	\begin{align*}
		\operatorname{Rad}\varphi(x)\coloneqq \frac{1}{\omega_{d-1}}\int_{\sd} \varphi(\|x\|\xi)\d\xi.
	\end{align*}
	In \cite[Thm.~4.2 i)]{rux2024slicing} it was shown, that $\textup{Rad}$ is a continuous projection of $\mathcal S(\R^d)$ to the space of radial Schwartz functions $\mathcal S_\textup{rad}(\R^d)$.
	By the uniqueness of rotational
	invariant measures on the sphere, see \cite[(2.3)]{Rag74}, we have
	$$
	\operatorname{Rad}\varphi(x)
	=
	\int_{SO(d)} f(Rx)\d\mathcal U_{SO(d)} (R),
	$$
	where $\mathcal U_{SO(d)}$ is the uniform measure on the set $SO(d)$ of $d\times d$ rotation matrices.
	Since the Fourier transform commutes with rotations, we have
	$
	\operatorname{Rad}[\mathcal F\varphi]
	=
	\mathcal F[\operatorname{Rad}\varphi]
	$ for all $\varphi\in\mathcal S(\R^d)$,
	so the operators $\mathcal F$ and $\textup{Rad}$ commute.
	It is easy to see that for $\varphi\in \mathcal S_2(\R^3)$, we also have $\operatorname{Rad}\varphi\in \mathcal S_2(\R^3)$.
	The action of the test functions $\varphi$ and $\operatorname{Rad}\varphi$ on $g(\|\cdot\|)$ is the same, because $g(\|\cdot\|)$ is radial.
	Hence we have for all $\varphi\in \mathcal S_2(\R^3)$ that
	\begin{align*}
		\int_{\R^3} g(\|x\|) \varphi(x)\d x&
		=\langle g(\|\cdot\|), \varphi\rangle 
		=\langle g(\|\cdot\|), \operatorname{Rad}\varphi\rangle 
		=\langle f(\|\cdot\|), \mathcal F[\operatorname{Rad}\varphi]\rangle\\&
		=\langle f(\|\cdot\|), \operatorname{Rad}\hat\varphi]\rangle
		=\langle f(\|\cdot\|), \hat \varphi \rangle
		=\int_{\R^3} f(\|x\|)\hat \varphi(x)\d x.
	\end{align*}
	Consequently, $f(\|\cdot\|)$ has the generalized Fourier transform $g(\|\cdot \|)$ of order $1$.
	Theorem~\ref{thm:cond_bochner} shows that $-f(\|\cdot\|)\notin \mathrm{CP}_1(\R^3)$, because $g$ changes its sign. Since $g(\|\cdot\|)$ is the generalized Fourier transform of $f(\|\cdot\|)$ for all $r\ge 1$, we see that $-f(\|\cdot\|)\notin \mathrm{CP}_r(\R^3)$ for all $r\ge 1$. Since $-f(\|\cdot\|)\notin \mathrm{CP}_1(\R^3)$,
	we have by \cite[Prop.~8.2]{Wendland2004} that $-f(\|\cdot\|)\notin \mathrm{CP}_r(\R^d)$ for all $r\ge 0$ and all $d\ge 3$.
	\hfill $\Box$
	\\[2ex]
	\textbf{Proof of Proposition \ref{prop:cpd1}.}
	Assume that $f \in \text{CP}_r(\R)$, then $f(x)\in \mathcal O(|x|^{2r})$, by  \cite[Cor~2.3]{madych1990multi}.
	The function $F$ is well-defined, because $f$ is continuous and is slowly increasing.
	Let $x_1,\ldots, x_N\in \R^d$ and $a_1, \ldots,a_n\in \R$ such that 
	\begin{equation} \label{ass}
		\sum_{j=1}^N a_j P(x_j) = 0 
	\end{equation}
	for all polynomials $P$ on $\R^d$ of degree  $< r$.
	In particular, any polynomial $p(t)=\sum_{k=0}^{r-1} c_k t^k$ on $\R$
	determines for an arbitrary fixed $\xi\in \sd$, a polynomial on $\R^d$ of degree $< r$ by
	\begin{equation*}
		P_\xi(x)\coloneqq p(\langle \xi,x\rangle)= \sum_{k=0}^{r-1} c_k \langle \xi,x\rangle^k=\sum_{k=0}^{r-1} c_k \left( \sum_{l=1}^d \xi_l x_l\right)^k.
	\end{equation*}
	By \eqref{ass}, we have 
	$$ \sum_{j=1}^N a_j P_\xi (x_j) 
	= 
	\sum_{j=1}^N a_j p(\langle \xi, x_j\rangle )
	= 0.$$
	
	Since $f$ is conditionally positive definite of order $r$, we know that
	\begin{align*}
		0 \le \sum_{j,k=1}^N a_j a_k 
		f(|\langle \xi,x_j\rangle -\langle \xi,x_k\rangle|),
	\end{align*}
	so that by Theorem \ref{thm:slicing_b} also
	\begin{align*}
		0 \le \frac{1}{w_{d-1}} \int_\sd \sum_{j,k=1}^N a_j a_k f(|\langle \xi,x_j-x_k\rangle |) \d\xi  \d\xi = \sum_{j,k=1}^N a_j a_k F(\|x_j-x_k\|). 
	\end{align*}
	Hence $F(\|\cdot\|)$ is conditionally positive definite of order $r$. 
	\hfill\qedsymbol
	\\[2ex]
	\textbf{Proof of Proposition \ref{prop:cpd2}.}
	In \cite[Eq.~(6) \& (7)]{rux2024slicing}, two operators were introduced:
	the {rotation operator} $\mathcal R_d$ acts on a function $F\colon [0,\infty)\to \R$ as $\mathcal R_d F(x)\coloneqq F(\|x\|)$, and the {spherical averaging operator} $\mathcal A_d$ assigns to a function $\Phi\colon \R^d\to \R$ integrable on the spheres $t \, \sd$ for all $t >0$ the function 
	\begin{equation*}
		\mathcal A_d\Phi(t)\coloneqq \frac{1}{\omega_{d-1}} \int_\sd \Phi(t\xi)\d\xi.
	\end{equation*}
	For $d=1$, the spherical averaging operator reduces to $\mathcal A_1 \Phi(t)=\frac{1}{2}(\Phi(t)+\Phi(-t))$. 
	Moving to distributions, the operator $\mathcal R_d^\star$ acts on a tempered distribution $T$ as
	$$ 
	\langle \mathcal R_d^\star T,\psi\rangle = \langle T, (\mathcal R_d\circ \mathcal A_1)\psi\rangle \quad \text{ for all } \quad \psi \in \mathcal S(\R).
	$$
	Since $F(\|\cdot\|)$ is continuous and slowly increasing, it can be identified with a tempered distribution.
	Let $\mathcal F_d$ denote the Fourier transform of tempered distributions.
	Since $F$ is $\lfloor\frac{d}{2}\rfloor$-times continuously differentiable, we have by \cite[Cor.~4.9]{rux2024slicing}  that 
	$$ f\coloneqq (\mathcal F_1\circ \mathcal R_d^\star \circ \mathcal F_d^{-1})[\mathcal R_dF]
	$$ 
	is a distribution arising from a continuous, even function which
	satisfies $\mathcal I_d[f]=F$. 
	
	Let $\psi\in \mathcal S_{2r}(\R)$ be an even.
	Then $(\mathcal R_d\circ \mathcal A_1)\psi=\psi(\|\cdot\|)$ is a radial Schwartz function in $\mathcal S_{2r}(\R^d)$. 
	Since $\mathcal R_d F$ has a Generalized Fourier transform $\rho(\|\cdot\|)$ of order $r$, we obtain
	\begin{align*}
		\int_\R f(r)\hat \psi(r)\d r&
		= \langle f,\hat \psi\rangle 
		=\langle \hat f, \psi\rangle 
		=\langle (\mathcal R_d^\star \circ \mathcal F_d^{-1})[\mathcal R_dF],\psi\rangle 
		=\langle\mathcal R_dF,(\mathcal F_d^{-1}\circ \mathcal R_d\circ\mathcal A_1)\psi\rangle \\&
		= \int_{\R^d} F(\|x\|) \mathcal F_d[\psi(\|\cdot\|)](x)\d x
		= \int_{\R^d} \rho(\|x\|)\psi(\|x\|)\d x\\&
		= \frac{w_{d-1}}{2} \int_\R \rho(\omega)|\omega|^{d-1}\psi(\omega)\d \omega.
	\end{align*}
	In particular, $f$ has the generalized Fourier transform $\frac{w_{d-1}}{2}\rho(\omega)|\omega|^{d-1}\in  \mathcal C(\R\setminus\{0\})$ of order $r$, which is nonnegative, so that $f$ is conditionally positive definite of order $r$. 
	\hfill $\Box$
	\\[2ex]
	\textbf{Proof of Proposition \ref{prop:of_F}.}
	For $d\ge 2$, the term $(1-t^2)^{\frac{d-3}{2}}$, $t \in [0,1]$ is integrable. Since $f=\abs*u_\varepsilon\in \mathcal C^n(\R)\subseteq L^\infty_\textup{loc}(\R)$, $n \in \N$, the function $F=\mathcal I_d[f]$ is  well-defined. By Proposition \ref{prop:1}, we know that $f$ is nonnegative and even. Hence, also $F$ is nonnegative and even.
	By Leibniz's integral rule and since 
	$f \in \mathcal C^{n+1}(\R)$, we obtain for $k=1,\ldots, n+2$
	that
	\begin{align*}
		\frac{\d^k}{\d s^k} F(s) &
		= \frac{\d^k}{\d s^k}c_d\int_0^1 f(ts)(1-t^2)^\frac{d-3}{2}\d t
		=c_d\int_0^1 \frac{\d^k}{\d s^k}f(ts)(1-t^2)^\frac{d-3}{2}\d t\\&
		=c_d\int_0^1 t^kf^{(k)}(ts)(1-t^2)^\frac{d-3}{2}\d t,
	\end{align*}
	so that $F\in  \mathcal C^{n+2}(\R)$. 
	Since $f$ is at least twice differentiable and $f(t)=\abs(t)$ for $|t|$ large enough, it follows, that $\|f'\|_\infty<\infty$.
	For the first derivative of $F$ we get
	\begin{align*}
		|F'(s)| &
		=\left|c_d\int_0^1 tf'(ts)(1-t^2)^\frac{d-3}{2}\d t\right|
		\le c_d  \|(1-t^2)^{\frac{d-3}{2}}\|_{L^1(0,1)}\|f'\|_\infty<\infty.
	\end{align*}
	The convexity of $F$ follows directly from the convexity of $f$.
	
	Let $\supp(u)\subseteq [-R,R]$ for some $R>0$. Then we have 
	by Proposition \ref{prop:1} for $s\ge R$ that $f(s)=s$. 
	Further, by Lemma \ref{lem:eigenf}, it holds $\mathcal I_d[\abs]= C_d \abs$.
	Hence, we obtain for  $s >R$ that
	\begin{align*}
		|F(s)- C_d \, \abs(s)|&
		=\left|c_d\int_0^1 (f(st)-\abs(st))(1-t^2)^\frac{d-3}{2}\d t\right|
		\\
		&\le 
		c_d\int_0^\frac{R}{s} |f(st)-st|(1-t^2)^\frac{d-3}{2}\d t
		\\
		&= \frac{c_d}{s}\int_0^R |f(t)-t|\left(1-\frac{t^2}{s^2}\right)^\frac{d-3}{2}\d t  
		\in \mathcal O\left(\frac{1}{s} \right).
	\end{align*}
	In particular, it holds $h \coloneqq F-C_d\abs \in  \mathcal C_0(\R) \cap L^2(\R)$.
	
	Finally, we  obtain by Proposition \ref{prop:1} that
	\begin{align}\label{eq:F_zeta}
		F^\varepsilon(s) &= c_d \int_0^1 (\abs*u_\varepsilon)(t s)(1-t^2)^\frac{d-3}{2}\d t\\
		&=c_d \varepsilon \int_0^1 f\left(\frac{st}{\varepsilon} \right)(1-t^2)^\frac{d-3}{2} \d t
		=\varepsilon F\left(\frac{s}{\varepsilon}\right).
	\end{align}
	and further
	\begin{align*}
		\|F^\varepsilon - C_d \abs\|_{L^2(\R)}^2
		&\le 
		\int_\R |F^\varepsilon (s )- C_d \abs(s) |^2 \d s
		=
		\varepsilon^2 \int_\R 
		\Big | F\left(\frac{s}{\varepsilon}\right)
		-
		C_d \abs\left(\frac{s}{\varepsilon}\right) \Big|^2\d s
		\\
		&=
		\varepsilon^2 \int_\R |h\left(\frac{s}{\varepsilon}\right)|^2\d s
		=
		\varepsilon^3 
		\|h\|_{L_2(\R)}^2.
	\end{align*}
	This gives us the order of convergence in $L_2(\R)$.
	The pointwise convergence of $F^\varepsilon$ directly follows from \eqref{eq:RiemannLFRI}.
	\hfill $\Box$
	\\[2ex]
	\textbf{Proof of Proposition \ref{prop:id_spline}.}
	By Corollary~\ref{cor:1}, we have
	$$
	f(x) = \frac{2}{(m+1)!}
	\sum_{k=0}^{m} (-1)^k \binom{m}{k} \Big(x-k + \frac{m}{2} \Big)_+^{m+1} - x
	,\qquad x\in\R.
	$$
	Defining for $m\in\N$ and $a\in\R$ the function 
	\begin{equation} \label{eq:bma}
		b_{m,a}(x) = (x-a)_+^m,
	\end{equation}
	we have 
	$$
	f(x) = \frac{2}{(m+1)!} \sum_{k=0}^{m} (-1)^k \binom{m}{k} b_{m+1,k-\frac m2}(x) - x.
	$$
	If $a\le0$, we have $b_{m,a}(x)=(x-a)^m$ for $x\ge0$.
	Then
	\begin{align*}
		\mathcal I_d[b_{m,k}](s)
		&=
		c_d \int_{0}^{1} f_{m,a}(st) (1-t^2)^{\frac{d-3}{2}} \d t
		=
		c_d \int_{0}^{1} (st-a)^m (1-t^2)^{\frac{d-3}{2}} \d t
		\\&=
		c_d \sum_{n=0}^m \binom{m}{n} s^n(-a)^{m-n} \int_{0}^{1} t^n (1-t^2)^{\frac{d-3}{2}} \d t
		\\&=
		\frac{c_d}{2} \sum_{n=0}^m \binom{m}{n} s^n(-a)^{m-n} \int_{0}^{1} t^{\frac{n-1}{2}} (1-t)^{\frac{d-3}{2}} \d t.
		\intertext{Since the Beta function satisfies $B(a,b)=\int_0^1t^{a-1}(1-t)^{b-1}=\frac{\Gamma(a)\Gamma(b)}{\Gamma(a+b)}$, see~\cite{abramowitz1964handbook}, we obtain}
		\mathcal I_d[b_{m,k}](s)
		&=
		\frac{c_d}{2} \sum_{n=0}^m \binom{m}{n} s^n(-a)^{m-n} \frac{\Gamma (\frac{d-1}{2}) \Gamma (\frac{n+1}{2})}{\Gamma (\frac{d+n}{2})}.
	\end{align*}
	If $a>0$ and $s\le a$, we have $\mathcal I_d[f_{m,a}](s)=0$.
	Otherwise, i.e.\ for $0<a<s$, we have
	\begin{align*}
		\mathcal I_d[b_{m,a}](s)
		&=
		c_d \int_{a/s}^{1} (st-a)^m (1-t^2)^{\frac{d-3}{2}} \d t
		\\&=
		\frac{c_d}{2} \sum_{n=0}^m \binom{m}{n} s^n(-a)^{m-n} \int_{a^2/s^2}^{1} t^{\frac{n-1}{2}} (1-t)^{\frac{d-3}{2}} \d t
		\\&=
		\frac{c_d}{2} \sum_{n=0}^m \binom{m}{n} s^n(-a)^{m-n} \left(B(\tfrac{n+1}{2},\tfrac{d-1}{2}) - B_{a^2/s^2}(\tfrac{n+1}{2},\tfrac{d-1}{2}) \right).
	\end{align*}
	The claim follows by collecting the terms and Lemma~\ref{lem:eigenf}. \hfill $\Box$
	\\[2ex]
	\textbf{Proof of Theorem \ref{thm:prop_of Phi}.}
	\begin{enumerate}[label=\roman*),noitemsep]
		\item Since $f\in \mathrm{CP}_1(\R)$ by Proposition~\ref{prop:1}, we obtain by Proposition~\ref{prop:cpd1} that $\Phi\in \mathrm{CP}_1(\R^d)$.
		\item 
		By \eqref{eq:slicingRLFI} we know that
		\begin{align*}
			\Phi(x)=F(\|x\|)=\frac{1}{\omega_{d-1}} \int_{\sd} f(\langle x,\xi\rangle)\d \xi \quad \text{ for all } x\in \R^d.
		\end{align*}
		Hence we get $\Phi(0)=F(0)=f(0)=-(\abs*M_2)(0)<0$. 
		\item 
		By Proposition~\ref{prop:of_F} we directly conclude iii).
		\item 
		Since $f$ is $n+2$ times continuously differentiable by Proposition \ref{prop:1} iii), we obtain for any multi-index $\alpha\in \N^d$ with $|\alpha|\le n+2$ that
		\begin{equation*} \label{eq:Phi-derivative}
			\partial^\alpha \Phi(x)=\frac{1}{\omega_{d-1}}\int_{\sd} \xi^\alpha f^{|\alpha|}(\langle x,\xi\rangle)\d \xi \quad \text{ for all } x\in \R^d.
		\end{equation*}
		\item 
		
		Since $f''=2u$ is bounded, the first derivative $f'$ is $2\|u\|_\infty$-Lipschitz continuous. Hence, for $x,y\in \R^d$, we can estimate by \eqref{eq:Phi-derivative}
		\begin{align*}
			|\partial_i \Phi(x)-\partial _i \Phi(y)|&
			\le \frac{1}{\omega_{d-1}} \int_\sd |\xi_i|\, \big|f'(\langle x,\xi\rangle)-f'(\langle y,\xi\rangle)\big|\d \xi\\&
			\le \frac{1}{\omega_{d-1}} \int_\sd 2\|u\|_\infty\|x-y\|\d \xi
			= 2\|u\|_\infty \|x-y\|,
		\end{align*}
		
		so that we obtain
		\begin{align*}
			\| \nabla \Phi(x)-\nabla \Phi(y)\|^2
			= \sum_{i=1}^{d} |\partial _i\Phi(x)-\partial \Phi(y)|^2
			\le {d} (2\|u\|_\infty \|x-y\|)^2. 
		\end{align*}
		\item 
		Since $\nabla \Phi$ is $L$-Lipschitz continuous with $L=\sqrt{d}2\|u\|_\infty$, we know, by \cite[Thm.\ 2.1.5]{Yurii2018} that $\Phi$ is $-L$-convex. 
		Furthermore, $\Phi$ is concave because, for $t\in [0,1]$ and $x,y\in \R^d$, it holds by the concavity of $f$ from Proposition \ref{prop:1} iii) that
		\begin{align*}
			\Phi((1-t)x+ty)
			&\ge \frac{1}{\omega_{d-1}} \int_{\sd} (1-t)f(\langle x,\xi\rangle)+tf(\langle y,\xi\rangle)\d\xi 
			\\&= (1-t)\Phi(x)+t\Phi(y).\tag*{\qedsymbol}
		\end{align*}
	\end{enumerate}
	
	\subsection*{Proofs of Section \ref{sec:sdistance_1}}
	\label{Apdx:sdistance_1}
	\textbf{Proof of Proposition \ref{thm:kme}.}
	Recall that $f=-\abs*u$ is even and $F=\mathcal I_d[f]$ satisfies  \eqref{eq:slicingRLFI}.
	Let $\mu\in \mathcal M_{\nicefrac{1}{2}}(\Rd)$, then the following integral exists
	\begin{align*}
		\|\mu\|_{\mathcal H_K}^2
		={}&\int_{\Rd}\int_{\Rd} K(x,y)\d\mu(x)\d\mu(y)
		\\={}&\int_{\Rd}\int_{\Rd} \left(F(\|x-y\|)-F(\|x\|)-F(\|y\|)\right)\d\mu(x)\d\mu(y)\\
		={}&\int_{\Rd}\int_{\Rd}\frac{1}{\omega_{d-1}} \int_\sd f(\langle x-y,\xi\rangle)-f(\langle x,\xi\rangle)-f(\langle -y,\xi\rangle)+f(0)\d\xi \d\mu(x)\d\mu(y)
		\\&-F(0)|\mu(\Rd)|^2.
	\end{align*}
	Denote by $\mathcal{T}_y$ the translation operator $ \mathcal T_y[g](y)=g(x-y)$, by $\mathcal M_y$ the modulation operator $\mathcal M_y[g](x)= \e^{-2\pi \i \langle x,y\rangle }g(x)$ and by $g_m(x)=\sqrt{\nicefrac{m}{\pi}}\e^{-mx^2}$ the Gaussian approximate identity as in \cite[Thm.~5.20]{Wendland2004}. Since $f$ is continuous and slowly increasing, \cite[Thm.~5.20 (4)]{Wendland2004} yields
	\begin{align*}
		f(\langle x-y,\xi\rangle)-f(\langle x,\xi\rangle)-f(\langle -y,\xi\rangle)+f(0)
		=\lim_{m\to \infty} \langle (\id -\mathcal T_{\langle \xi,x\rangle})[(\id-\mathcal T_{\langle \xi,-y\rangle })[g_m]],f\rangle.
	\end{align*}
	Let $\varphi_m\coloneqq (\id-\mathcal M_{-\langle \xi,x\rangle})[(\id-\mathcal M_{-\langle \xi,-y\rangle })[\hat g_m]]\in \mathcal S_2(\R^1)$. The function $f$ has the generalized Fourier transform 
	$$
	\hat f(r)=\frac{\hat u(r)}{2\pi^2 r^2}
	$$
	of order $1$ by Lemma \ref{lem:four_conv} and \eqref{fabs}. As $g_m$ is even, we have $\hat{\hat{g}}_m=g_m$ and for all $m\in \N$, $m\ge 1$ we can write
	\begin{align*}
		\langle (\id-\mathcal T_{\langle \xi,x\rangle})[(\id-\mathcal T_{\langle \xi,-y\rangle })[g_m]],f\rangle &
		=\langle \hat \varphi_m , f\rangle 
		=\langle \varphi_m,\hat f\rangle\\&
		=  \int_\R (1- \e^{2\pi \i\langle \xi,x\rangle r})(1- \e^{-2\pi \i\langle \xi,y\rangle r })\hat g_m(r) \hat f(r)\d r.
	\end{align*}
	Since $u$ is continuous with compact support, its Fourier transform is bounded. The term $r\mapsto (1- \e^{2\pi \i\langle \xi,x\rangle r})(1- \e^{-2\pi \i\langle \xi,y\rangle r })$ is bounded and has a zero of order $2$ at zero, so that
	$$\hat f(r) (1- \e^{2\pi \i\langle \xi,x\rangle r})(1- \e^{-2\pi \i \langle \xi,y\rangle r })
	=\frac{\hat u(r)}{2\pi^2r^2}(1- \e^{2\pi \i\langle \xi,x\rangle r})(1- \e^{-2\pi \i\langle \xi,y\rangle r })$$
	is integrable.
	Moreover, we have
	$$|\hat g_m(r)|=\e^{-\frac{\pi^2r^2}{m}}\le {1}\forrall r\in \R\text{ and } m\in \N, m\ge 1.$$
	Since $\hat g_m$  converges pointwise to the constant ${1}$, Lebesgue's convergence theorem yields
	\begin{align*}
		f(\langle x-y,\xi\rangle)-f(\langle x,\xi\rangle)-f(\langle -y,\xi\rangle)+f(0)
		= \int_\R (1- \e^{2\pi \i\langle \xi,x\rangle r})(1- \e^{-2\pi \i \langle \xi,y\rangle r }) \hat f(r)\d r .
	\end{align*}
	Further, we obtain using Fubini's theorem 
	\begin{align*}
		&{\omega_{d-1}}(\|\mu\|_{\mathcal H_K}^2+F(0)\mu(\Rd)^2)
		\\&
		=\int_\sd \int_\R \int_\Rd (1-\e^{2\pi \i\langle r\xi,x\rangle})\d\mu(x) \int_\Rd (1-\e^{-2\pi \i\langle r\xi,y\rangle})\d\mu(y) \hat f(r)\d r\d\xi\\&
		=\int_\sd \int_\R |\mu(\Rd)-\hat \mu(r\xi)|^2\hat f(r)\d r\d\xi\\&
		=2\int_\sd \int_0^\infty |\mu(\Rd)-\hat \mu(r\xi)|^2 \frac{\hat f(\|r\xi\|)}{\|r\xi\|^{d-1}}r^{d-1}\d r\d\xi \\&
		=2\int_\Rd |\mu(\Rd)-\hat \mu(x)|^2\frac{\hat f(\|x\|)}{\|x\|^{d-1}}\d x.
	\end{align*}
	Inserting $\hat f$, we can write
	\begin{equation} \label{eq:normHKmu}
		\|\mu\|_{\mathcal H_K}^2=\frac{1}{\pi^2\omega_{d-1}}\int_\Rd|\hat\mu(0)-\hat \mu(x)|^2\frac{\hat u(\|x\|)}{\|x\|^{d+1}}\d x -F(0) \hat \mu(0)^2.
	\end{equation}
	Now assume that $\mu\in \mathcal M_{\nicefrac{1}{2}}(\Rd)$ with $\|\mu\|_{\mathcal H_K}=0$. 
	Because $F(0)<0$ by Theorem~\ref{thm:prop_of Phi} and $\hat u\ge0$ by \eqref{eq:Un}, both summands in \eqref{eq:normHKmu} are nonnegative and therefore must vanish.
	The second term yields that $\mu(\R^d)=\hat \mu=0$.
	Since $\supp \hat u(\|\cdot\|)\|\cdot\|^{-(d+1)}=\Rd$, cf.\ \cite[Lem~2.39]{plonka2018numerical}, it follows that $\hat \mu$ is constant with $\hat \mu \equiv \hat\mu(0)=0$. This implies $\mu=0$ as the Fourier transform $\mathcal F\colon \mathcal M(\Rd)\to  \mathcal C_b(\Rd)$ is injective. Consequently, the KME is injective, which means that $K$ is characteristic.
	\hfill\qedsymbol
	\\[2ex]
	\textbf{Proof of Theorem~\ref{prop:disc_of_ker}.}
	Since $\Phi(x)\in \mathcal O(\|x\|^\alpha)$, we can estimate $|\Phi(x)|\le C(1+\|x\|^\alpha)$ for all $x\in \R^d$.
	For $\alpha\ge 1$ we have by convexity of $\|\cdot\|^\alpha$ that
	$$\|x+y\|^\alpha \le 2^{\alpha-1} (\|x\|^\alpha+\|y\|^\alpha).$$
	For $\alpha\in (0,1)$, we define the function $f\colon [0,\infty)\to \R$ by $f(x)\coloneqq x^\alpha$, which is concave and monotone increasing. Then we obtain for $x,y\ge 0$ with $x+y>0$
	that
	\begin{align*}
		f(x)&\ge \frac{y}{x+y}f(0)+\frac{x}{x+y}f(x+y),\\
		f(y)&\ge \frac{x}{x+y}f(0)+\frac{y}{x+y}f(x+y).
	\end{align*}
	Adding both equation yields 
	\begin{equation*}
		f(x)+f(y)\ge f(x+y).
	\end{equation*}
	Since $f$ is monotone increasing, we obtain by the triangle inequality
	\begin{align*}
		\|x+y\|^\alpha\le (\|x\|+\|y\|)^\alpha \le \|x\|^\alpha +\|y\|^\alpha.    
	\end{align*}
	Summarizing,  we have for $\alpha\ge 0$ that
	$$\|x+y\|^\alpha \le 2^\alpha (\|x\|^\alpha+\|y\|^\alpha).$$
	Therefore, we can guarantee the existence of the integral 
	\begin{align*}
		\int_{\R^d} \int_{\R^d} |\Phi(x-y)| \d\sigma(x)\d\sigma(y)&
		\le \int_{\R^d}\int_{\R^d} C(1+2^\alpha(\|x\|^\alpha+\|y\|^\alpha))\d\sigma(x)\d\sigma(y)\\&
		\le C \sigma(\R^d)\left(\sigma(\R^d) +2^{\alpha+1} \int_{\R^d} \|x\|^\alpha\d\sigma(x)\right)<\infty.
	\end{align*}
	Hence, the discrepancy $d_{\tilde K}(\mu,\nu)$ is well-defined
	for $\mu,\nu \in \mathcal M_\alpha(\R^d)$.
	
	By \eqref{make_pos_def}, we see that
	$K(x,x) \in \mathcal O(\|x\|^{2\beta})$, 
	$\beta \coloneqq \max\{r-1,(\alpha + r-1)/2\}$
	such that
	$d_K$ is by \eqref{k1} well-defined for measures in $\mathcal M_\beta$.
	
	Now assume additionally that the first $r-1$ moments of $\mu$ and $\nu$ coincide. This implies that for all $p_j \in \Pi_{r-1}(\R^d)$ that
	$$\int_{\R^d} p_j(x)\d(\mu-\nu)(x)=0.$$
	Then we obtain
	\begin{align*}
		d_K(\mu,\nu)^2
		=d_{\tilde K}(\mu, \nu)^2 
		&-\sum_{j=1}^N \int_{\R^d}p_j(x)\d(\mu-\nu)(x) \int_{\R^d} \Phi(\xi_j-y)\d(\mu-\nu)(y)\\&
		-\sum_{k=1}^N \int_{\R^d}p_k(y)\d(\mu-\nu)(y) \int_{\R^d} \Phi(x-\xi_k)\d(\mu-\nu)(x)\\&
		+\sum_{k,j=1}^N \Phi(\xi_j-\xi_k)\int_{\R^d}p_j(x)\d(\mu-\nu)(x) \int_{\R^d} p_k(y)\d(\mu-\nu)(y)\\&
		= d_{\tilde K}(\mu,\nu)^2. \tag*{\qedsymbol}
	\end{align*}
	\\[2ex]
	\textbf{Proof of Proposition~\ref{prop:RKHS_M_conv}.}
	1. First, we show that $K$ is $\big \lfloor \frac{n+2}{2} \big \rfloor $ times continuously differentiable.
	Let $\alpha \in \N^d$ with $|\alpha|\le  \big\lfloor \frac{n+2}{2} \big\rfloor$.
	The case $|\alpha| = 0$ is clear. 
	For $|\alpha|\ge 1$, we obtain
	\begin{align}
		\partial_x^\alpha \partial_y^\alpha K(x,y) 
		& =\partial_x^{\alpha} \partial_y^{\alpha}F(\|x-y\|)
		=\partial_x^{\alpha} \partial_y^{\alpha}\frac{1}{\omega_{d-1}} \int_{\sd} f(\langle \xi,x-y\rangle)\d\xi \\&
		= \frac{(-1)^{|\alpha|}}{\omega_{d-1}} \int_\sd \xi^{2\alpha} f^{2|\alpha|}(\langle \xi,x-y\rangle)\d \xi.\label{eq:diff_of_S}
	\end{align}
	By \cite[Cor. 4.36]{steinwart2008supportderivfeature}, this implies that every $h\in \mathcal H_K$ is at least $\big \lfloor \frac{n+2}{2} \big \rfloor $-times continuously differentiable.
	
	2. For the second part, assume that $n\ge 2$.
	By \cite[Lem 4.34]{steinwart2008supportderivfeature} and \eqref{eq:diff_of_S}, we obtain
	\begin{align*}
		\langle \partial_{x_i} K(x, \cdot),\partial_{x_i} K(y, \cdot)\rangle _{\mathcal H_K}&
		=\partial_{x_i}\partial_{y_i}K(x,y)
		=\frac{-1}{\omega_{d-1}}\int_\sd \xi_i^2 f''(\langle \xi,x-y\rangle)\d\xi.
	\end{align*}
	By Proposition \ref{prop:1}, the function $f$ is even and $f'''=2u'$. Hence $u'$ is odd and $\|u''\|_\infty$-Lipschitz continuous and $f'''(0)=2u'(0)=0$. Thus, we obtain for $s>0$ that
	\begin{align*}
		|f''(s)-f''(0)|
		\le \int_0^s| f'''(t)|\d t
		=\int_0^s |f'''(t)-f'''(0)|\d t
		\le \int_0^s 2\|u''\|_\infty t\d t
		=\|u''\|_\infty s^2.
	\end{align*}
	Hence we can estimate
	\begin{align*}
		\|\partial_{x_i} K(x,\cdot)-\partial_{x_i} K(y,\cdot)\|_{\mathcal H_K}^2&
		=\|\partial_{x_i} K(x,\cdot)\|_{\mathcal{H}_K}^2\!+\|\partial_i K(y,\cdot)\|_{\mathcal{H}_K}^2\!-2\langle \partial_{x_i}  K(x,\cdot),\partial_{x_i} K(y,\cdot)\rangle _{\mathcal H_K}
		\\
		&=-\frac{2}{\omega_{d-1}}\int_\sd \xi_i^2 \left(f''(0)-f''(\langle \xi,x-y\rangle) \right)\d\xi
		\\
		&\le \frac{2\|u''\|_\infty}{\omega_{d-1}}\int_\sd |\langle \xi,x-y\rangle|^2\d\xi 
		\\
		&\le 2\|u''\|_\infty  \|x-y\|^2.
	\end{align*}
	Therefore, $\partial_{x_i} K(x,\cdot)$ is Lipschitz continuous with constant
	$L \coloneqq \sqrt{2\|u''\|_\infty}$.
	Finally, we see again by (the proof of) \cite[Cor. 4.36]{steinwart2008supportderivfeature}, for any $h \in \mathcal H_K$, that
	\begin{align*}
		|\partial_{x_i} h(x)-\partial_{x_i} h(y)|
		&=|\langle h,\partial_{x_i} K(x,\cdot)\rangle_{\mathcal H_K}
		-\langle h,\partial_{x_i} K(y,\cdot)\rangle_{\mathcal H_K}|\\
		&=|\langle h,\partial_{x_i} K(x,\cdot) -\partial_{x_i} K(y,\cdot)\rangle_{\mathcal H_K}|
		\\
		&\le 
		\|h\|_{\mathcal H_K} L \|x-y\|,
	\end{align*}
	which gives the assertion by
	$$
	\|\nabla h(x)-\nabla h(y)\|
	\le \sqrt{d} L \|h\|_{\mathcal H_K}  \|x-y\|.
	\qquad \qedsymbol
	$$
	
	\section{Geodesic  \texorpdfstring{$\lambda$}{lambda}-Convexity of MMD Functional with Smoothed Distance Kernel}\label{app:geodesic}
	A \emph{generalized geodesic} is an interpolating curve $\gamma_t\colon [0,1]\to \mathcal P_2(\R^d)$ that connects two measures $\mu^2$ and $\mu^3\in \mathcal P_2(\R^d)$ via a three-plan $\mu$. More specifically, for a base $\mu^1\in \mathcal P(\R^d)$, this three-plan $\mu\in \mathcal P_2(\R^d\times \R^d\times \R^d)$ has marginals $\pi^i_\# {\mu} = \mu^i, i=1,2,3$ and must be optimal in the sense that $\pi^{1,i}_\#\mu\in \Pi_\textup{opt}(\mu^1,\mu^i)$ for $i=1,2$.
	A generalized geodesic $\gamma_t$ joining $\mu^2$ with $\mu^3$ via $\mu^1$ is defined as $\gamma_t\coloneqq ((1-t)\pi^2+t\pi^3)_\#\mu$. 
	For any choice of $\mu^1,\mu^2,\mu^3\in \mathcal P_2(\R^d)$, we can always find optimal plans $\mu^{1,i}\in \Pi_\textup{opt}(\mu^1,\mu^i)$ for $i=1,2$, and by the Gluing Lemma \cite{villani2009optimal} there exists a three plan $\mu\in \mathcal P_2(\R^d\times \R^d\times \R^d)$ with marginals $\pi^i_\# {\mu} = \mu^i, i=1,2,3$ and $\pi^{1,i}_\#\mu\in \Pi_\textup{opt}(\mu^1,\mu^i)$ for $i=1,2$. This means that there always exists at least one generalized geodesic joining $\mu^2$ with $\mu^3$ via $\mu^1$. However, this generalized geodesic is not necessarily unique.
	
	Given $\lambda\in \R$, a function $G\colon \mathcal P_2(\R^d)\to [0,\infty]$ is \emph{$\lambda$-convex along generalized geodesics}, if  for any choice $\mu^1,\mu^2,\mu^3\in \dom(G)$ there always exists a generalized geodesic $\gamma_t$ joining $\mu^2$ with $\mu^3$ via $\mu^1$, such that 
	\begin{equation*}
		G(\gamma_t)\le (1-t)G(\mu^2)+tG(\mu^3)-\frac{\lambda}{2}t(1-t)\int_{\R^d\times \R^d\times \R^d} \|x_2-x_3\|^2\d\mu(x_1,x_2,x_3).
	\end{equation*}
	For a more detailed description we refer to \cite[Section 9.2]{ambrosio2005gruenesbuchgradient}.
	
	In \cite{ambrosio2005gruenesbuchgradient} sufficient conditions for the $\lambda$-convexity of the following two typical energy functionals were given.
	The \emph{potential energy}  $V\colon \R^d\to \R$ is defined by
	\begin{align*}
		\mathcal V(\mu)\coloneqq \int_{\R^d} V(x)\d\mu(x).
	\end{align*}
	
	\begin{lemma}\label{prop:pot_E}
		Let $V$ be lower semi-continuous and have quadratic grow, i.e. 
		$$
		V(x)\ge -A-B\|x\|^2 \quad \text{for all} \quad  x \in \R^d
		$$
		with $A,B\in \R$.
		If $V$ is $\lambda$-convex, then $\mathcal V$ is $\lambda$-convex along generalized geodesics.
	\end{lemma}
	
	For $K\colon \R^d\times \R^d\to \R$, the \emph{interaction energy} is given by
	\begin{align*}
		\mathcal K(\mu)\coloneqq \int_{\R^d\times \R^d} K(x,y)\d\mu(x)\d\mu(y).
	\end{align*}
	Since the interaction energy can be seen as a potential energy on the product space, Proposition \ref{prop:pot_E} also applies to the interaction energy.
	
	\begin{lemma}
		\label{prop:intact_E}
		Let $K$ be lower semi-continuous and have  quadratic grow, i.e. 
		$$
		K(x,y)\ge -A-B(\|x\|^2+\|y\|^2) \quad \text{for all} \quad  x,y\in \R^d
		$$
		with $A,B\in \R$.
		If $K$ is $\lambda$-convex, then $\mathcal K$ is $\lambda$-convex along generalized geodesics.
	\end{lemma}
	
	Let $F$ be defined as in \eqref{def} and $K(x,y) = F(\|x-y \|)$. Then, the MMD functional $G$ from \eqref{eq:GK} can be rewritten as
	\begin{equation}\label{eq:def_VandW}
		G(\mu)=\frac12\mathcal K(\mu) +\mathcal V(\mu)+ \frac12 c_\nu, \quad  V(x)\coloneqq -\int_{\R^d} F(\|x-y\|)\d\nu(y))
	\end{equation}
	where $c_\nu\ge0$ is a constant.
	Both $V$ and $K$ suffice the conditions in Lemmas \ref{prop:pot_E} and \ref{prop:intact_E}.
	
	\begin{proposition}\label{prop:G_is_lam_cvx}
		Let $F$ be defined as in \eqref{def} and $V$ and $ K$ in \eqref{eq:def_VandW}, then $\mathcal V$ and $\mathcal K$ are 
		lower semi-continuous and have quadratic grow. Moreover, $V$ is convex and $K$ is $\lambda$-convex with $\lambda = -4\sqrt{d}\|u\|_\infty $.
		In  summary, the MMD functional $G$ from \eqref{eq:def_VandW} is lower semi-continuous and $\lambda$-convex with $\lambda$ above.
	\end{proposition}
	
	\begin{proof}
		By Theorem~\ref{thm:prop_of Phi} , we can write $F(s)=-C_d|s|+\varphi(s)$ with $\varphi\in \mathcal C_0(\R)$.
		For the lower semi-continuity of $V$, we have
		\begin{align*}
			|V(x_1)-V(x_2)|&
			\le \int_{\R^d} |F(\|x_1-y\|)-F(\|x_2-y\|)|\d\nu(y)\\&
			=\int_{\R^d} |C_d\|x_1-y\|+\varphi(\|x_1-y\|)-C_d\|x_2-y\|-\varphi(\|x_2-y\|)|\d\nu(y)\\&
			\le \int_{\R^d} C_d|\|x_1-y\|-\|x_2-y\||+|\varphi(\|x_1-y\|)-\varphi(\|x_2-y\|)|\d\nu(y)\\&
			\le C_d\|x_1-x_2\|+\int_{\R^d} \|\varphi(\|x_1-y\|)-\varphi(\|x_2-y\|)|\d\nu(y).
		\end{align*}
		Since $\varphi\in  \mathcal C_0(\R)$, it follows by Lebesgue's dominated convergence theorem that $V$ is continuous.
		Moreover, choosing  $A,B=0$, we see that $V$ has  quadratic grow.
		By Theorem \ref{thm:prop_of Phi} iii), the function $-F(\|x\|)$  is convex. For $x_1,x_2\in \R^d$ and $t\in [0,1]$, it holds 
		\begin{align*}
			V((1-t)x_1+tx_2)&
			=\int_{\R^d} -F(\|(1-t)(x_1-y)+t(x_2-y)\|)\d\nu(y)\\&
			\le \int_{\R^d} -(1-t)F(\|x_1-y\|)-tF(\|x_2-y)\|)\d\nu(y)\\&
			=(1-t)V(x_1)+tV(x_2).
		\end{align*}
		Hence $V$ is convex, too.
		
		For the interaction energy, it is clear that $K$ is continuous, because $F$ is continuous, and we can choose $A=\|\varphi\|_\infty+2C_d$ and $B=C_d$ to obtain
		\begin{align*}
			F(\|x-y\|)&
			=-C_d\|x-y\|+\varphi(\|x-y\|)
			\ge -\|\varphi\|_\infty -C_d(\|x\|+\|y\|)\\&
			\ge -\|\varphi\|_\infty -C_d(1+\|x\|^2+1+\|y\|^2)
			\ge -A -B(\|x\|^2+\|y\|^2).
		\end{align*}
		By Corollary \ref{thm:prop_of Phi} v), we get that $K$ is $\lambda$-convex with $\lambda = - 4\sqrt{d}\|u\|_\infty $.
	\end{proof}
	
	By \cite[11.2.1b]{ambrosio2005gruenesbuchgradient}, a lower bounded $\lambda$-convex functional is always coercive, so that \cite[Thm.11.2.1]{ambrosio2005gruenesbuchgradient} can be formulated as follows:
	\begin{theorem} \label{thm:lam_cvx_WGF}
		Let $G \colon \mathcal P_2(\R^d)\to [0,\infty]$ be proper lower semi-continuous and $\lambda$ convex. Then, for $\gamma^{(0)}\in \dom(G)$, there is a unique Wasserstein gradient flow $\gamma_t$ starting in $\gamma^{(0)}$. Moreover,
		the piecewise constant curve $\gamma^\tau_t\coloneqq \gamma^{(k)}$, $t\in ((k-1)\tau,k\tau]$
		given by the implicit Euler scheme (JKO scheme)
		\begin{equation}\label{eq:impl_eul}
			\gamma^{(k+1)} \in  \argmin_{\gamma\in \mathcal P_2(\R^d)} \frac{1}{2\tau } W_2^2(\gamma^{(k)},\gamma)+\phi(\gamma),
		\end{equation}
		converges locally uniformly to $\gamma_t$.
		In particular, this holds true for our MMD functional with smooed kernel $G$ in \eqref{eq:def_VandW}.
	\end{theorem}
	
	It was shown in \cite[Prop.~9]{HGBS2024} that for certain functionals, e.g. $G$ in  \eqref{eq:def_VandW}, the so-called Wasserstein steepest descent flows (explicit scheme) and the Wasserstein gradient flows (implicit scheme) coincide. 
	
	Since the MMD functional with our SND kernel (same for the Gaussian kernel)
	is only $\lambda$-convex with $\lambda < 0$ it is not ensured that its Wasserstein gradient flow, resp. its approximation by an Euler forward scheme  converges towards the 
	target $\nu$. Here is an example.
	\begin{example} \label{ex:no_target}
		In general, it is not clear whether the gradient flow $\gamma_t$ the MMD functional  with smooth kernels converges towards the target measure $\nu$ as $t \to \infty$.
		To this end, consider the symmetric setup with the target and initial measures  
		\begin{align}
			\nu &\coloneqq \tfrac{1}{2} (\delta_{y_1} + \delta_{y_2}), \quad y_1 = e_1, \; y_2 = -e_1,\\
			\mu &\coloneqq \tfrac{1}{2} (\delta_{x_1} + \delta_{x_2}),
			\quad x_1 = \tfrac{1}{\sqrt{3}} e_2, \; x_2 = -\tfrac{1}{\sqrt{3}} e_2.
		\end{align}
		Then, with $\tilde F(s) \coloneqq \frac{F'(s)}{s}$, the velocity field becomes for $i = 1,2$
		\begin{align*}
			v_t(x_i) &=
			\tfrac{1}{2} (x_i - x_j) \tilde{F}(\|x_i - x_j\|)
			- \tfrac{1}{2} \left( (x_i - y_1) \tilde{F}(\|x_i - y_1\|) + (x_i - y_2) \tilde{F}(\|x_i - y_2\|) \right) \\
			&= \tfrac{1}{\sqrt{3}} e_2 \tilde{F}(\tfrac{2}{\sqrt{3}}) - \tfrac{1}{2} \left( \tfrac{2}{\sqrt{3}} e_2 - e_1 + e_1 \right) \tilde{F}(\tfrac{2}{\sqrt{3}}) 
			= 0,
		\end{align*}
		so that we get stuck in  $\gamma_t = \mu$ for all $t \geq 0$.
		In general, ensuring convergence towards the target is challenging  due to local extrema. 
		However, in the numerical experiments, we observed that for both ND and SND, the flows typically performed well in approximating the target.
	\end{example}
	\noindent
	\textbf{Proof of Proposition~\ref{prop:dirac_flow}.}
	The simplification of the iterations to
	\begin{align} \label{eq:single_update}
		x^{(k+1)} &= x^{(k)} + \tau \frac{x^{(k)} - y}{\|x^{(k)} - y\|} F'(\|x^{(k)} - y\|) \\
		&= y + (x^{(k)} - y) \left(1 + \tau \frac{F'(\|x^{(k)} - y\|)}{\|x^{(k)} - y\|} \right)
	\end{align}
	is straightforward.
	
	i) For $F = -\frac12 \abs$, we have $F'(s) = -\frac{1}{2}$ for $s > 0$. 
	Then we obtain by \eqref{eq:single_update} and 
	since $\|x^{(0)} - y\| < \frac{\tau}{2}$ that
	\begin{align} \label{eq:x_k-y}
		\|y - x^{(1)}\| =  \frac{\tau}{2} - \|x^{(0)} - y\| < \frac{\tau}{2}.
	\end{align}
	The second step, $x^{(2)}$ jumps exactly back to $x^{(0)}$ because
	\begin{align*}
		x^{(2)} &= y + (x^{(1)} - y) \left(1 - \frac{\tau}{2\|x^{(1)} - y\|}\right) 
		\\
		&= y + (x^{(0)} - y) \left(1 - \frac{\tau}{2\|x^{0} - y\|}\right) \left(1 - \frac{\tau}{2\|x^{(1)} - y\|}\right) \\
		&= y + (x^{(0)} - y) \frac{(2\|x^{(0)} - y\| - \tau)(- 2\|x^{(0)} - y\|)}{2\|x^{(0)} - y\| (\tau - 2\|x^{(0)} - y\|)} 
		= x^{(0)}.
	\end{align*}
	Consequently, $(x^{(k)})_k$ oscillates between $x^{(0)}$ and $x^{(1)}$.
	
	ii)
	Generally, for $\lambda$-convex functionals with $\lambda > 0$, Baillon-Haddad's theorem \cite[Cor.~18.17]{BHC2017} ensures convergence of \eqref{eq:dirac_up} for $\tau < \lambda^{-1}$. For completeness, we provide a simpler proof for our setting.
	Let $F = \mathcal{I}_d[-|u|]$, with $u \in \mathcal{U}^0(\mathbb{R})$. We know that $F$ is convex and twice differentiable. In particular, we have for 
	$\tilde F(s) \coloneqq \frac{F'(s)}{s}$ 
	that 
	$\tilde{F}(0) = F''(0) < 0$. 
	Since $\tilde{F} \in \mathcal{C}(\mathbb{R})$ by Proposition \ref{prop:of_F}, we can find $\delta > 0$ such that $F(x) < \tfrac{1}{2} F''(0)$ for $|x| < \delta$. If we assume that $\|x^{(k)} - y\| < \tau$ and $\tau < \min\{\delta, \|\tilde{F}'\|^{-1}_\infty\}$, we obtain
	\begin{align*}
		\|x^{(k+1)} - y\| = \|x^{(k)} - y\| \left(1 + \tau \tilde{F}(\|x^{(k)} - y\|)\right).
	\end{align*}
	We always have
	\begin{align*}
		1 + \tau \tilde{F}(\|x^{(k)} - y\|) > 1 - \tau \|\tilde{F}\|_\infty > 0.
	\end{align*}
	Since $\|x^{(k)} - y\| < \delta$, we know that $F(\|x^{(k)} - y\|) < \tfrac{1}{2}F''(0) < 0$, which implies
	\begin{align*}
		1 + \tau \tilde{F}(\|x^{(k)} - y\|) < 1 + \tfrac{\tau}{2} F''(0) < 1.
	\end{align*}
	This yields $\|x^{(k+1)} - y\| < \|x^{(k)} - y\| < \delta$, and thus, by induction,
	\begin{align*}
		\|x^{(k+1)} - y\| \le \|x^{(0)} - y\| (1 + \tfrac{\tau}{2} F''(0))^{k}.
	\end{align*}
	Therefore, we have exponential convergence when $\tau$ and $\|x^{(k)} - y\|$ are sufficiently small.
	\hfill $\Box$
	
	{\section{Additional Numerical Results}\label{app:annulus}%

		\paragraph{Comparison of Filters $M_n$.}
		
		\begin{figure}
			\centering
			\includegraphics[width=0.45\linewidth]{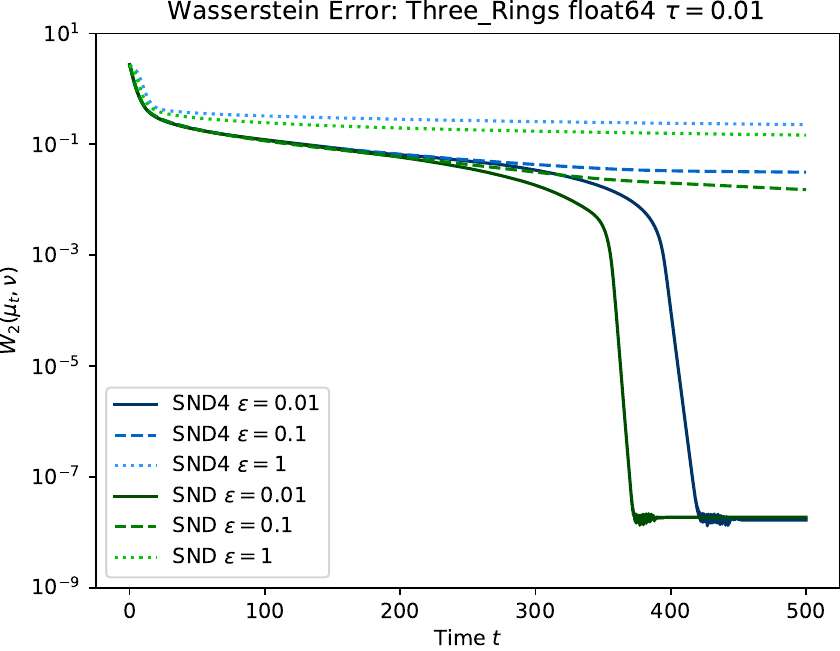}
			\includegraphics[width=0.45\linewidth]{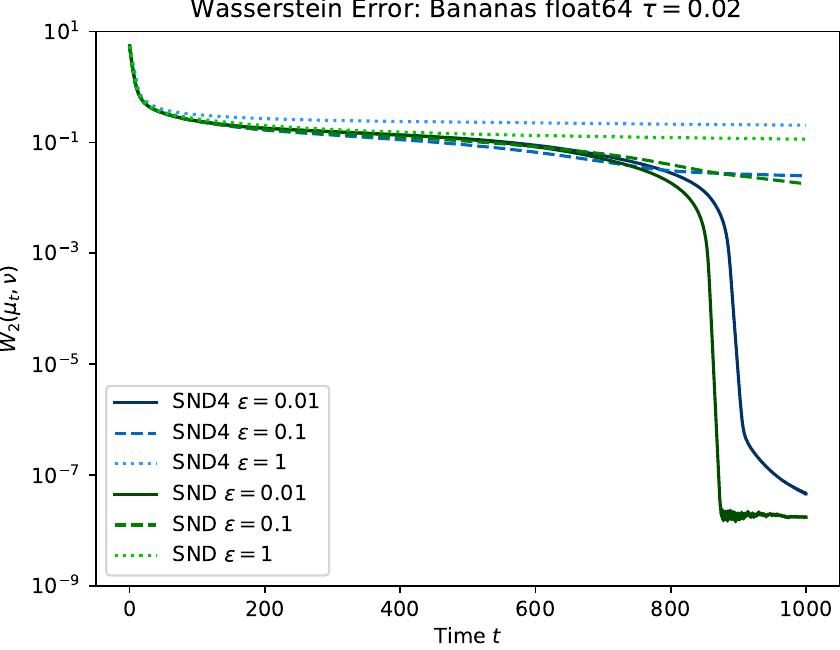}
			\caption{Wasserstein 2 error between target $\nu$ and flow $\gamma_n^\tau$ after $n$ iterations. Horizontal axis in time $t=\tau n$. Both computed with double precision and step size $\tau=0.01$ for the Three-Rings target (left) and $\tau=0.02$ for the Bananas target (right). }
			\label{fig:Bananas_and_3rings_M4}
		\end{figure}%
		
		Figure \ref{fig:Bananas_and_3rings_M4} shows the Wasserstein error between the flow and the target for the SND kernel smoothed with $u=M_2$ and $u=M_4$. 
		Here, we denote by \textbf{SND4} the smoothed negative distance $F\coloneqq -\mathcal I_3[\abs *u_\varepsilon]$ for $u_\varepsilon(x)=\frac{1}{\varepsilon}M_4(\frac{x}{\varepsilon})$. 
		We keep the notation \textbf{SND} if we smooth with $u_\varepsilon(x)=\frac{1}{\varepsilon}M_2(\frac{x}{\varepsilon})$.
		Both SND and SND4 exhibit comparable error decay. Visually, the flows in Figures \ref{fig:Annulus_SND} and \ref{fig:Annulus_SND4} also show similar behavior. This suggests that the choice of the filter $u$ has little impact on the behavior of the gradient flow. 
		Note that using the same $\varepsilon$ with $M_4$ or $M_2$ results in different smoothing strengths, as their supports differ. For large $n$, the derivation of $\mathcal I_3[\abs*M_n]$ becomes increasingly tedious and also the numerical evaluation gets more expensive.

		\paragraph{Annulus Target.}
		The Annulus target consists of two concentric circles with radius $1$ and $0.3$. Each is discretized with $50$ points, so that $\nu$ consists of $M= 100$ points.
		Here we use a step size of $\tau=0.003$ and double precision.
		The MMD flows are depicted in Figure~\ref{fig:Annulus_flow} and the respective errors in Figure~\ref{fig:Annulus_W2}.
		
		\begin{figure}
			\centering
			
			\begin{subfigure}{\textwidth}
				\centering\footnotesize
				\begin{tabular}{c c c c c} 
					\vspace{-4pt}
					\rotatebox{90}{\parbox{2cm}{\centering $\sigma=0.06$}} &
					\includegraphics[width=3cm]{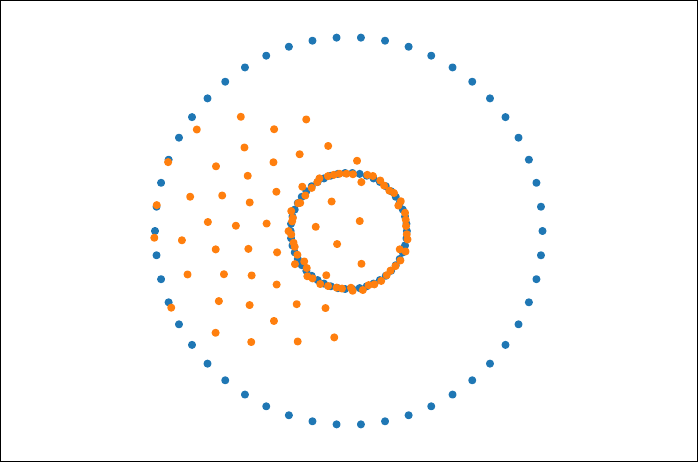} &
					\includegraphics[width=3cm]{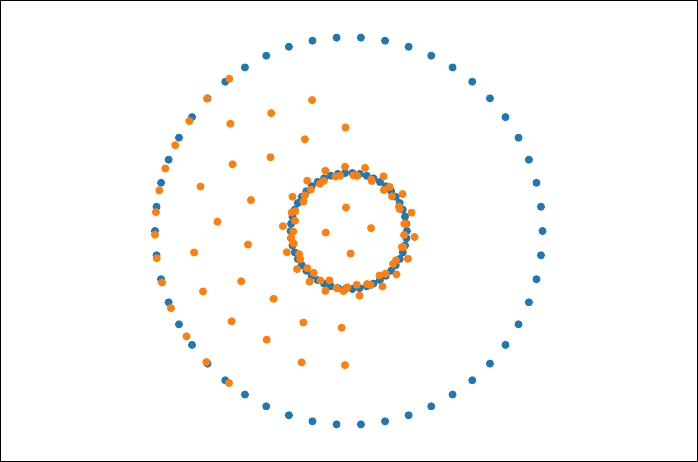} &
					\includegraphics[width=3cm]{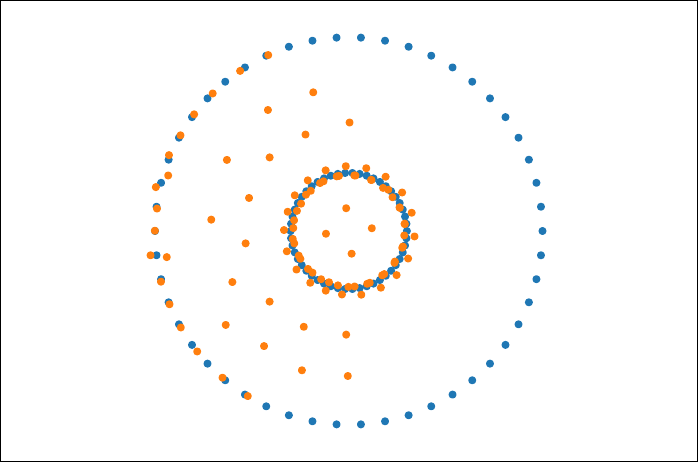} &
					\includegraphics[width=3cm]{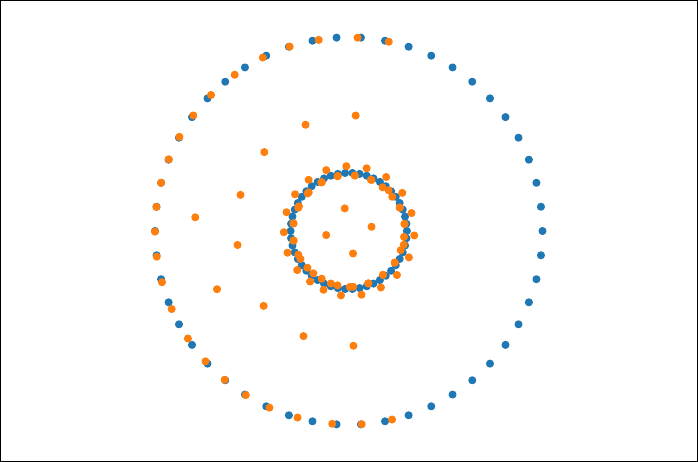} \\
					\vspace{-4pt}
					\rotatebox{90}{\parbox{2cm}{\centering $\sigma=0.3$}} &
					\includegraphics[width=3cm]{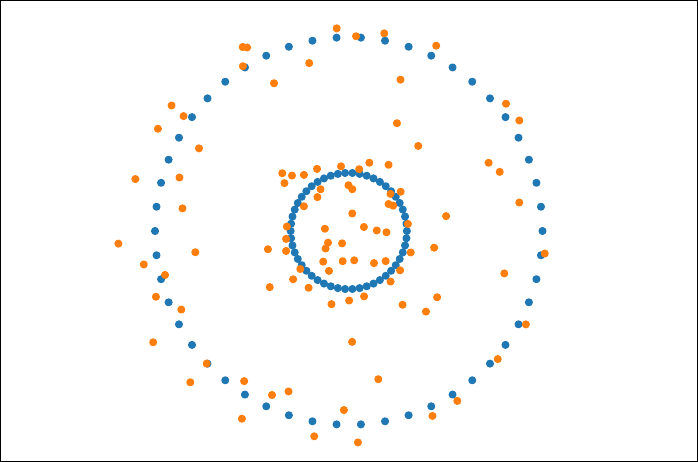} &
					\includegraphics[width=3cm]{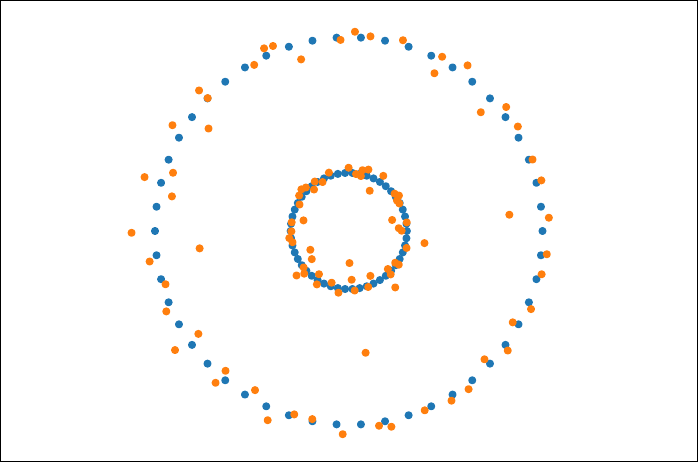} &
					\includegraphics[width=3cm]{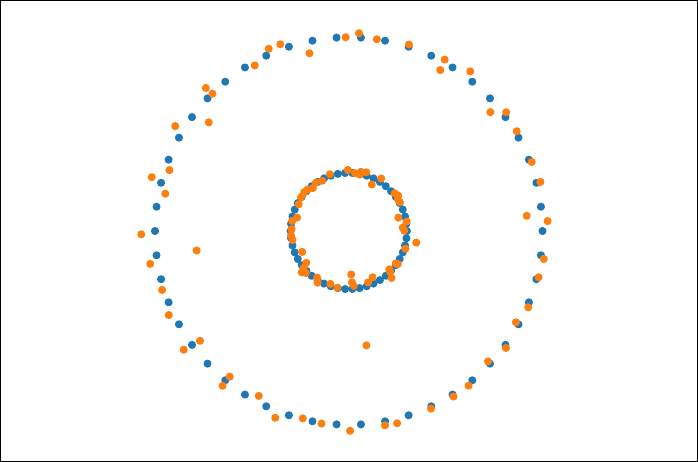} &
					\includegraphics[width=3cm]{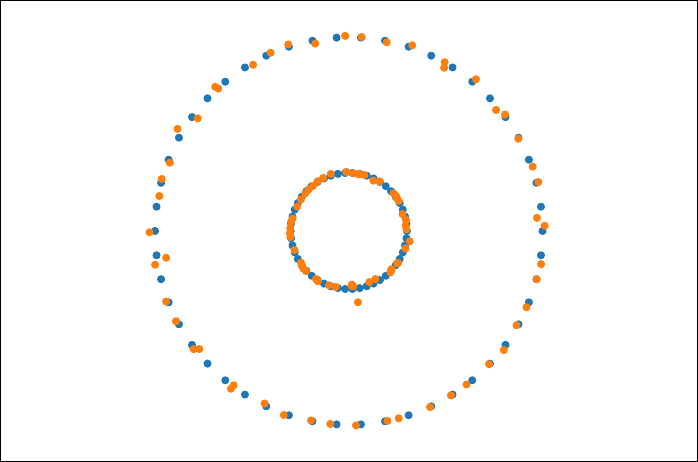} \\
					\vspace{-2pt}
					\rotatebox{90}{\parbox{2cm}{\centering $\sigma=1$}} &
					\includegraphics[width=3cm]{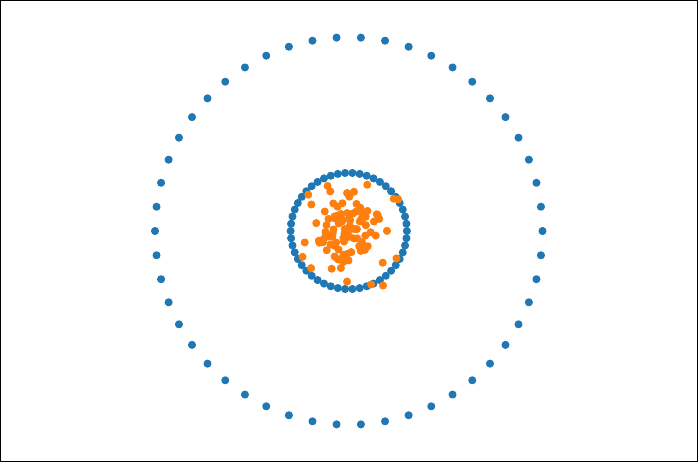} &
					\includegraphics[width=3cm]{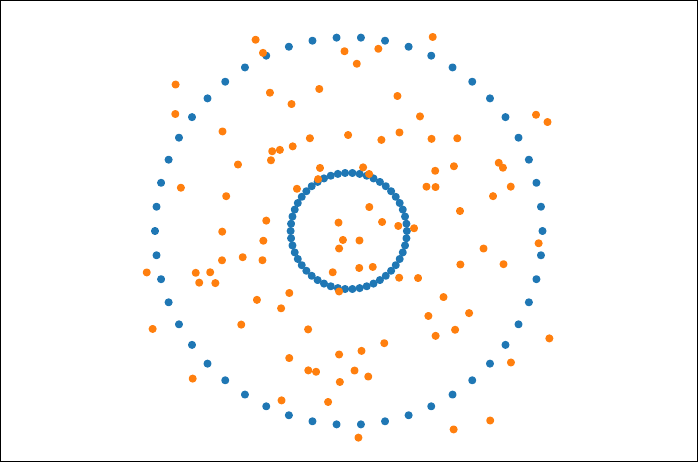} &
					\includegraphics[width=3cm]{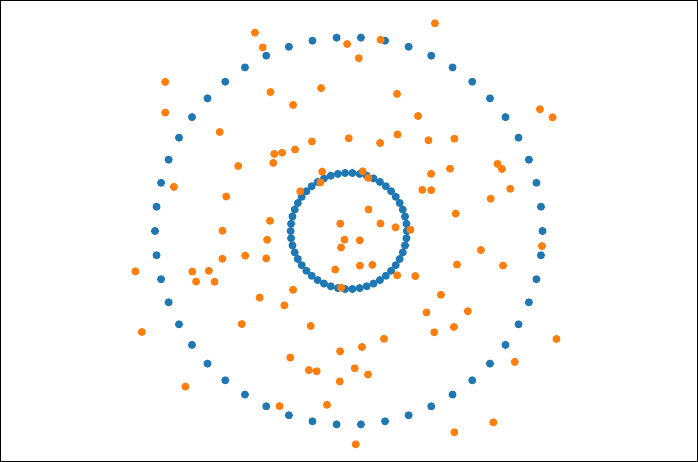} &
					\includegraphics[width=3cm]{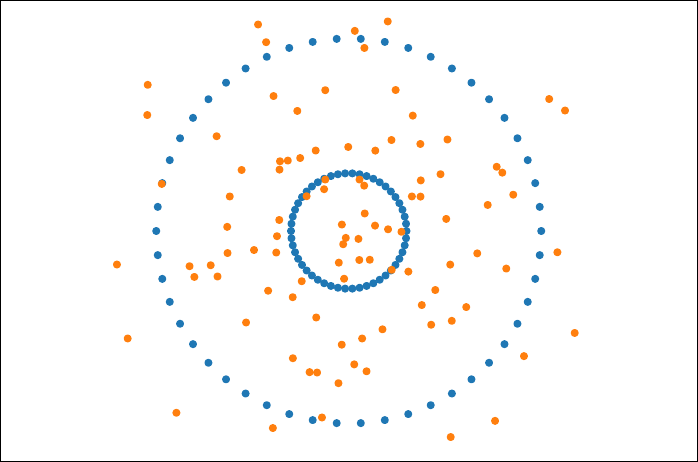} \\
					
					& \multicolumn{1}{c}{$t=3$} & \multicolumn{1}{c}{$t=15$} & \multicolumn{1}{c}{$t=30$} & \multicolumn{1}{c}{$t=150$} \\
				\end{tabular}
				\caption{Gaussian}
				\label{fig:Annulus_GAUSS}
			\end{subfigure}
			
			\begin{subfigure}{\textwidth}
				\centering\footnotesize
				\begin{tabular}{c c c c c} 
					\vspace{-4pt}
					\rotatebox{90}{\parbox{2cm}{\centering $\varepsilon=1$}} &
					\includegraphics[width=3cm]{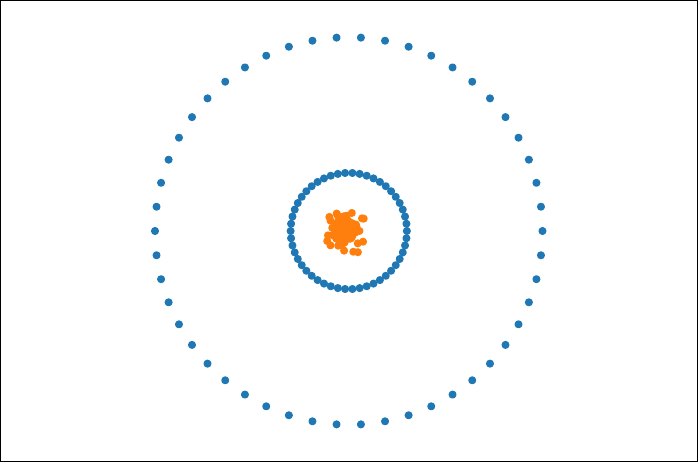} &
					\includegraphics[width=3cm]{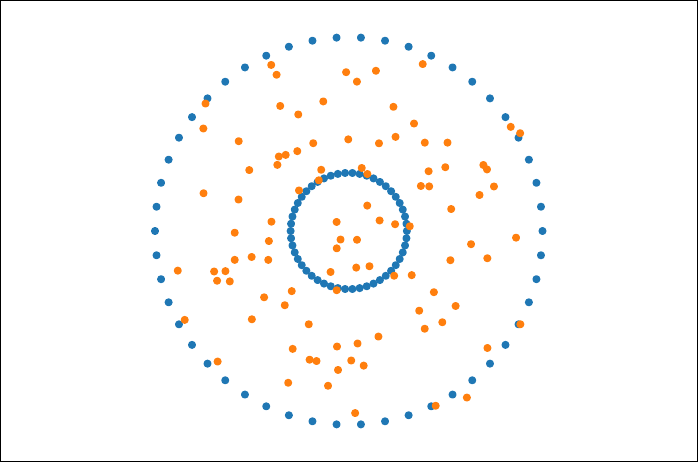} &
					\includegraphics[width=3cm]{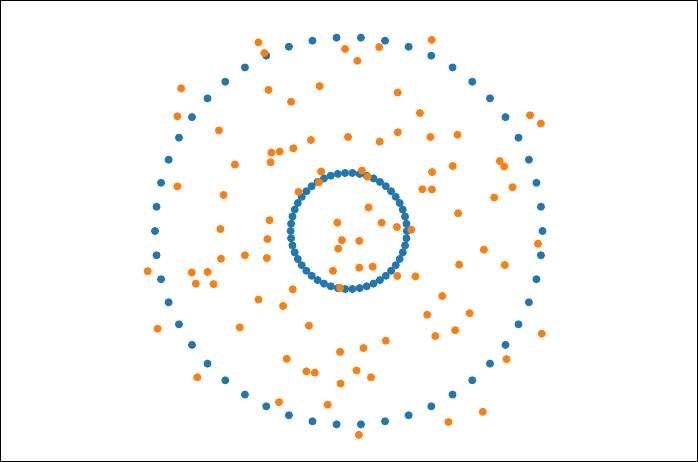} &
					\includegraphics[width=3cm]{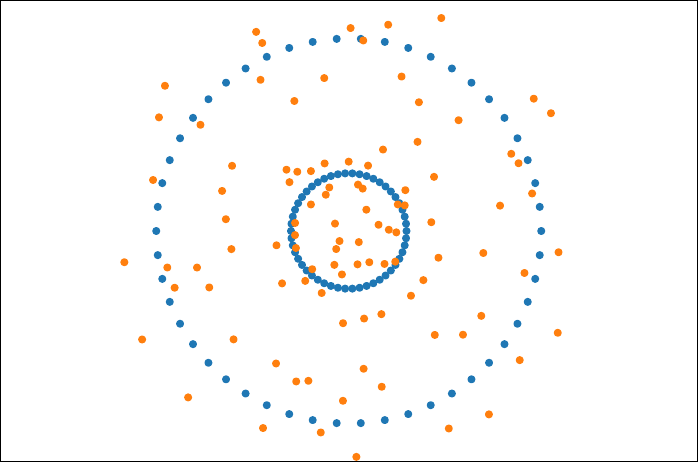} \\
					\vspace{-4pt}
					\rotatebox{90}{\parbox{2cm}{\centering $\varepsilon=0.1$}} &
					\includegraphics[width=3cm]{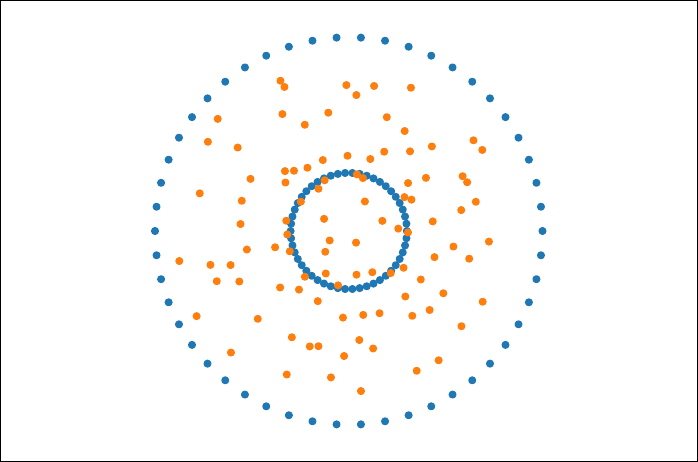} &
					\includegraphics[width=3cm]{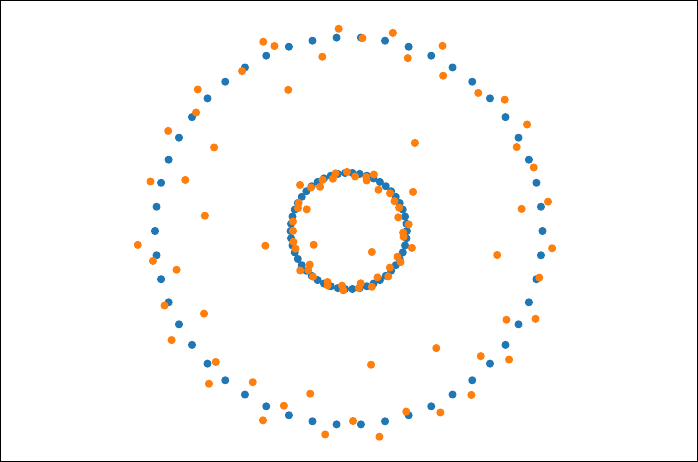} &
					\includegraphics[width=3cm]{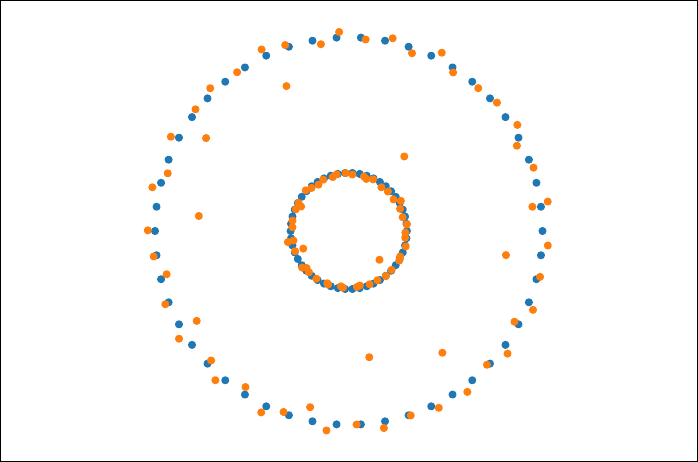} &
					\includegraphics[width=3cm]{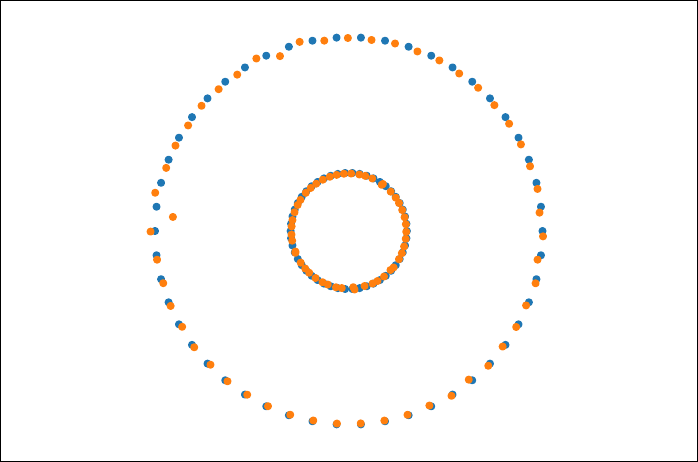} \\
					\vspace{-2pt}
					\rotatebox{90}{\parbox{2cm}{\centering $\varepsilon=0.01$}} &
					\includegraphics[width=3cm]{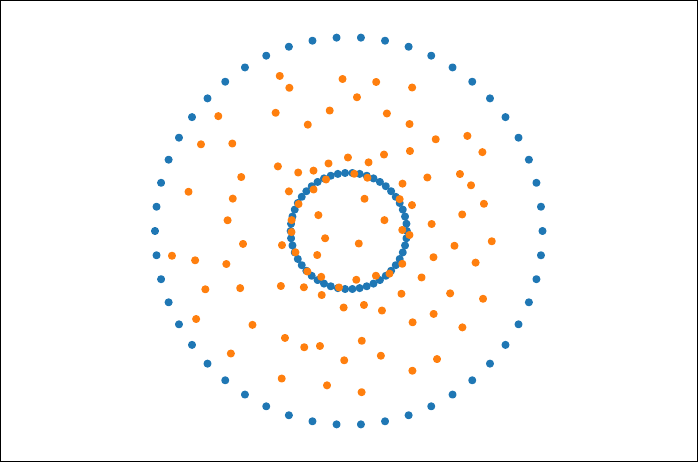} &
					\includegraphics[width=3cm]{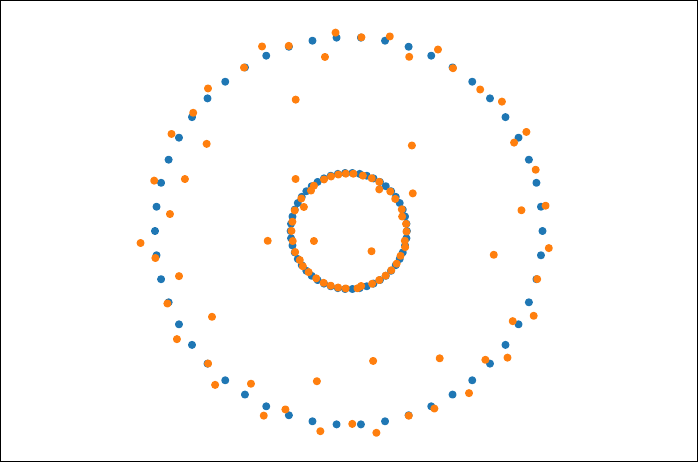} &
					\includegraphics[width=3cm]{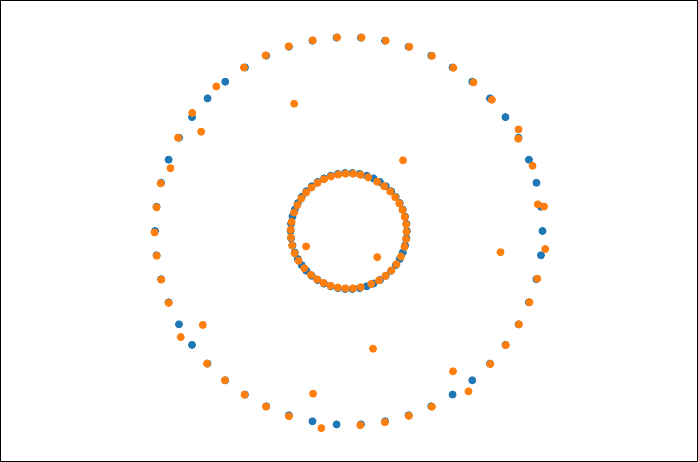} &
					\includegraphics[width=3cm]{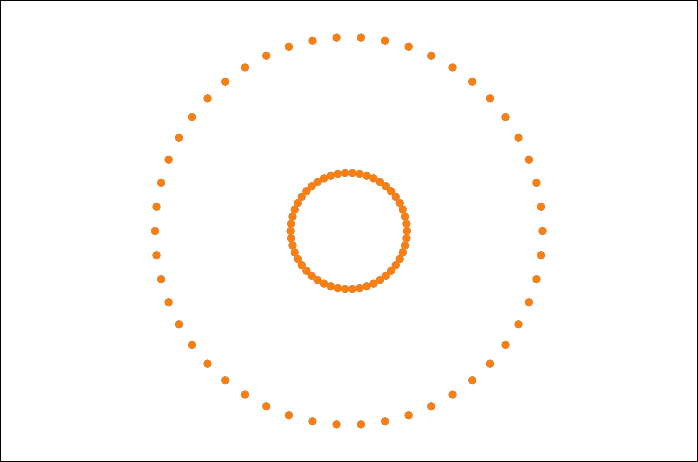} \\

					& \multicolumn{1}{c}{$t=3$} & \multicolumn{1}{c}{$t=15$} & \multicolumn{1}{c}{$t=30$} & \multicolumn{1}{c}{$t=150$} \\
				\end{tabular}
				\caption{SND}
				\label{fig:Annulus_SND}
			\end{subfigure}
			
			\begin{subfigure}{\textwidth}
				\centering\footnotesize
				\begin{tabular}{c c c c c}  
					\vspace{-4pt}
					\rotatebox{90}{\parbox{2cm}{\centering $\varepsilon=1$}} &
					\includegraphics[width=3cm]{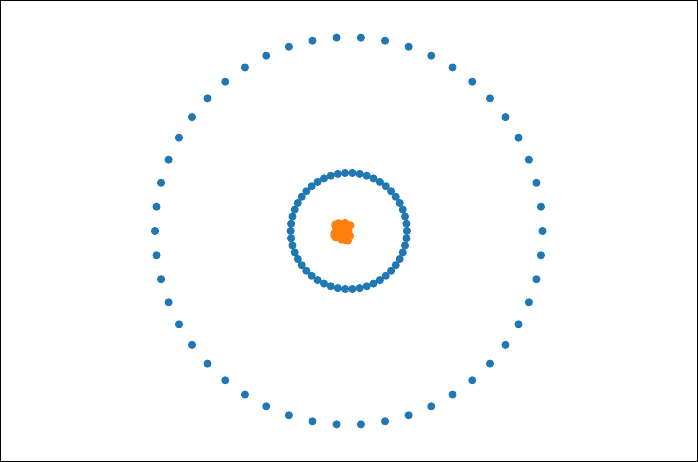} &
					\includegraphics[width=3cm]{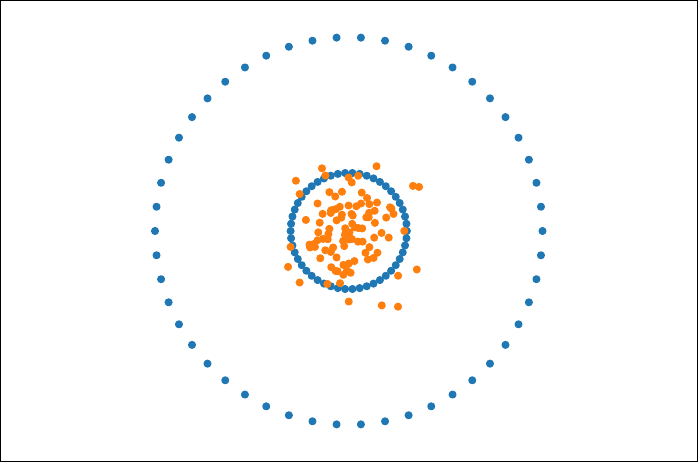} &
					\includegraphics[width=3cm]{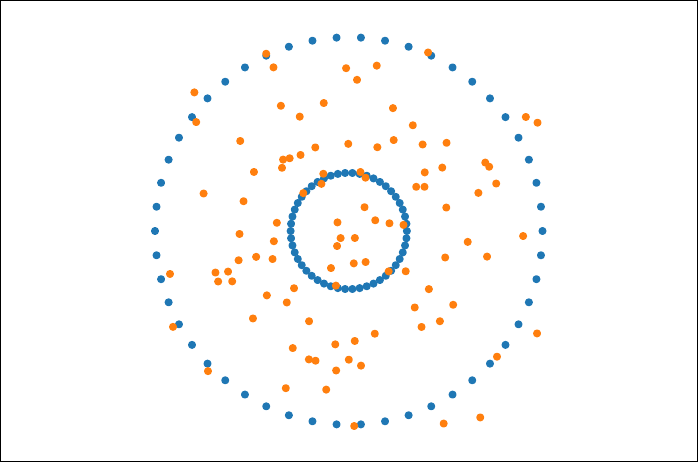} &
					\includegraphics[width=3cm]{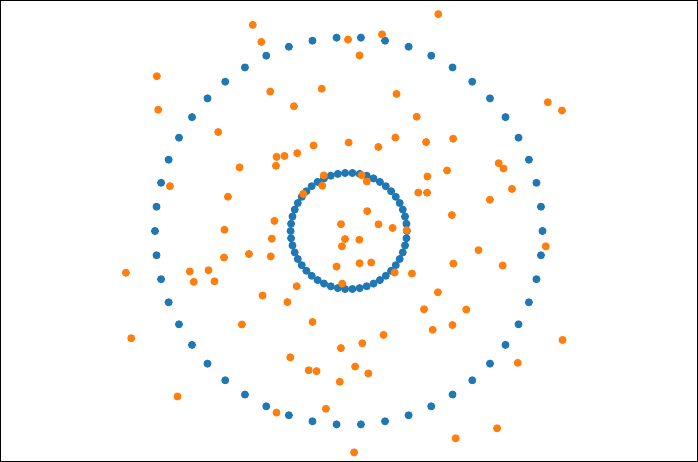} \\
					\vspace{-4pt}
					\rotatebox{90}{\parbox{2cm}{\centering $\varepsilon=0.1$}} &
					\includegraphics[width=3cm]{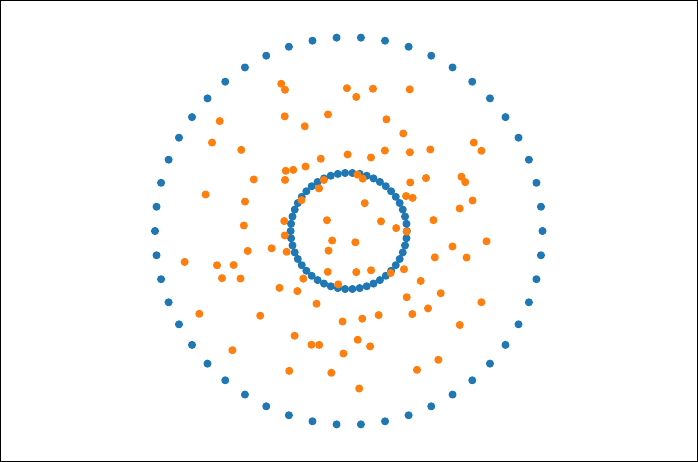} &
					\includegraphics[width=3cm]{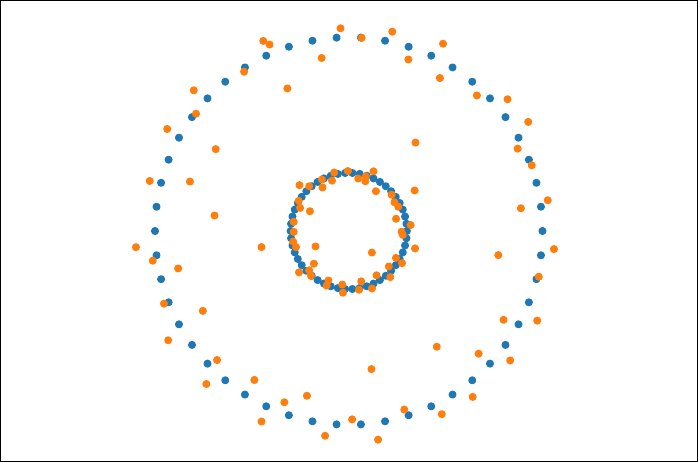} &
					\includegraphics[width=3cm]{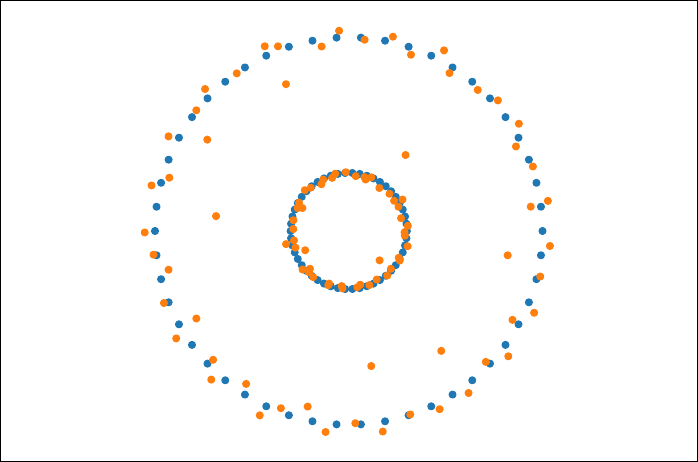} &
					\includegraphics[width=3cm]{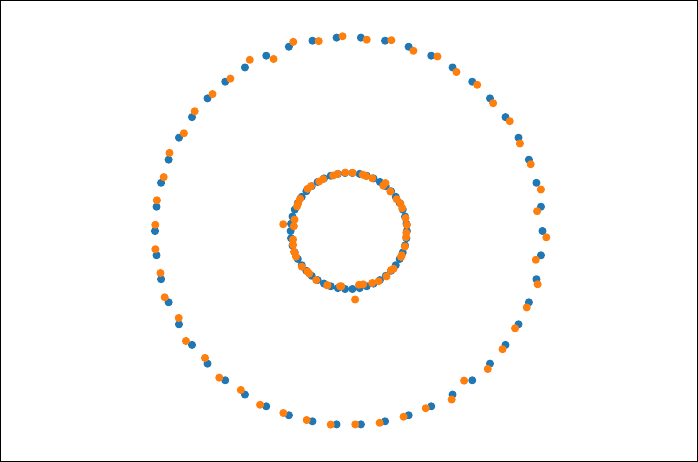} \\
					\vspace{-2pt}
					\rotatebox{90}{\parbox{2cm}{\centering $\varepsilon=0.01$}} &
					\includegraphics[width=3cm]{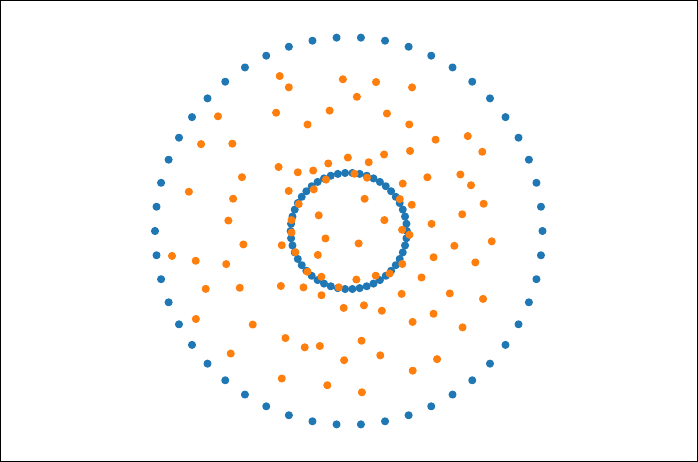} &
					\includegraphics[width=3cm]{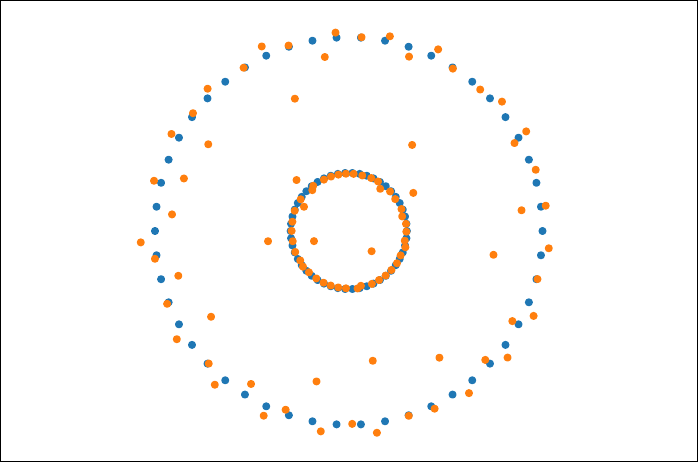} &
					\includegraphics[width=3cm]{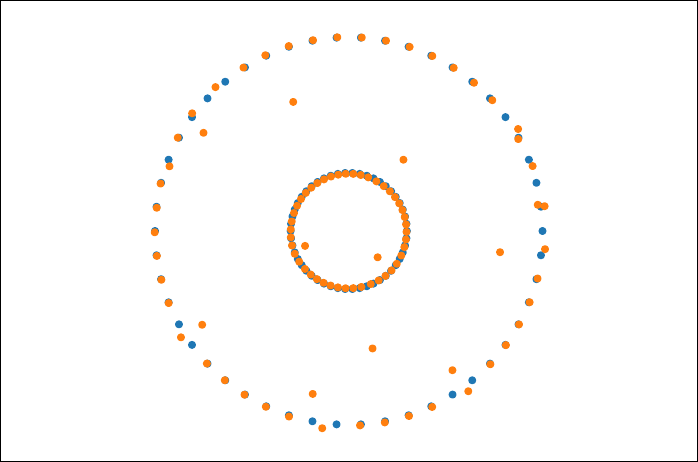} &
					\includegraphics[width=3cm]{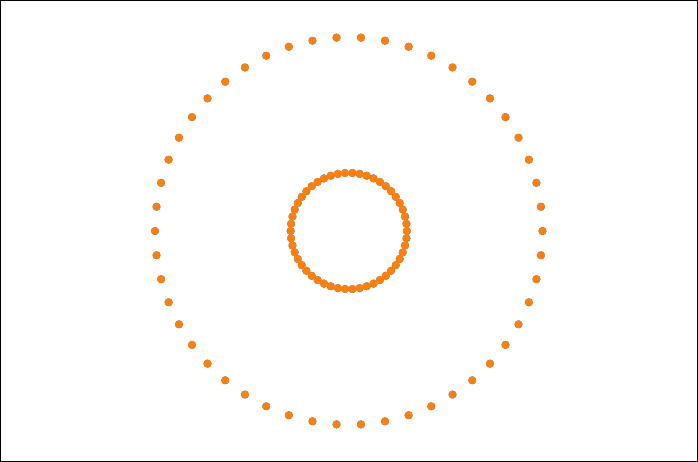} \\
					& \multicolumn{1}{c}{$t=3$} & \multicolumn{1}{c}{$t=15$} & \multicolumn{1}{c}{$t=30$} & \multicolumn{1}{c}{$t=150$} \\
				\end{tabular}
				\caption{SND4}
				\label{fig:Annulus_SND4}
			\end{subfigure}
		\end{figure}
		\begin{figure}[ht]\ContinuedFloat
			
			\begin{subfigure}{\textwidth}
				\centering\footnotesize
				\begin{tabular}{c c c c c} 
					\vspace{-2pt}
					\rotatebox{90}{\parbox{2cm}{}} &
					\includegraphics[width=3cm]{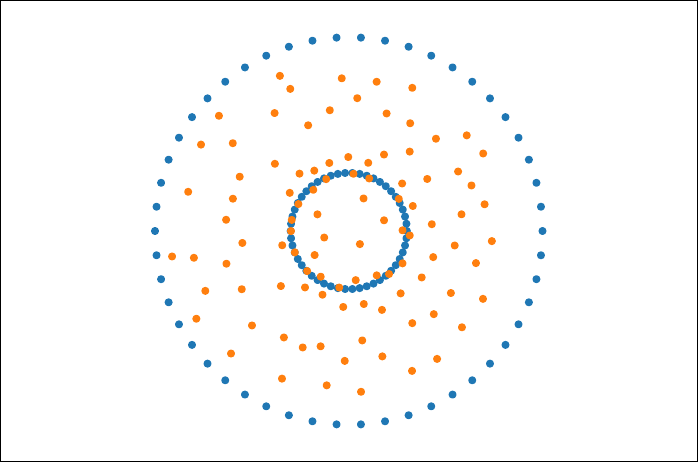} &
					\includegraphics[width=3cm]{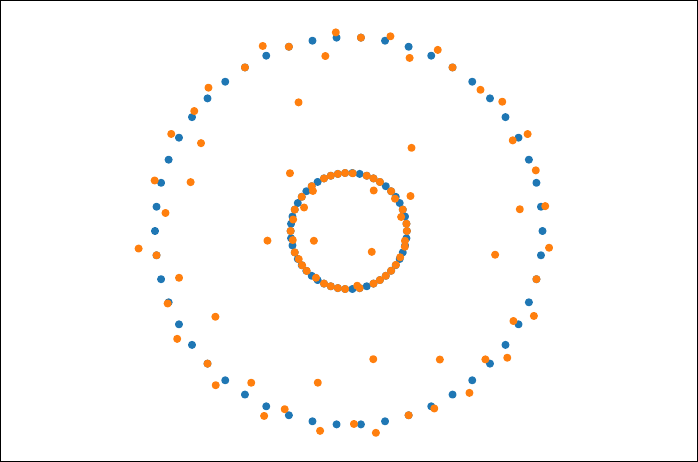} &
					\includegraphics[width=3cm]{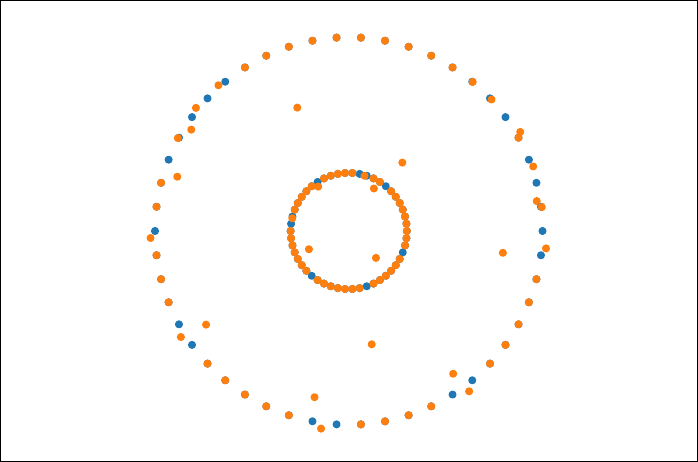} &
					\includegraphics[width=3cm]{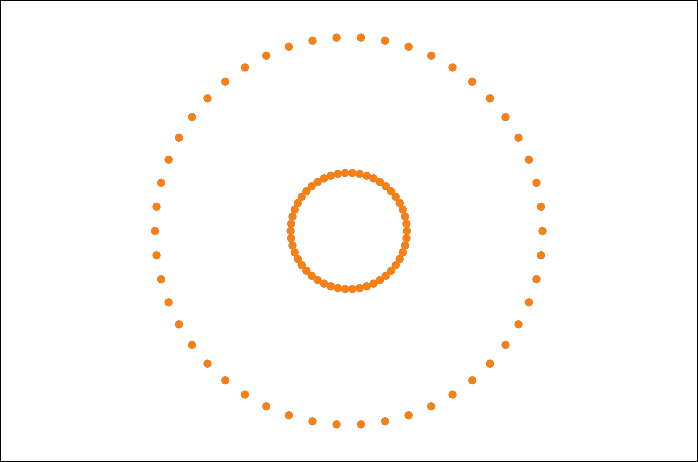} \\
					
					& \multicolumn{1}{c}{$t=3$} & \multicolumn{1}{c}{$t=15$} & \multicolumn{1}{c}{$t=30$} & \multicolumn{1}{c}{$t=150$} \\
				\end{tabular}
				\caption{ND}
				\label{fig:Annulus_ND}
			\end{subfigure}
			\caption{MMD flow \eqref{eq:mmd-flow} with step size $\tau=0.02$.}
			\label{fig:Annulus_flow}
		\end{figure}

		\begin{figure}
			\centering
			\includegraphics[width=6cm]{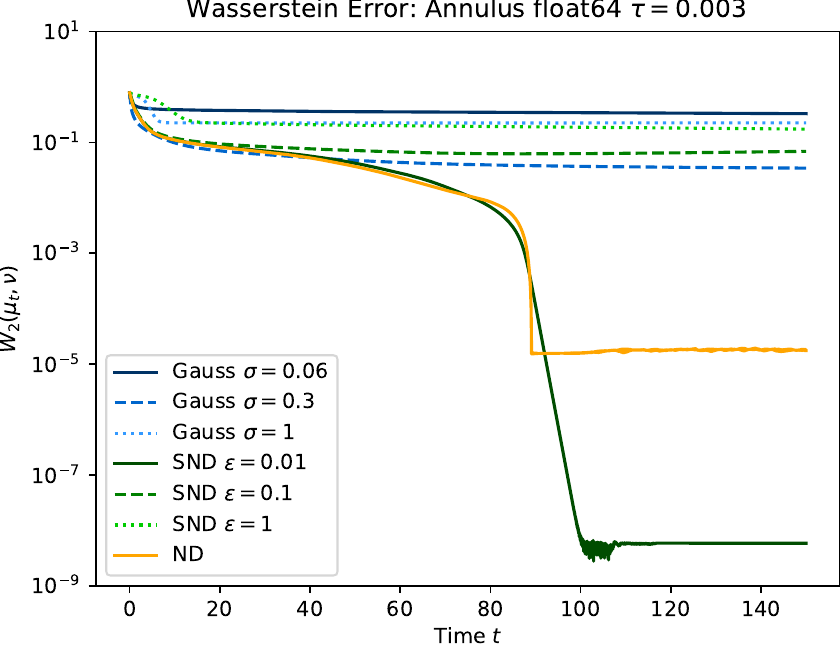}
			\includegraphics[width=6cm]{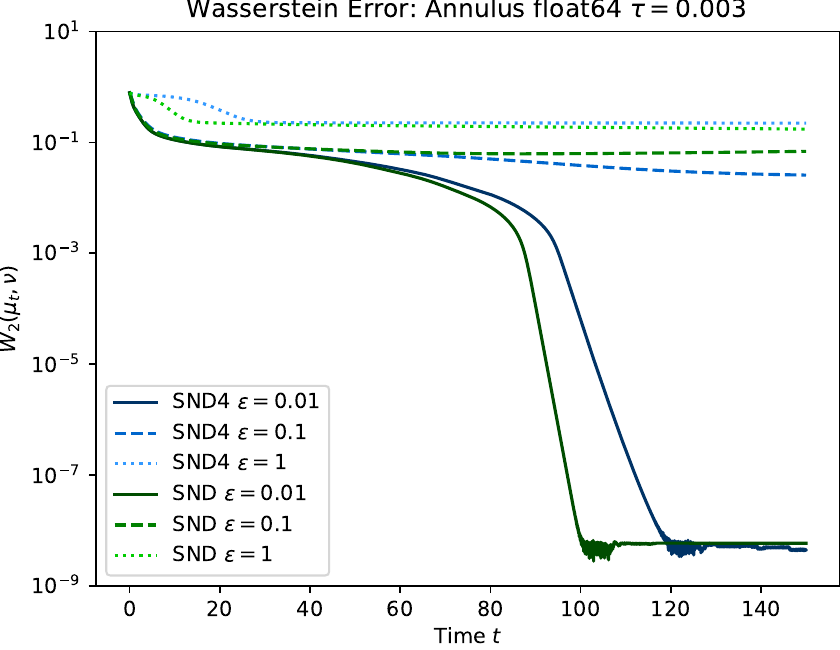}
			\caption{
				Wasserstein 2 error between Annulus target $\nu$ and flow  $\gamma_n^\tau$ after $n$ iterations. Horizontal axis in time $t=\tau n$. Both computed with double precision and step size $\tau=0.003$.
			}
			\label{fig:Annulus_W2}
		\end{figure}

		\paragraph{Computational Times.}
		
		Table \ref{tab:runtime} provides an overview of the runtime of the four considered kernels for the Annulus target. The SND4 is significantly slower than the SND, due to the more complicated structure, see Example~\ref{ex:spline_2}.
		However, as we saw above, it offers barely an advantage in accuracy.
		
		\begin{table}[]
			\centering
			\begin{tabular}{l|cccc}
				Kernel  & Gauss     & SND       & ND      & SND4 \\
				\hline
				Time (s) & $\num{11.44} $ & $\num{20.07} $ & $\num{13.63} $ & $\num{33.37} $\\
			\end{tabular}
			\caption{Runtime in seconds for the Annulus target on a GPU for $50\,000$ gradient steps, averaged over $3$ runs each.}
			\label{tab:runtime}
		\end{table}
		
	}
	
\end{document}